\documentclass[10pt,letterpaper]{article}

\usepackage{fullpage}
\usepackage{authblk}
\usepackage[utf8]{inputenc}
\usepackage[T1]{fontenc}
\usepackage{lmodern}
\usepackage{microtype}
\usepackage[british]{babel}

\usepackage{amsmath,amssymb,amsfonts,amsthm,mathtools,bm}
\usepackage{natbib}
\setcitestyle{authoryear,round,citesep={;},aysep={,},yysep={;}}
\usepackage{graphicx,xcolor,listings,booktabs,multirow,subcaption,float,placeins,needspace}
\usepackage{tikz}
\usetikzlibrary{positioning,arrows}
\usepackage{algorithm}
\usepackage[noend]{algpseudocode}
\usepackage{enumitem}
\usepackage[normalem]{ulem}
\usepackage{xurl}
\usepackage{hyperref}
\usepackage{aliascnt}
\usepackage{cleveref}
\usepackage{nicefrac}
\hypersetup{pdftitle={Doubly Robust Deep Kernel Tests for Structured Outcomes},pdfauthor={}}
\makeatletter
\def\theHALG@line{\thealgorithm.\arabic{ALG@line}}
\makeatother
\newtheorem{thm}{Theorem}[section]
\newaliascnt{lem}{thm}
\newtheorem{lem}[lem]{Lemma}
\aliascntresetthe{lem}
\crefname{lem}{lemma}{lemmas}
\Crefname{lem}{Lemma}{Lemmas}

\newaliascnt{prop}{thm}
\newtheorem{prop}[prop]{Proposition}
\aliascntresetthe{prop}
\crefname{prop}{proposition}{propositions}
\Crefname{prop}{Proposition}{Propositions}
\newtheorem{cor}[thm]{Corollary}
\newaliascnt{assum}{thm}
\newtheorem{assum}[assum]{Assumption}
\aliascntresetthe{assum}
\crefname{assum}{assumption}{assumptions}
\Crefname{assum}{Assumption}{Assumptions}
\theoremstyle{definition}

\theoremstyle{remark}
\newtheorem{rem}[thm]{Remark}

\DeclareMathOperator{\E}{\mathbb{E}}

\newcommand{\R}{\mathbb{R}}
\newcommand{\mathbbm}[1]{\mathbf{#1}}

\newcommand{\cH}{\mathcal{H}}

\newcommand{\cY}{\mathcal{Y}}

\newcommand{\cD}{\mathcal{D}}

\newcommand{\indep}{\mathrel{\perp\!\!\!\perp}}
\newcommand*\diff{\mathop{}\mathrm{d}}

\newcommand{\pmerr}[1]{{\scriptstyle\pm#1}}

\lstdefinestyle{python}{
    language=Python,
    basicstyle=\ttfamily\small,
    keywordstyle=\color{blue!65!black},
    stringstyle=\color{green!35!black},
    commentstyle=\color{black!55}\itshape,
    backgroundcolor=\color{black!3},
    rulecolor=\color{black!20},
    frame=single,
    framesep=6pt,
    xleftmargin=0.5em,
    framexleftmargin=0.5em,
    showstringspaces=false,
    columns=fullflexible,
    keepspaces=true,
    breaklines=true
}

\title{Doubly Robust Deep Kernel Tests\\for Structured Outcomes}
\author[1]{Robin Leman}
\author[2]{Bariscan Bozkurt}
\author[3]{Sofiane Tanji}
\author[2]{Arthur Gretton}
\author[2]{Houssam Zenati}
\affil[1]{University College London}
\affil[2]{Gatsby Computational Neuroscience Unit, University College London}
\affil[3]{ICTEAM/INMA, UCLouvain}
\date{}

\begin{document}
\maketitle
\begin{abstract}
We propose a test of whether two interventions induce the same outcome distribution when outcomes are structured and high-dimensional, such as images, text and graphs. We represent the interventional distributions as mean embeddings on learned deep features, such that the treatment effect is a doubly  robust representation of the difference in mean embeddings. Our test, DR-KTE-Deep,  optimises the learned  features so as to maximise an estimated power criterion on this treatment effect representation. To motivate our learning objective, we show that the  power criterion characterises the asymptotic signal-to-noise ratio of the test under fixed alternatives, and we derive a uniform finite-sample bound for its regularised estimate. We show that, with the learned kernel fixed, a held-out test is asymptotically valid under the null given suitable nuisance-estimation rates. Empirically, DR-KTE-Deep achieves substantially higher power than fixed and bandwidth-optimised kernel baselines across image, text and graph outcomes, while maintaining well-calibrated Type-I error.
\end{abstract}

\begin{figure}[H]
    \centering
    \includegraphics[width=\linewidth]{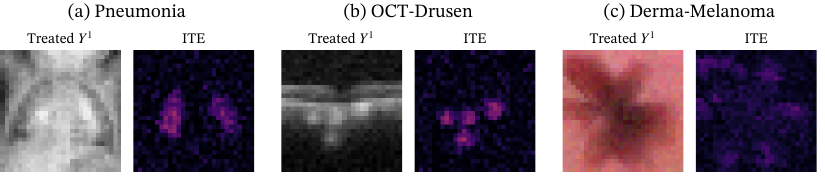}
    \caption{Semi-synthetic medical-image experiments adapted from MedMNIST \citep{medmnistv1,medmnistv2}. Each pair shows a treated potential outcome $Y^1$ and its ground-truth individual treatment effect (ITE).}
    \label{fig:medical-image-headline}
\end{figure}

\section{Introduction}
\label{chapter:introduction}

In causal problems involving structured and high-dimensional outcomes, such as images, text, and graphs, a mean outcome may be undefined or poorly aligned with the structure of the data. For example, comparing protein or text sequences requires choosing a numerical representation or a notion of similarity \citep{saigo2004protein,lodhi2002text,reimers2019sentence} and comparing images with coordinate-wise averaging ignores the spatial structure of the data \citep{lecun1998gradient}. Distributional causal inference addresses this by defining the interventional distributions directly on the structured outcome space, without reducing outcomes to an arbitrary scalar summary or notion of averaging \citep{rothe2010nonparametric,chernozhukov2013inference,kennedy2023counterfactualdensity, muandet2021counterfactual}. Kernels define similarities on these spaces \citep{gartner2003survey,borgwardt2006integrating,shervashidze2011weisfeiler} and, through kernel mean embeddings, represent interventional distributions in a reproducing kernel Hilbert space (RKHS) \citep{smola2007hilbert}.

Kernels define the geometry used to compare outcome distributions \citep{SriGreFukLanetal10,simon2018kernel, gretton2013introduction}. If this geometry is fixed a priori, the kernel may be poorly suited to detecting differences between the interventional distributions, leading to low test power \citep{gretton2012kernel,sutherland2017generative}. Deep kernels address this by using neural networks to learn a representation of the outcomes and compare distributions in the resulting space \citep{wilson2016deep}. Rather than learning this representation through a reconstruction task, the learned kernel can be optimised by directly maximising a test-power criterion \citep{sutherland2017generative}.

We consider two soft interventions $\pi$ and $\pi'$ inducing the outcome distributions $P_Y^\pi$ and $P_Y^{\pi'}$. Determining whether an intervention changes the outcome distribution can be formulated as a hypothesis testing problem, testing $H_0: P_Y^\pi = P_Y^{\pi'}$ against $H_1: P_Y^\pi \neq P_Y^{\pi'}$.

Treatment effects often need to be estimated from observational data, because experimentation can be costly, unethical or infeasible. In this setting, interventional distributions are not directly observed and must be recovered \citep{Robins1986,hernan2020causal,muandet2021counterfactual, zenati2025cpme}. This can be done by estimating either treatment assignment or the conditional outcome embedding \citep{song2009hilbert,grunewalder2012conditional}. However, either approach alone is sensitive to model misspecification. We therefore use a doubly robust estimator, which combines both approaches and remains consistent when either model is correctly specified \citep{bang2005doubly,luedtke2024}.

In this paper, we introduce DR-KTE-Deep, which combines doubly robust interventional embeddings with deep kernel learning and testing. Our main contributions are: (i) we develop a method for learning deep kernels for tests of interventional distributions, estimating both the distributional discrepancy and the test-power criterion from observational data; (ii) we derive a uniform finite-sample bound for the estimated learning criterion, accounting for both kernel-class complexity and errors in estimating the treatment-assignment and outcome models; (iii) we establish conditional asymptotic validity of the held-out test after nuisance estimation and kernel training and selection; and (iv) we evaluate the method on image, text, and graph outcomes, showing substantial power gains over fixed and bandwidth-optimised doubly robust kernel tests, while maintaining Type-I error close to the nominal level.

The remainder of the paper is organised as follows. \Cref{chapter:method} develops the doubly robust test. \Cref{section:deep_kernels} develops the deep-kernel learning objective. \Cref{chapter:experiments} evaluates its power and calibration across image, text, and graph outcomes. \Cref{chapter:conclusion} concludes.

\section{Related Work}

Several works study distributional causal effects. Quantile treatment effects characterise changes across an ordered outcome distribution \citep{abadie2002instrumental,chernozhukov2005ivqte,firpo2007efficient}, while counterfactual distribution and density methods target the full marginal law \citep{rothe2010nonparametric,chernozhukov2013inference,kennedy2023counterfactualdensity}. These methods are typically developed for scalar or finite-dimensional Euclidean outcomes. We instead consider structured outcomes whose interventional distributions must be compared without reducing them to low-dimensional summaries.

Kernel mean embeddings represent interventional outcome distributions in an RKHS. \citet{muandet2021counterfactual} introduce counterfactual mean embeddings, while \citet{park2021conditional}, \citet{naf2026causaldrf}, and \citet{jain2026conditional} use kernel methods to study conditional distributional treatment effects. For marginal effects, \citet{fawkes2024doubly} and \citet{martinez2023} develop doubly robust kernel tests, and \citet{zenati2025cpme} extend them to general stochastic policies. Finally, \citet{zenati2026kernel} further consider adaptively collected data. These methods use prescribed outcome kernels, including aggregation over kernel bandwidths, without learning a deep outcome representation. A separate literature uses kernels to estimate causal mean functions through kernel ridge regression \citep{singh2024kernel,kim2026continuous}, instrumental variable regression \citep{singh2019kernel,meunier2024nonparametric}, and proxy bridge learning \citep{bozkurt2025densityratiofree,bozkurt2026neural}. These approaches use kernels or learned representations for causal function estimation, rather than for comparing interventional outcome distributions.

Other work adapts the kernel or outcome representation to improve test power. \citet{gretton2012optimal} optimise the power of linear-time MMD tests over kernel mixtures and analyse convergence of the selection criterion. \citet{sutherland2017generative} optimise kernel parameters for two-sample tests, while \citet{liu2021learningdeepkernelsnonparametric} learn characteristic deep kernels and prove uniform convergence of their learning criterion. Other approaches combine multiple kernels through aggregation or fusion \citep{schrab2023mmdagg,biggs2023mmdfuse}. In causal settings, \citet{zenati2026semiparametric} learn informative witness locations for distributional treatment effects, while \citet{wibisono2026unstructured} and \citet{durand2025causalresponse} learn outcome features or projections to detect causal effects. Our method instead learns the outcome kernel for an omnibus test of marginal interventional distributions. We estimate the learning criterion from observational data using causal nuisance models and analyse how errors in these models affect the criterion.

\section{Doubly Robust Kernel Test}
\label{chapter:method}

\paragraph{Problem Setting.}

We are given an observational dataset $\cD$ of $N$ independent and identically distributed (i.i.d.) samples $O_1,\ldots,O_N{\sim}_{\mathrm{i.i.d}} \mathbb P_0$, where $O_i=(X_i,A_i,Y_i)\in\mathcal X\times\mathcal A\times\mathcal Y$. $\mathcal X$, $\mathcal A$, and $\mathcal Y$ are standard Borel spaces. Write $\mathbb P_X$ for the $X$-marginal of $\mathbb P_0$, and $\mathbb P_{A,X}$ for its treatment-covariate marginal. Let $\rho$ be a $\sigma$-finite reference measure on $\mathcal A$. The logging policy assigns $A\mid X=x$ with density $\pi_0(\cdot\mid x)$ relative to $\rho$. We consider a unit-level potential-outcome setting with no interference between units \citep{rubin2005causal,imbens_rubin_2015}. We compare two soft interventions $\pi$ and $\pi'$ with conditional densities $\pi(\cdot\mid x)$ and $\pi'(\cdot\mid x)$ respectively.

\paragraph{Identification.} The observational data are generated under the logging policy $\pi_0$, whereas the targets we consider are the outcome distributions under $\pi$ and $\pi'$. To identify these distributions, we work under the standard assumption that the treatment assignment is unconfounded given the observed covariates, and that the interventions assign only treatments supported by the logging policy. We formalise these conditions through the common assumptions below. 

\begin{assum}[Identification]\label{assum:identification}
    There exist a standard Borel space $\mathcal U$, a $\mathcal U$-valued random variable $U$, and a jointly measurable map $F:\mathcal A\times\mathcal X\times\mathcal U\to\mathcal Y$ such that: 
    \begin{enumerate}
        \item \textit{Consistency and structural response:} $Y=F(A,X,U)$ almost surely; for every $a\in\mathcal A$, define $Y^a:=F(a,X,U)$. 
        \item \textit{Conditional exogeneity:} $U\indep A\mid X$. 
        \item \textit{Positivity on intervention support:} $\pi_0(a\mid x)>0$ whenever $\pi(a\mid x)+\pi'(a\mid x)>0$, up to a $\mathbb P_X\otimes\rho$-null set.
    \end{enumerate}
\end{assum}

For $p\in\{\pi,\pi'\}$, draw $A^p\mid X$ with density $p(\cdot\mid X)$ independently of $U$ given $X$, and define $Y^p:=F(A^p,X,U)$ and $P_Y^p:=\mathcal L(Y^p)$. We test $H_0:P_Y^\pi=P_Y^{\pi'}$ against $H_1:P_Y^\pi\neq P_Y^{\pi'}$.

Let $Q(\cdot \mid a,x)$ denote an observational conditional distribution of $Y$ given $(A,X)=(a,x)$. Under \Cref{assum:identification}, the intervention changes the treatment assignment mechanism while preserving the conditional outcome response. The interventional outcome distribution is then identified by the g-formula   \citep{Robins1986,hernan2020causal}:
\begin{align*}
    P_Y^p(B)
=\int_{\mathcal X}\int_{\mathcal A}
Q(B\mid a,x)p(a\mid x)\rho(\diff a)\mathbb P_X(\diff x),
\end{align*}
for every measurable $B\subseteq\mathcal Y$.

This formula averages the conditional outcome distributions over the treatments assigned by $p$ and the covariate distribution of the population. Positivity ensures that this average only involves treatment-covariate pairs supported by the observational data. We establish these claims in Proposition~\ref{prop:soft-intervention-identification}.


\paragraph{Kernel Formulation.}

To compare distributions over structured outcomes, we represent them as elements of an RKHS. Let $\Theta\subseteq\mathbb R^{d_\Theta}$ be a non-empty parameter set. For each $\theta\in\Theta$, let $k_\theta:\mathcal Y\times\mathcal Y\to\mathbb R$ be a positive-semidefinite kernel on $\cY$ parameterised by $\theta$ with separable RKHS $\mathcal H_\theta$ and strongly measurable canonical feature map $\varphi_\theta(y):=k_\theta(y,\cdot)$.

We define the interventional mean embedding \citep{muandet2021counterfactual,zenati2025cpme}, embedding contrast, and squared maximum mean discrepancy (MMD) \citep{gretton2012kernel} respectively by
\[
    \chi_\theta(p) :=\int_{\mathcal Y}\varphi_\theta(y)P_Y^p(\diff y), \quad
    \Delta_\theta(\pi,\pi') :=\chi_\theta(\pi)-\chi_\theta(\pi'), \quad
    D_\theta(\pi,\pi') :=\|\Delta_\theta(\pi,\pi')\|_{\mathcal H_\theta}^2.
\]

\begin{assum}[Embedding integrability]\label{assum:embedding-integrability}
    For every $\theta\in\Theta$ and $p\in\{\pi,\pi'\}$, assume $\mathbb E_{\mathbb P_0}\sqrt{k_\theta(Y,Y)}<\infty$ and $\int_{\mathcal Y}\sqrt{k_\theta(y,y)}P_Y^p(\diff y)<\infty$. These first moments ensure Bochner integrability under the observed and interventional laws.
\end{assum}

If $k_\theta$ is characteristic, testing $H_0$ is equivalent to testing $D_\theta(\pi,\pi')=0$ \citep[Section~2.2]{gretton2012kernel}.

\paragraph{Doubly Robust Representation.}

To estimate $D_\theta(\pi,\pi')$ from observational data, we construct a doubly robust representation of the embedding contrast $\Delta_\theta(\pi,\pi')$, following \citet{zenati2025cpme}. We further compare the choice of estimator in \Cref{sec:ablation-estimator-comparison}.

For a fixed $\theta\in\Theta$, define the conditional outcome mean embedding \citep{grunewalder2012conditional} by
\[
\mu_\theta(a,x):=\int_{\mathcal Y}\varphi_\theta(y)Q(\diff y\mid a,x).
\]
We use the null-set convention described in \Cref{sec:soft-intervention-identification}. This embedding satisfies $\mu_\theta(A,X)
=\mathbb E_{\mathbb P_0}[\varphi_\theta(Y)\mid A,X]$ almost surely. For ease of notation, write $r(a,x):=\pi(a\mid x)-\pi'(a\mid x)$. Let $e$ be a measurable candidate logging-policy density that is positive wherever $|r|>0$. Let $m_\theta$ be a measurable candidate conditional outcome embedding. We set ratios and integrands weighted by $r$ or $|r|$ to zero where $r=0$.

Define the Hilbert-valued pseudo-feature map
\begin{equation}
\label{eq:generic-dr-feature}
g_{\theta}^{\pi,\pi'}(O;e,m_\theta)
:=
\frac{r(A,X)}{e(A \mid X)}
\left\{\varphi_{\theta}(Y)-m_{\theta}(A,X)\right\}
+
\int_{\mathcal A}r(a,X)m_{\theta}(a,X)\rho(\diff a).
\end{equation}

The second term estimates the embedding contrast through the conditional outcome model, while the first corrects it using propensity-weighted outcome residuals. 

We suppress the intervention superscript when no ambiguity arises, $g_\theta(O;e,m_\theta) := g_\theta^{\pi, \pi'}(O;e,m_\theta)$. The oracle pseudo-feature is then $g_\theta(O) := g_\theta(O;\pi_0,\mu_\theta)$. Write $\mathbb E_0$ for expectation under $\mathbb P_0$.

\begin{prop}[Observed-data representation and double robustness]
    \label{prop:dr-representation}
    Suppose \Cref{assum:identification,assum:embedding-integrability}
    hold and the candidate nuisances satisfy the absolute-integrability conditions of \Cref{sec:proof-dr-representation}. Write $\Delta_\theta := \Delta_\theta(\pi, \pi')$. For candidate nuisance
    functions $e$ and $m_\theta$,
    \begin{align*}
    \E_0
    \left[g_\theta(O;e,m_\theta)\right]
    -\Delta_\theta =
    \E_0\left[
    \int_{\mathcal A}
    r(a,X)
    \left\{\frac{\pi_0(a\mid X)}{e(a\mid X)}-1\right\}
    \left\{\mu_\theta(a,X)-m_\theta(a,X)\right\}
    \rho(\diff a)
    \right].
    \end{align*}
    Thus, $\E_0 \left[g_\theta(O;e,m_\theta)\right] = \Delta_\theta$ if either $e=\pi_0$ or $m_\theta=\mu_\theta$ on the support of $|r|$.
\end{prop}

A proof is given in \Cref{sec:proof-dr-representation}. We examine the effect of outcome-model misspecification empirically in \Cref{sec:ablation-outcome-model-misspecification}, using a misspecified $m_\theta$ but oracle propensity $e \equiv \pi_0$.

\paragraph{Test Statistic.}

 Since $D_\theta(\pi,\pi')=\|\Delta_\theta(\pi,\pi')\|_{\mathcal H_\theta}^2$, the embedding contrast vanishes  under $H_0$ and the induced cross-U kernel over $\cH_\theta$ is one-degenerate. We therefore estimate $D_\theta(\pi,\pi')$ by taking the inner product of two independent estimates of $\Delta_\theta(\pi,\pi')$, following the cross-U and cross-MMD constructions of \citet{kim2024dimension} and \citet{shekhar2022permutation}.
 
 Let $\hat\pi_0$ and $\hat\mu_\theta$ be nuisance estimators fitted on a sample independent of the evaluation sample, and set $\hat g_{\theta,i}:=g_\theta(O_i;\hat\pi_0,\hat\mu_\theta)$. We split the evaluation sample into two disjoint halves $\mathcal I_1$ and $\mathcal I_2$ of size $m$. Define $\bar g_{\theta,\ell} := \tfrac{1}{m} \sum_{i\in \mathcal I_\ell} \hat g_{\theta,i}$ for $\ell \in \{1, 2\}$. The cross-MMD and variance estimators are then, respectively,
\[
\widehat D_\theta^{\mathrm{cross}}
:= \langle \bar g_{\theta,1},\bar g_{\theta,2}\rangle_{\cH_\theta}
= \frac{1}{m^2}\sum_{i\in\mathcal I_1}\sum_{j\in\mathcal I_2}
\langle \hat g_{\theta,i},\hat g_{\theta,j}\rangle_{\cH_\theta},
\quad
\hat s_\theta^2
:= \frac{1}{m}\sum_{j\in\mathcal I_2}
\Bigl(\langle \bar g_{\theta,1},\hat g_{\theta,j}\rangle_{\cH_\theta} - \widehat D_\theta^{\mathrm{cross}}\Bigr)^{2}.
\]

Our test statistic is the studentised cross-U statistic
\begin{equation}\label{eq:test-statistic}
T_\theta := {\sqrt{m} \widehat D_\theta^{\mathrm{cross}}} / {\hat s_\theta}.
\end{equation}
We set $T_\theta=0$ whenever $\hat s_\theta=0$.

Let $\Phi$ denote the standard normal distribution function. For a fixed kernel, \Cref{thm:heldout-validity} gives, under the corresponding moment, projected non-degeneracy and nuisance-bias conditions of \Cref{assum:heldout-regularity}, $T_\theta\rightarrow^{d}\mathcal N(0,1)$ under $H_0$. Since $D_\theta(\pi,\pi')>0$ for a characteristic kernel under $H_1$, large positive values of $T_\theta$ therefore provide evidence against $H_0$. We reject at level $\alpha$ when $\hat p = 1-\Phi(T_\theta) < \alpha$.

\section{Learning Deep Kernels}\label{section:deep_kernels}

\paragraph{Learning Objective.}

The choice of outcome kernel affects both the population discrepancy $D_\theta$ and the sampling variability of its estimator. We therefore select $\theta$ using the signal-to-noise ratio \citep{gretton2012optimal,sutherland2017generative,liu2021learningdeepkernelsnonparametric}

\[
    J_{\theta} = {D_{\theta}} / {s_{\theta}},
    \qquad \text{with }
    D_\theta
    :=
    D_\theta(\pi,\pi'), \quad
    s_\theta^2
    :=
    \operatorname{Var}_{O \sim \mathbb P_0} \left[
    \left\langle
    \Delta_\theta,
    g_\theta(O)
    \right\rangle_{\mathcal H_\theta}
    \right]
\]
whenever $s_\theta>0$. $J_\theta$ is unchanged by a global rescaling of the kernel and governs the growth of the studentised test statistic under fixed alternatives.

\begin{prop}[Fixed-alternative growth rate]\label{prop:fixed-kernel-efficacy}
Suppose \Cref{assum:identification,assum:embedding-integrability} hold. For a fixed $\theta$, suppose $D_\theta>0$, $s_\theta>0$, and $\mathbb E_{O\sim\mathbb P_0}[\|g_\theta^{\pi,\pi'}(O)\|_{\mathcal H_\theta}^4]<\infty$. Suppose the statistics use the true nuisance functions. Then the cross-MMD and variance estimators satisfy
\[
\widehat D_\theta^{\mathrm{cross}}\xrightarrow{\mathbb P}D_\theta,
\qquad \hat s_\theta^2\xrightarrow{\mathbb P}s_\theta^2,
\qquad \nicefrac{T_\theta}{\sqrt m} \xrightarrow{\mathbb P} J_\theta.
\]
\end{prop}
A proof is given in \Cref{sec:proof-fixed-alternative}. \Cref{sec:ablation-representation-objective} further studies the choice of learning criterion. This result motivates maximising $J_\theta$; kernels with larger $J_\theta$ yield a test statistic that is asymptotically larger under fixed alternatives. Since $D_\theta$ and $s_\theta$ are unknown, we use the empirical regularised criterion \citep{jitkrittum2016, jitkrittum2017}
\begin{equation}
\label{eq:selection-criterion}
    \widehat{J}_{\lambda, \theta}^{\mathrm{DR}}
    = {\widehat D_{\theta}^{\mathrm{cross}}} /{\hat s_{\theta, \lambda}},
\end{equation}
where $\hat s_{\theta, \lambda}^2 := \hat s_{\theta}^2 + \lambda$ and $\lambda>0$. The regularisation prevents the learning objective from becoming artificially large when the estimated variance $\hat s_\theta^2$ is close to zero. It is used only to stabilise kernel selection; the final test statistic remains unregularised.

While the learning criterion can be used to optimise the parameters of a simple kernel family, like the lengthscale of a Gaussian kernel, the kernel still compares outcomes isotropically throughout the outcome space, and may have low power against localised or anisotropic perturbations at finite sample size \citep[Section~3]{liu2021learningdeepkernelsnonparametric} (\Cref{fig:deep_vis}).

\begin{figure}[t]
    \centering
    \includegraphics[width=\linewidth]{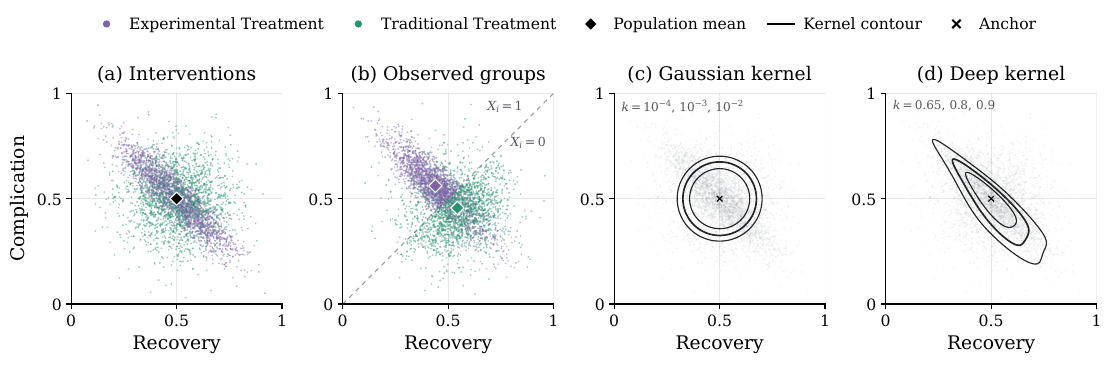}
    \caption{Synthetic example where treatment assignment is driven by an indicator of disease severity $X$. The potential-outcome distributions are Gaussians with the same mean, but the traditional-treatment distribution is isotropic while the experimental-treatment distribution is anisotropic. The deep kernel adapts to the geometry of the distributions, while the Gaussian kernel treats points uniformly throughout the outcome space.}
    \label{fig:deep_vis}
\end{figure}

\paragraph{Deep Kernels.}

We therefore learn a deep kernel combining a neural component with a base characteristic kernel. Since deep kernel optimisation is sensitive to random initialisation, we train and select an ensemble of $M$ kernels, and average them. Let $\Xi\subseteq\mathbb R^{d_0}$ be the parameter space of a kernel member and let $\theta=(\vartheta_1,\ldots,\vartheta_M) \in \Theta$ collect each $\vartheta_b\in\Xi$ for $b \in [1, M]$. Define
\begin{equation}
\label{eq:deep-kernel}
k_{\vartheta_b}(y,y')
=(1-w_{\vartheta_b})\kappa_{\vartheta_b}\bigl(\zeta_{\vartheta_b}(y),\zeta_{\vartheta_b}(y')\bigr)
+w_{\vartheta_b} k_{\vartheta_b}^{\mathrm{char}}(y,y'),\quad k_\theta(y,y')=\frac1M\sum_{b=1}^M k_{\vartheta_b}(y,y').
\end{equation}
Here $\zeta_\vartheta:\mathcal Y\to\mathbb R^d$ is a neural feature map, $\kappa_\vartheta$ is a Gaussian kernel on the learned feature space, $k_\vartheta^{\mathrm{char}}$ is a characteristic kernel on $\mathcal Y$, and $w_\vartheta\in[w_{\min},1]$ is a learned mixing parameter where $w_{\min}>0$.

The learned component $\kappa_{\vartheta_b}(\zeta_{\vartheta_b}(y),\zeta_{\vartheta_b}(y'))$ adapts the geometry used to distinguish the interventional distributions, while the characteristic component preserves sensitivity to arbitrary distributional differences.

\begin{lem}[Characteristic ensemble safeguard]\label{lem:characteristic-ensemble}
    For every pair of probability laws $P,Q$ for which the component mean embeddings exist,
    \begin{equation}
    \label{eq:characteristic-safeguard}
    \operatorname{MMD}_{k_\theta}^2(P,Q)
    \ge
    \frac{w_{\min}}M\sum_{b=1}^M
    \operatorname{MMD}_{k_{\vartheta_b}^{\mathrm{char}}}^2(P,Q).
    \end{equation}
    Consequently $k_\theta$ is characteristic for every $\theta\in\Theta$.
\end{lem}
A proof is given in \Cref{sec:proof-characteristic}. For every $\theta\in\Theta$, $D_\theta(\pi,\pi')=0$ if and only if the two interventional outcome distributions are equal.

\paragraph{Training and Testing Procedure.}

We detail the training and testing procedure in \Cref{alg:dr-dk-causal-mmd} and \Cref{sec:training-testing-details}. $\cD$ is partitioned into independent learning, validation and testing samples \citep{gretton2012optimal,liu2021learningdeepkernelsnonparametric}. We optimise the kernel parameters on the learning sample, select a kernel on the validation sample, and compute $T_\theta$ on the testing sample. The learning sample is split into three disjoint sets $\cD_{\mathrm{fit}}$, $\cD_{\mathrm{sel},1}$, and $\cD_{\mathrm{sel},2}$. We fit the nuisance estimators on $\cD_{\mathrm{fit}}$ and compute $\widehat J_{\lambda,\vartheta_b}^{\mathrm{DR}}$ on $\cD_{\mathrm{sel},1}$ and $\cD_{\mathrm{sel},2}$. At each training step, we rotate the roles of the three folds and average the resulting criterion, then optimise each $\vartheta_b$ by gradient ascent. During training, we maintain an exponential moving average (EMA) of each member's kernel parameters \citep{tarvainen2017mean} and select the EMA checkpoint with the largest unregularised validation statistic.

Let $\mathcal F_{\mathrm{learn}}$ denote the data and randomness used to estimate the nuisances and train and select the deep kernel. The fitted nuisances and selected kernel are therefore fixed when conditioned on $\mathcal F_{\mathrm{learn}}$, and the test observations remain independent. We suppose \Cref{assum:heldout-regularity} holds, including conditional fourth-moment control, non-degenerate projected variance, and a conditional pseudo-feature bias $o_{\mathbb P}(m^{-1/2})$. The latter follows from a corresponding product-rate condition on the nuisance errors (Proposition~\ref{prop:dr-representation}).

\begin{thm}[Held-out null validity]
\label{thm:heldout-validity}
Suppose \Cref{assum:heldout-regularity} holds, and set $T_{\hat\theta}=0$ when $\hat s_{\hat\theta}=0$. Under $H_0$, as the size $m$ of each test half tends to infinity,
\begin{equation}
\label{eq:heldout-null-limit}
\sup_{t\in\mathbb R}
\left|
\mathbb P(T_{\hat\theta}\le t\mid\mathcal F_{\mathrm{learn}})
-\Phi(t)
\right|
\xrightarrow{\mathbb P}0.
\end{equation}
Consequently, $T_{\hat\theta}\xrightarrow{d}\mathcal N(0,1)$ and the test has asymptotic level $\alpha$ for every fixed $\alpha\in(0,1)$.
\end{thm}

A proof is given in \Cref{sec:proof-heldout-validity}. The result is conditional on the full learning procedure and does not require the optimised parameters to converge to a deterministic kernel.

\paragraph{Uniform Control of the Learning Criterion.}

We provide a uniform analysis of the learning criterion (see also \citet[Theorem~1]{gretton2012optimal}, \citet[Theorem~11]{liu2021learningdeepkernelsnonparametric}) with the fitted nuisances $\hat\pi_0$ and $\hat\mu_\theta$, which will introduce a nuisance estimation error. The bound does not guarantee that training attains a global maximiser. Define
\begin{equation}\label{eq:population-regularized-criterion}
J_{\lambda,\theta}
:=
\frac{D_\theta}{\sqrt{s_\theta^2+\lambda}},
\qquad \lambda>0.
\end{equation}

Let the nuisance-fitting sample have size $n$ and each selection half have size $m$. Write $\mathcal F_{\mathrm{fit}}$ for the data and randomness used to fit the nuisance functions. The nuisance functions are therefore fixed when conditioned on $\mathcal F_{\mathrm{fit}}$, but retain their dependence on $\theta$.

On a compact finite-dimensional parameter class, we impose the conditions of \Cref{assum:selection-regularity}. These require the fitted kernels, constructed from $k_\theta$ and the nuisance estimates, to be bounded and Lipschitz in $\theta$, together with overlap and measurability conditions.

\begin{thm}[Uniform feasible learning criterion]
\label{thm:feasible-criterion-uniform}
Let \Cref{assum:identification,assum:selection-regularity} hold and fix $\lambda_0>0$. Let $\epsilon_{\pi,n}^2$ and $\epsilon_{\mu,n}^2$ be the nuisance errors conditional on $\mathcal F_{\mathrm{fit}}$ defined in \Cref{appendix:theory-proofs}. For every $\lambda\ge\lambda_0$ and $\eta\in(0,1)$, there is a finite constant $C$, depending only on the constants in \Cref{assum:selection-regularity} and $\lambda_0$, such that, with conditional probability at least $1-\eta$ over the two selection halves,
\begin{equation}\label{eq:regularized-uniform-bound}
\sup_{\theta\in\Theta}
|\widehat J_{\lambda,\theta}^{\mathrm{DR}}-J_{\lambda,\theta}|
\le
C\left\{
\frac{\sqrt{d_\Theta\log(1+2R_\Theta\sqrt m)+\log(4/\eta)}+1}{\sqrt m}
+\epsilon_{\pi,n}\epsilon_{\mu,n}
+\epsilon_{\pi,n}
+\epsilon_{\mu,n}^2
\right\}.
\end{equation}
For three cyclic rotations of the learning folds satisfying the same regularity constants, the averaged criterion obeys the analogous bound with unconditional probability at least $1-3\eta$, with the nuisance remainder averaged over rotations.
\end{thm}

The uniform bound separates sampling error from nuisance-estimation error. The MMD numerator retains the doubly robust product remainder $\epsilon_{\pi,n}\epsilon_{\mu,n}$. Estimating $s_\theta^2$ introduces additional terms $\epsilon_{\pi,n}$ and $\epsilon_{\mu,n}^2$. For fixed class constants and confidence level, the sampling term is $O(\sqrt{\log m/m})$, which is asymptotically faster than the $O(m^{-1/3})$ of \citet[Theorem~1]{gretton2012optimal}, \citet[Theorem~11]{liu2021learningdeepkernelsnonparametric}.

\Cref{thm:feasible-criterion-uniform} controls $J_{\lambda,\theta}$ uniformly over the full parameter class $\Theta$ for $\lambda$ bounded away from zero. On subclasses where the variance is bounded below,
\[
\Theta_s:=\{\theta\in\Theta:s_\theta^2\ge s^2\},
\]
the results in \Cref{appendix:theory-proofs} allow $\lambda_m\to0$ and yield uniform approximation to the unregularised criterion $J_\theta$. The comparison with earlier rates follows the proof of \Cref{thm:oracle-adaptation}.

Since each kernel member is trained separately, we apply these results setting $M=1$, $\Theta=\Xi$ and $d_\Theta=d_0$, under the corresponding assumptions. With $M>1$, Proposition~\ref{prop:memberwise-ensemble} gives $J_{\lambda,\theta}\ge\min_b J_{\lambda,\vartheta_b}$ for $\lambda\ge0$ with positive denominators. Conditional on the common learning and validation data and shared training randomness, independent and identically distributed member training reduces the training-randomness variance of the population discrepancy by a factor of $1/M$. Proposition~\ref{prop:memberwise-ensemble} establishes these population guarantees. \Cref{sec:ablation-ensemble-size} studies the effect of ensemble size on test power. \Cref{sec:ablation-joint-maximisation} analyses the effect of jointly maximising the averaged criterion.

\section{Experimental Evaluation}\label{chapter:experiments}

We evaluate our doubly robust deep kernel test, DR-KTE-Deep, on image, text and graph outcomes. Evaluations on low-dimensional outcomes are included in \Cref{appendix:additional-experiments}. We compare our method with DR-KTE \citep{martinez2023}, which uses a fixed Gaussian kernel, and DR-KTE-Opt, which optimises the lengthscale of a Gaussian kernel. Experimental details are given in \Cref{appendix:experimental-details}. \Cref{sec:ablation-linear-kernel} additionally compares our method to linear kernels. All methods are fitted and evaluated using the same total sample size. Our code is available at \url{https://www.github.com/robinlmn/deep-kte}.

We learn the deep kernel and nuisance parameters on $\mathcal{D}_{\text{learn}}$, sampled from the dataset's training pool, and select the kernel on $\mathcal{D}_{\text{val}}$, sampled from the validation pool. We then refit the nuisances on $\mathcal{D}_{\text{learn}} \cup \mathcal{D}_{\text{val}}$ and estimate the rejection rate at $\alpha=0.05$ over 200 test datasets $\mathcal{D}_{\text{test}}$, sampled from the testing pool. We repeat this process over 20 independent outer repetitions and report the mean rejection rate $\pm$ one standard error across the outer repetitions.

\paragraph{Image Outcomes.}\label{sec:image-outcomes} We first evaluate our method on mean-preserving distributional effects on medical-image outcomes. We extend the OCTMNIST construction of \citet{zenati2026semiparametric}. Potential outcomes are constructed as $Y_i^a=B_i+R_i(a)$ where $B_i$ is sampled from normal or healthy MedMNIST images \citep{medmnistv2} and $R_i(a)$ is a medically motivated perturbation representing a localised lesion, scattered deposits, or morphological changes. We scale the perturbation by an effect strength $\delta$. The residual construction is mean-matched, $\mathbb E[Y_i^1-Y_i^0\mid X_i]=0$, such that the effect is mainly distributional under the alternative. We consider ten constructions, three of which are illustrated in \Cref{fig:medical-image-headline}. \Cref{appendix:medical-image-details} details the full generative process.

DR-KTE-Deep achieves substantially higher estimated power than the baselines on all constructions (\Cref{tab:medmnist-power}), while DR-KTE-Opt only slightly improves power over the fixed kernel. Under $H_0$, rejection rates across the three methods remain close to the nominal level of $\alpha=0.05$, with only a slight inflation for OCT-DME. In \Cref{fig:oct-drusen-sweep}, the power of the deep kernel grows with sample size and effect strength, while Gaussian baselines remain near the nominal rejection rate across the sample-size sweep.

\begin{table}[!htbp]
    \centering
    \small
    \setlength{\tabcolsep}{4pt}
    \begin{tabular}{@{}lccc@{}}
    \toprule
    Scenario & DR-KTE & DR-KTE-Opt & DR-KTE-Deep \\
    \midrule
    Breast-Tumor & $0.749\pmerr{0.018}/0.053\pmerr{0.003}$ & $0.761\pmerr{0.017}/0.051\pmerr{0.003}$ & $\mathbf{0.989\pmerr{0.002}}/0.054\pmerr{0.003}$ \\
    Chest-Cardiomegaly & $0.156\pmerr{0.007}/0.050\pmerr{0.003}$ & $0.158\pmerr{0.008}/0.050\pmerr{0.003}$ & $\mathbf{0.928\pmerr{0.046}}/0.051\pmerr{0.004}$ \\
    Chest-Effusion & $0.231\pmerr{0.009}/0.053\pmerr{0.004}$ & $0.245\pmerr{0.011}/0.053\pmerr{0.004}$ & $\mathbf{0.999\pmerr{0.000}}/0.048\pmerr{0.004}$ \\
    Chest-Mass & $0.436\pmerr{0.018}/0.050\pmerr{0.003}$ & $0.461\pmerr{0.019}/0.050\pmerr{0.003}$ & $\mathbf{1.000\pmerr{0.000}}/0.056\pmerr{0.004}$ \\
    Derma-Melanoma & $0.171\pmerr{0.009}/0.051\pmerr{0.002}$ & $0.172\pmerr{0.009}/0.050\pmerr{0.002}$ & $\mathbf{0.898\pmerr{0.040}}/0.053\pmerr{0.004}$ \\
    OCT-CNV & $0.907\pmerr{0.014}/0.052\pmerr{0.003}$ & $0.917\pmerr{0.014}/0.051\pmerr{0.003}$ & $\mathbf{1.000\pmerr{0.000}}/0.053\pmerr{0.003}$ \\
    OCT-DME & $0.139\pmerr{0.008}/0.055\pmerr{0.004}$ & $0.144\pmerr{0.008}/0.055\pmerr{0.004}$ & $\mathbf{1.000\pmerr{0.000}}/0.063\pmerr{0.005}$ \\
    OCT-Drusen & $0.051\pmerr{0.003}/0.058\pmerr{0.004}$ & $0.052\pmerr{0.003}/0.058\pmerr{0.004}$ & $\mathbf{0.371\pmerr{0.080}}/0.058\pmerr{0.005}$ \\
    Pneumonia & $0.857\pmerr{0.021}/0.055\pmerr{0.003}$ & $0.871\pmerr{0.020}/0.056\pmerr{0.003}$ & $\mathbf{0.998\pmerr{0.001}}/0.053\pmerr{0.004}$ \\
    Retina-Severe & $0.285\pmerr{0.015}/0.050\pmerr{0.003}$ & $0.288\pmerr{0.015}/0.049\pmerr{0.003}$ & $\mathbf{0.959\pmerr{0.008}}/0.052\pmerr{0.003}$ \\
    \bottomrule
    \end{tabular}
    \caption{Power ($H_1$) and Type-I error ($H_0$) for medical-image outcomes at $N=1500$ and $\alpha = 0.05$. Each cell reports $H_1/H_0$, with means $\pm$ one standard error across 20 independent repetitions.}
    \label{tab:medmnist-power}
    \label{tab:medmnist-type1}
\end{table}

\begin{figure}[!htb]
    \centering
    \includegraphics[width=\linewidth]{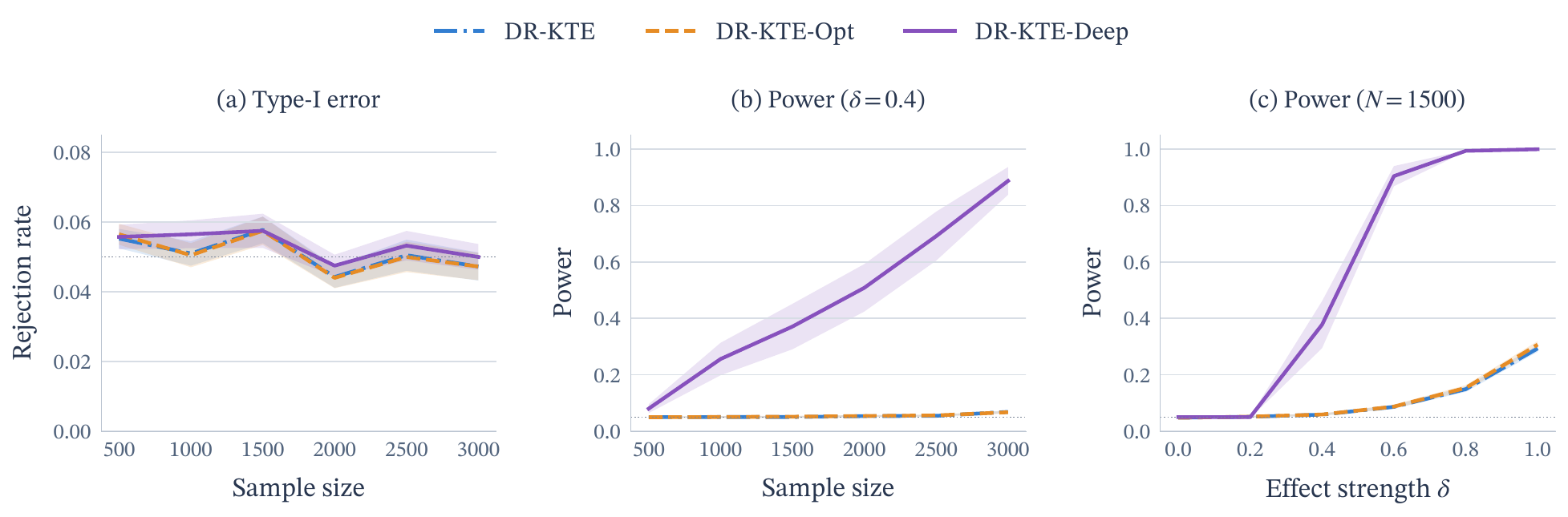}
    \caption{Results for OCT-Drusen. (a) Type-I error under $H_0$, (b) average power when varying $N$ while keeping $\delta=0.4$ fixed, (c) average power when varying $\delta$ while keeping $N=1500$ fixed. Shaded bands show $\pm$ one standard error across 20 independent repetitions. The dotted grey line marks $\alpha=0.05$.}
    \label{fig:oct-drusen-sweep}
\end{figure}

\FloatBarrier
\paragraph{Text Outcomes.}
\label{sec:text-outcome-experiments}

We now evaluate our method on text outcomes, using a language model to generate synthetic data under the assumed causal structure \citep{bynum-cho-2025-language}. Covariates and treatments generate a scalar pseudo-outcome, discretised into one of three tones used to condition the language model. We consider three scenarios: Oncology Notes, Movie Reviews and Clinical Notes. The generated text is represented as a sequence of tokens and used as input to each kernel method. Construction details are given in \Cref{appendix:text-constructions}. \Cref{appendix:text-representation-ablations} additionally compares each method through an alternative text embedding outcome representation.

\begin{table}[!htbp]
    \centering
    \small
    \setlength{\tabcolsep}{4pt}
    \begin{tabular}{@{}lccc@{}}
    \toprule
    Scenario & DR-KTE & DR-KTE-Opt & DR-KTE-Deep \\
    \midrule
    Oncology Notes & $0.371\pmerr{0.015}/0.048\pmerr{0.003}$ & $0.371\pmerr{0.015}/0.048\pmerr{0.003}$ & $\mathbf{0.809\pmerr{0.023}}/0.055\pmerr{0.003}$ \\
    Movie Reviews & $0.255\pmerr{0.012}/0.050\pmerr{0.003}$ & $0.254\pmerr{0.012}/0.050\pmerr{0.003}$ & $\mathbf{0.449\pmerr{0.026}}/0.053\pmerr{0.004}$ \\
    Clinical Notes & $0.058\pmerr{0.004}/0.052\pmerr{0.003}$ & $0.058\pmerr{0.004}/0.051\pmerr{0.003}$ & $\mathbf{0.939\pmerr{0.013}}/0.058\pmerr{0.004}$ \\
    \bottomrule
    \end{tabular}
    \caption{Power ($H_1$) and Type-I error ($H_0$) for text outcomes, represented as a sequence of tokens, at $N=1500$ and $\alpha=0.05$. Each cell reports $H_1/H_0$, with means $\pm$ one standard error across 20 independent repetitions.}
    \label{tab:text-results}
\end{table}

DR-KTE-Deep achieves the highest estimated power ({\Cref{tab:text-results}), with substantial gains on Oncology Notes and Clinical Notes and smaller improvements on Movie Reviews. Under $H_0$, average rejection rates are close to the nominal level of $\alpha=0.05$. In \Cref{fig:movie-reviews-sweep}, we vary the sample size $N$ and a treatment-effect strength $\delta$. DR-KTE-Deep maintains higher power across all sample sizes, with larger gains as $\delta$ increases.

\begin{figure}[!htb]
    \centering
    \includegraphics[width=\linewidth]{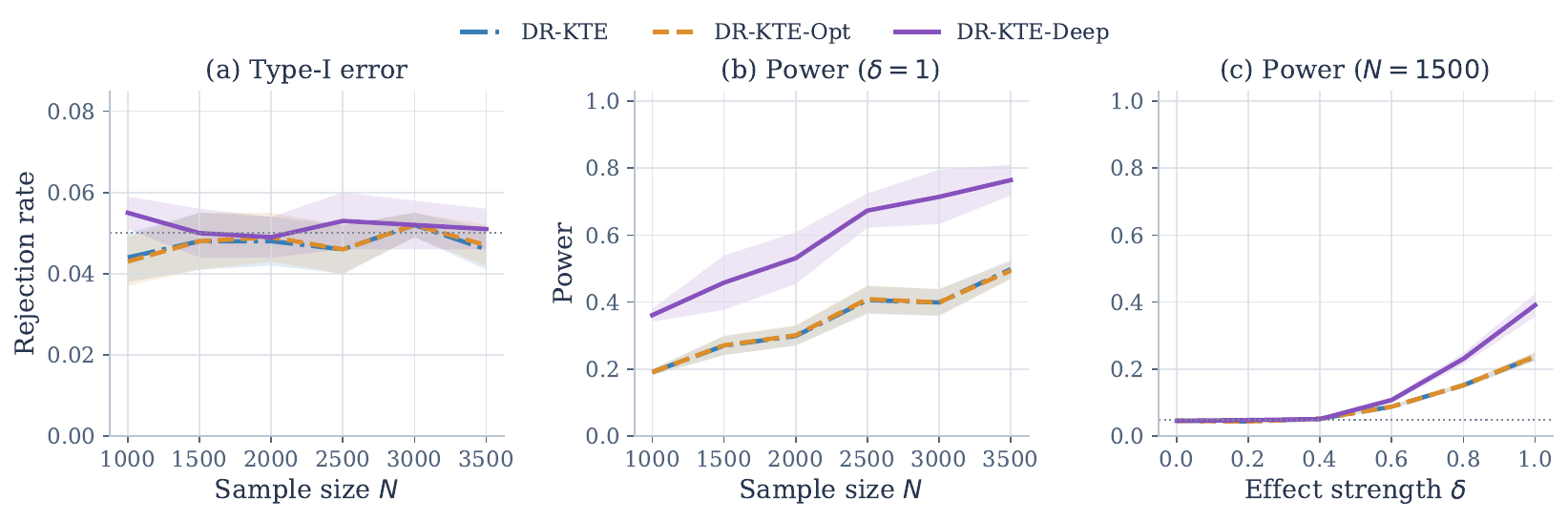}
    \caption{Results for Movie Reviews. (a) Type-I error under $H_0$, (b) average test power when varying $N$ while keeping $\delta=1.0$ fixed, (c) average test power when varying $\delta$ while keeping $N=1500$ fixed. Shaded bands show $\pm$ one standard error across 20 independent repetitions. The dotted grey line marks $\alpha=0.05$.}
    \label{fig:movie-reviews-sweep}
\end{figure}

\paragraph{Graph Outcomes.}
\label{sec:graph-outcome-experiments}

We consider two experiments in which the outcome is a labelled graph. For LLM Code Repairs, we simulate the effect of using a large language model (LLM) to fix a buggy Python function represented as its abstract syntax tree (AST). For Contact Networks, we simulate the effect of social distancing on the spread of an infection using contact networks sampled from the Infectious SocioPatterns dataset \citep{isella2011crowd}. The graph constructions and encodings are detailed in \Cref{appendix:graph-constructions}. In addition to Gaussian-based kernels, we evaluate Weisfeiler-Lehman (WL) subtree kernels \citep{shervashidze2011weisfeiler} in \Cref{sec:ablation-wl-kernel}.

DR-KTE-Deep achieves higher estimated average power on both LLM Code Repairs and Contact Networks, with smaller improvements on the latter. The Gaussian baselines already have high power on both constructions. Under $H_0$, rejection rates range from $0.048$ to $0.052$, close to the nominal level (\Cref{tab:graph-results}).

\begin{table}[!htbp]
    \centering
    \small
    \setlength{\tabcolsep}{4pt}
    \begin{tabular}{@{}lccc@{}}
    \toprule
    Scenario & DR-KTE & DR-KTE-Opt & DR-KTE-Deep \\
    \midrule
    LLM Code Repairs & $0.968\pmerr{0.008}/0.048\pmerr{0.004}$ & $0.963\pmerr{0.009}/0.048\pmerr{0.004}$ & $\mathbf{1.000\pmerr{0.000}}/0.051\pmerr{0.003}$ \\
    Contact Networks & $0.959\pmerr{0.006}/0.052\pmerr{0.003}$ & $0.958\pmerr{0.006}/0.052\pmerr{0.003}$ & $\mathbf{0.962\pmerr{0.006}}/0.051\pmerr{0.002}$ \\
    \bottomrule
    \end{tabular}
    \caption{Power ($H_1$) and Type-I error ($H_0$) for graph outcomes at $N=1500$. Each cell reports $H_1/H_0$, with means $\pm$ one standard error across 20 outer repetitions.}
    \label{tab:graph-results}
\end{table}

\FloatBarrier
\section{Discussion}
\label{chapter:conclusion}
In this paper, we introduced DR-KTE-Deep, a doubly robust test of interventional distributions. We learned a deep outcome kernel by maximising an estimated power criterion, allowing the kernel to adapt to the geometry of the outcome space. We established conditional asymptotic null validity after kernel learning and selection, and a uniform finite-sample bound for the regularised learning criterion. DR-KTE-Deep achieved substantially higher empirical power than fixed and bandwidth-optimised baselines, while maintaining Type-I error close to the nominal level. \Cref{appendix:additional-experiments} further motivates our choice of estimator and learning objective, and examines how outcome-model misspecification, outcome representation, ensemble size and outcome dimension affect the test. We use separate samples to learn the kernel and compute the test statistic. Future work could investigate whether reusing observations can further improve power while maintaining Type-I error control. We also restrict the test to two interventions, fixed in advance. Future work could extend our method to multiple candidate interventions.

\clearpage
\section*{AI Use Statement}
Language models were used to generate synthetic text outcomes for our experiments, as described in \Cref{appendix:text-constructions}. We also used generative AI for coding assistance, primarily for plotting and running scripts. We also used generative AI for proof strategies and checking theoretical results. Finally, we used generative AI to proofread and improve writing and clarity. The synthetic text datasets were manually checked by the authors for quality. LLM-generated code was verified and tested for correctness by the authors. We take responsibility for the final content of this work, including text, claims or artifacts produced with the aid of generative AI.
\bibliographystyle{iclr2027_conference}
\bibliography{references}

@article{saigo2004protein,
  title   = {Protein Homology Detection Using String Alignment Kernels},
  author  = {Saigo, Hiroto and Vert, Jean-Philippe and Ueda, Nobuhisa and Akutsu, Tatsuya},
  journal = {Bioinformatics},
  volume  = {20},
  number  = {11},
  pages   = {1682--1689},
  year    = {2004},
  doi     = {10.1093/bioinformatics/bth141}
}

@article{lodhi2002text,
  title   = {Text Classification Using String Kernels},
  author  = {Lodhi, Huma and Saunders, Craig and Shawe-Taylor, John and Cristianini, Nello and Watkins, Chris},
  journal = {Journal of Machine Learning Research},
  volume  = {2},
  pages   = {419--444},
  year    = {2002}
}

@inproceedings{reimers2019sentence,
  title     = {Sentence-{BERT}: Sentence Embeddings using {S}iamese {BERT}-Networks},
  author    = {Reimers, Nils and Gurevych, Iryna},
  booktitle = {Proceedings of the 2019 Conference on Empirical Methods in Natural Language Processing and the 9th International Joint Conference on Natural Language Processing},
  pages     = {3982--3992},
  year      = {2019},
  publisher = {Association for Computational Linguistics},
  doi       = {10.18653/v1/D19-1410}
}

@inproceedings{
zenati2026kernel,
title={Kernel Treatment Effects with Adaptively Collected Data},
author={Houssam Zenati and Bariscan Bozkurt and Arthur Gretton},
booktitle={The 29th International Conference on Artificial Intelligence and Statistics},
year={2026},
url={https://openreview.net/forum?id=di2b5POqcE}
}

@inproceedings{naf2026causaldrf,
  title     = {Causal-DRF: Conditional Kernel Treatment Effect Estimation using Distributional Random Forest},
  author    = {N{\"a}f, Jeffrey and Park, Junhyung and Susmann, Herbert},
  booktitle = {Proceedings of the 29th International Conference on Artificial Intelligence and Statistics},
  series    = {Proceedings of Machine Learning Research},
  volume    = {300},
  pages     = {1000--1008},
  year      = {2026},
  publisher = {PMLR}
}

@article{SriGreFukLanetal10,
  author  = {B. Sriperumbudur and A. Gretton and K. Fukumizu and
             B. Sch{\"o}lkopf and G. Lanckriet},
  journal = {Journal of Machine Learning Research},
  pages   = {1517-1561},
  title   = {{Hilbert} Space Embeddings and Metrics on Probability
             Measures},
  volume  = 11,
  year    = 2010
}

@book{imbens_rubin_2015,
  place     = {Cambridge},
  title     = {Causal Inference for Statistics, Social, and Biomedical Sciences: An Introduction},
  publisher = {Cambridge University Press},
  author    = {Imbens, Guido W. and Rubin, Donald B.},
  year      = {2015}
}

@inproceedings{martinez2023,
  author    = {Martinez-Taboada, Diego and Ramdas, Aaditya and Kennedy, Edward H.},
  booktitle = {Advances in Neural Information Processing Systems},
  pages     = {59924--59952},
  url       = {https://proceedings.neurips.cc/paper_files/paper/2023/hash/bccdd196d798a51a4961989984a9ed4a-Abstract-Conference.html},
  title     = {An Efficient Doubly-Robust Test for the Kernel Treatment Effect},
  volume    = {36},
  year      = {2023}
}

@article{fawkes2024doubly,
  title   = {Doubly Robust Kernel Statistics for Testing Distributional Treatment Effects},
  author  = {Jake Fawkes and Robert Hu and Robin J. Evans and Dino Sejdinovic},
  journal = {Transactions on Machine Learning Research},
  year    = {2024}
}

@article{luedtke2024,
  author  = {Alex Luedtke and Incheoul Chung},
  title   = {{One-step estimation of differentiable Hilbert-valued parameters}},
  volume  = {52},
  journal = {The Annals of Statistics},
  number  = {4},
  pages   = {1534 -- 1563},
  year    = {2024}
}

@inproceedings{jitkrittum2016,
 author = {Jitkrittum, Wittawat and Szab\'{o}, Zolt\'{a}n and Chwialkowski, Kacper P and Gretton, Arthur},
 booktitle = {Advances in Neural Information Processing Systems},
 editor = {D. Lee and M. Sugiyama and U. Luxburg and I. Guyon and R. Garnett},
 pages = {},
 publisher = {Curran Associates, Inc.},
 title = {Interpretable Distribution Features with Maximum Testing Power},
 volume = {29},
 year = {2016}
}

@inproceedings{jitkrittum2017,
 author = {Jitkrittum, Wittawat and Xu, Wenkai and Szabo, Zoltan and Fukumizu, Kenji and Gretton, Arthur},
 booktitle = {Advances in Neural Information Processing Systems},
 editor = {I. Guyon and U. Von Luxburg and S. Bengio and H. Wallach and R. Fergus and S. Vishwanathan and R. Garnett},
 pages = {},
 publisher = {Curran Associates, Inc.},
 title = {A Linear-Time Kernel Goodness-of-Fit Test},
 volume = {30},
 year = {2017}
}

@InProceedings{zenati25a,
  title = 	 {Double Debiased Machine Learning for Mediation Analysis with Continuous Treatments},
  author =       {Zenati, Houssam and Ab{\'e}cassis, Judith and Josse, Julie and Thirion, Bertrand},
  booktitle = 	 {International Conference on Artificial Intelligence and Statistics},
  pages = 	 {4150--4158},
  year = 	 {2025},
  volume = 	 {258},
}

@article{bhattacharjee2025doubly,
  title   = {Doubly Robust Estimation of Causal Effects for Random Object Outcomes with Continuous Treatments},
  author  = {Bhattacharjee, Satarupa and Li, Bing and Wu, Xiao and Xue, Lingzhou},
  journal = {arXiv preprint arXiv:2506.22754},
  year    = {2025}
}

@inproceedings{bozkurt2025densityratiofree,
 author = {Bozkurt, Bariscan and Zenati, Houssam and Meunier, Dimitri and Xu, Liyuan and Gretton, Arthur},
 booktitle = {Advances in Neural Information Processing Systems},
 title = {Density Ratio-Free Doubly Robust Proxy Causal Learning},
 year = {2025}
}

@misc{bozkurt2026neural,
      title={Doubly Robust Proxy Causal Learning with Neural Mean Embeddings}, 
      author={Bariscan Bozkurt and Alexandre Galashov and Dimitri Meunier and Zikai Shen and Arthur Gretton and Houssam Zenati},
      year={2026},
      eprint={2605.09514},
      url={https://arxiv.org/abs/2605.09514}, 
}

@inproceedings{kim2026continuous,
  title = {{Estimating Continuous Treatment Effects with Two-Stage Kernel Ridge Regression}},
  author = {Kim, Seok-Jin and Wang, Kaizheng},
  booktitle = {International Conference on Machine Learning},
  year = {2026},
}

@article{Robins1986,
  author  = {Robins, JM},
  title   = {A new approach to causal inference in mortality studies
             with a sustained exposure period - application to control of the healthy worker
             survivor effect},
  journal = {Mathematical Modelling},
  year    = {1986},
  volume  = {7},
  pages   = {1393--1512}
}

@article{kim2024dimension,
  title     = {Dimension-agnostic inference using cross {U}-statistics},
  author    = {Kim, Ilmun and Ramdas, Aaditya},
  journal   = {Bernoulli},
  volume    = {30},
  number    = {1},
  pages     = {683--711},
  year      = {2024},
  publisher = {Bernoulli Society for Mathematical Statistics and Probability}
}

@techreport{Colangelo2020,
  author = {Kyle Colangelo and Ying-Ying Lee},
  title  = {{Double Debiased Machine Learning Nonparametric Inference with Continuous Treatments}},
  year   = 2020,
  url    = {https://arxiv.org/pdf/2004.03036}
}

@article{singh2024kernel,
  title     = {Kernel methods for causal functions: dose, heterogeneous and incremental response curves},
  author    = {Singh, Rahul and Xu, Liyuan and Gretton, Arthur},
  journal   = {Biometrika},
  volume    = {111},
  number    = {2},
  pages     = {497--516},
  year      = {2024},
  publisher = {Oxford University Press}
}

@inproceedings{smola2007hilbert,
  title        = {A {H}ilbert space embedding for distributions},
  author       = {Smola, Alex and Gretton, Arthur and Song, Le and Sch{\"o}lkopf, Bernhard},
  booktitle    = {International conference on algorithmic learning theory},
  pages        = {13--31},
  year         = {2007},
  organization = {Springer}
}

@article{rothe2010nonparametric,
  title   = {Nonparametric estimation of distributional policy effects},
  author  = {Rothe, Christoph},
  journal = {Journal of Econometrics},
  volume  = {155},
  number  = {1},
  pages   = {56--70},
  year    = {2010}
}

@article{chernozhukov2013inference,
  title   = {Inference on counterfactual distributions},
  author  = {Chernozhukov, Victor and Fern{\'a}ndez-Val, Iv{\'a}n and Melly, Blaise},
  journal = {Econometrica},
  volume  = {81},
  number  = {6},
  pages   = {2205--2268},
  year    = {2013}
}

@article{muandet2021counterfactual,
  title   = {Counterfactual mean embeddings},
  author  = {Muandet, Krikamol and Kanagawa, Motonobu and Saengkyongam, Sorawit and Marukatat, Sanparith},
  journal = {Journal of Machine Learning Research},
  volume  = {22},
  number  = {162},
  pages   = {1--71},
  year    = {2021}
}

@article{gretton2012kernel,
  title   = {A kernel two-sample test},
  author  = {Gretton, Arthur and Borgwardt, Karsten M and Rasch, Malte J and Sch{\"o}lkopf, Bernhard and Smola, Alexander},
  journal = {The Journal of Machine Learning Research},
  volume  = {13},
  number  = {25},
  pages   = {723--773},
  year    = {2012}
}

@inproceedings{park2021conditional,
  title     = {Conditional distributional treatment effect with kernel conditional mean embeddings and {U}-statistic regression},
  author    = {Park, Junhyung and Shalit, Uri and Sch{\"o}lkopf, Bernhard and Muandet, Krikamol},
  booktitle = {Proceedings of the 38th International Conference on Machine Learning},
  series    = {Proceedings of Machine Learning Research},
  volume    = {139},
  pages     = {8401--8412},
  year      = {2021},
  url       = {https://proceedings.mlr.press/v139/park21c.html}
}

@article{gartner2003survey,
  title   = {A survey of kernels for structured data},
  author  = {G{\"a}rtner, Thomas},
  journal = {ACM SIGKDD explorations newsletter},
  volume  = {5},
  number  = {1},
  pages   = {49--58},
  year    = {2003}
}

@misc{gretton2013introduction,
  author = {Gretton, Arthur},
  title = {Introduction to {RKHS}, and Some Simple Kernel Algorithms},
  year = {2019},
  howpublished = {Lecture notes, University College London},
  url = {https://www.gatsby.ucl.ac.uk/~gretton/coursefiles/lecture4_introToRKHS.pdf}
}

@inproceedings{song2009hilbert,
  title     = {Hilbert space embeddings of conditional distributions with applications to dynamical systems},
  author    = {Song, Le and Huang, Jonathan and Smola, Alex and Fukumizu, Kenji},
  booktitle = {Proceedings of the 26th Annual International Conference on Machine Learning},
  pages     = {961--968},
  year      = {2009}
}

@article{borgwardt2006integrating,
  title     = {Integrating structured biological data by kernel maximum mean discrepancy},
  author    = {Borgwardt, Karsten M and Gretton, Arthur and Rasch, Malte J and Kriegel, Hans-Peter and Sch{\"o}lkopf, Bernhard and Smola, Alex J},
  journal   = {Bioinformatics},
  volume    = {22},
  number    = {14},
  pages     = {e49--e57},
  year      = {2006},
  publisher = {Oxford University Press}
}

@article{simon2018kernel,
  title   = {Kernel distribution embeddings: Universal kernels, characteristic kernels and kernel metrics on distributions},
  author  = {Simon-Gabriel, Carl-Johann and Sch{\"o}lkopf, Bernhard},
  journal = {Journal of Machine Learning Research},
  volume  = {19},
  number  = {44},
  pages   = {1--29},
  year    = {2018}
}

@article{bang2005doubly,
  title   = {Doubly robust estimation in missing data and causal inference models},
  author  = {Bang, H. and Robins, J. M.},
  journal = {Biometrics},
  volume  = {61},
  number  = {4},
  pages   = {962--973},
  year    = {2005}
}

@inproceedings{grunewalder2012conditional,
  title     = {Conditional mean embeddings as regressors},
  author    = {Gr{\"u}new{\"a}lder, Steffen and Lever, Guy and Baldassarre, Luca and Patterson, Sam and Gretton, Arthur and Pontil, Massimiliano},
  booktitle = {International Conference on Machine Learning},
  year      = {2012}
}

@inproceedings{zenati2025cpme,
  author    = {Zenati, Houssam and Bozkurt, Bariscan and Gretton, Arthur},
  title     = {Doubly-Robust Estimation of Counterfactual Policy Mean Embeddings},
  booktitle = {Advances in Neural Information Processing Systems},
  volume    = {38},
  year      = {2025},
  url       = {https://proceedings.neurips.cc/paper_files/paper/2025/hash/49a77f5557668b87a568c2d9468d4d6b-Abstract-Conference.html}
}

@misc{zenati2026semiparametric,
  title         = {Semiparametric Efficient Test for Interpretable Distributional Treatment Effects},
  author        = {Houssam Zenati and Arthur Gretton},
  year          = {2026},
  eprint        = {2605.08034},
  archiveprefix = {arXiv},
  primaryclass  = {stat.ML},
  url           = {https://arxiv.org/abs/2605.08034}
}

@inproceedings{gretton2012optimal,
  author    = {Gretton, Arthur and Sejdinovic, Dino and Strathmann, Heiko and Balakrishnan, Sivaraman and Pontil, Massimiliano and Fukumizu, Kenji and Sriperumbudur, Bharath K.},
  title     = {Optimal Kernel Choice for Large-Scale Two-Sample Tests},
  booktitle = {Advances in Neural Information Processing Systems},
  volume    = {25},
  pages     = {1205--1213},
  year      = {2012},
  url       = {https://proceedings.neurips.cc/paper/2012/hash/dbe272bab69f8e13f14b405e038deb64-Abstract.html}
}

@inproceedings{liu2021learningdeepkernelsnonparametric,
  author    = {Liu, Feng and Xu, Wenkai and Lu, Jie and Zhang, Guangquan and Gretton, Arthur and Sutherland, Danica J.},
  title     = {Learning Deep Kernels for Non-Parametric Two-Sample Tests},
  booktitle = {Proceedings of the 37th International Conference on Machine Learning},
  series    = {Proceedings of Machine Learning Research},
  volume    = {119},
  pages     = {6316--6326},
  year      = {2020},
  url       = {https://proceedings.mlr.press/v119/liu20m.html}
}

@article{rubin2005causal,
  title   = {Causal Inference Using Potential Outcomes: Design, Modeling, Decisions},
  author  = {Rubin, Donald B.},
  journal = {Journal of the American Statistical Association},
  volume  = {100},
  number  = {469},
  pages   = {322--331},
  year    = {2005},
  doi     = {10.1198/016214504000001880}
}

@book{hernan2020causal,
  title     = {Causal Inference: What If},
  author    = {Hern{\'a}n, Miguel A. and Robins, James M.},
  publisher = {Chapman \& Hall/CRC},
  address   = {Boca Raton},
  year      = {2020}
}

@article{abadie2002instrumental,
  title   = {Instrumental Variables Estimates of the Effect of Subsidized Training on the Quantiles of Trainee Earnings},
  author  = {Abadie, Alberto and Angrist, Joshua and Imbens, Guido},
  journal = {Econometrica},
  volume  = {70},
  number  = {1},
  pages   = {91--117},
  year    = {2002},
  doi     = {10.1111/1468-0262.00270}
}

@article{chernozhukov2005ivqte,
  title   = {An {IV} Model of Quantile Treatment Effects},
  author  = {Chernozhukov, Victor and Hansen, Christian},
  journal = {Econometrica},
  volume  = {73},
  number  = {1},
  pages   = {245--261},
  year    = {2005},
  doi     = {10.1111/j.1468-0262.2005.00570.x}
}

@article{firpo2007efficient,
  title   = {Efficient Semiparametric Estimation of Quantile Treatment Effects},
  author  = {Firpo, Sergio},
  journal = {Econometrica},
  volume  = {75},
  number  = {1},
  pages   = {259--276},
  year    = {2007},
  doi     = {10.1111/j.1468-0262.2007.00738.x}
}

@inproceedings{wilson2016deep,
  title     = {Deep Kernel Learning},
  author    = {Wilson, Andrew Gordon and Hu, Zhiting and Salakhutdinov, Ruslan and Xing, Eric P.},
  booktitle = {Proceedings of the 19th International Conference on Artificial Intelligence and Statistics (AISTATS)},
  pages     = {370--378},
  year      = {2016}
}

@inproceedings{sutherland2017generative,
  title     = {Generative Models and Model Criticism via Optimized Maximum Mean Discrepancy},
  author    = {Sutherland, Danica J. and Tung, Hsiao-Yu and Strathmann, Heiko and De, Soumyajit and Ramdas, Aaditya and Smola, Alex and Gretton, Arthur},
  booktitle = {International Conference on Learning Representations (ICLR)},
  year      = {2017},
  url       = {https://openreview.net/forum?id=HJWHIKqgl}
}

@article{lecun1998gradient,
  title   = {Gradient-based learning applied to document recognition},
  author  = {LeCun, Yann and Bottou, L{\'e}on and Bengio, Yoshua and Haffner, Patrick},
  journal = {Proceedings of the IEEE},
  volume  = {86},
  number  = {11},
  pages   = {2278--2324},
  year    = {1998},
  doi     = {10.1109/5.726791}
}

@article{medmnistv2,
  title     = {{MedMNIST} v2---A large-scale lightweight benchmark for {2D} and {3D} biomedical image classification},
  author    = {Yang, Jiancheng and Shi, Rui and Wei, Donglai and Liu, Zequan and Zhao, Lin and Ke, Bilian and Pfister, Hanspeter and Ni, Bingbing},
  journal   = {Scientific Data},
  volume    = {10},
  number    = {1},
  pages     = {41},
  year      = {2023},
  publisher = {Nature Publishing Group UK London}
}

@inproceedings{medmnistv1,
  title     = {{MedMNIST} Classification Decathlon: A Lightweight {AutoML} Benchmark for Medical Image Analysis},
  author    = {Yang, Jiancheng and Shi, Rui and Ni, Bingbing},
  booktitle = {IEEE 18th International Symposium on Biomedical Imaging (ISBI)},
  pages     = {191--195},
  year      = {2021}
}

@article{breastmnist,
  title   = {Dataset of breast ultrasound images},
  author  = {Walid Al-Dhabyani and Mohammed Gomaa and Hussien Khaled and Aly Fahmy},
  journal = {Data in Brief},
  volume  = {28},
  pages   = {104863},
  year    = {2020}
}

@article{retinamnist,
  author    = {Ruhan Liu and Xiangning Wang and Qiang Wu and Ling Dai and Xi Fang and Tao Yan and Jaemin Son and Shiqi Tang and Jiang Li and Zijian Gao and Adrian Galdran and J. M. Poorneshwaran and Hao Liu and Jie Wang and Yerui Chen and Prasanna Porwal and Gavin Siew {Wei Tan} and Xiaokang Yang and Chao Dai and Haitao Song and Mingang Chen and Huating Li and Weiping Jia and Dinggang Shen and Bin Sheng and Ping Zhang},
  title     = {{DeepDRiD}: Diabetic Retinopathy-Grading and Image Quality Estimation Challenge},
  journal   = {Patterns},
  volume    = {3},
  number    = {6},
  pages     = {100512},
  year      = {2022},
  doi       = {10.1016/j.patter.2022.100512},
  publisher = {Elsevier}
}

@inproceedings{shekhar2022permutation,
  title     = {A Permutation-Free Kernel Two-Sample Test},
  author    = {Shekhar, Shubhanshu and Kim, Ilmun and Ramdas, Aaditya},
  booktitle = {Advances in Neural Information Processing Systems},
  volume    = {35},
  pages     = {18168--18180},
  year      = {2022},
  url       = {https://proceedings.neurips.cc/paper_files/paper/2022/hash/731b952bdc833485eb72f458cdd5c489-Abstract-Conference.html}
}

@article{shervashidze2011weisfeiler,
  title   = {{Weisfeiler--Lehman} Graph Kernels},
  author  = {Shervashidze, Nino and Schweitzer, Pascal and van Leeuwen, Erik Jan and Mehlhorn, Kurt and Borgwardt, Karsten M.},
  journal = {Journal of Machine Learning Research},
  volume  = {12},
  pages   = {2539--2561},
  year    = {2011}
}

@inproceedings{ehsani2025hierarchical,
  author    = {Ramtin Ehsani and Esteban Parra and Sonia Haiduc and Preetha Chatterjee},
  title     = {Hierarchical Knowledge Injection for Improving {LLM}-based Program Repair},
  booktitle = {Proceedings of the 40th IEEE/ACM International Conference on Automated Software Engineering (ASE)},
  year      = {2025},
  pages     = {1440--1452},
  publisher = {IEEE},
  doi       = {10.1109/ASE63991.2025.00122}
}

@inproceedings{widyasari2020bugsinpy,
  author    = {Ratnadira Widyasari and Sheng Qin Sim and Camellia Lok
               and Haodi Qi and Jack Phan and Qijin Tay and Constance Tan
               and Fiona Wee and Jodie Ethelda Tan and Yuheng Yieh
               and Brian Goh and Ferdian Thung and Hong Jin Kang
               and Thong Hoang and David Lo and Eng Lieh Ouh},
  title     = {{BugsInPy}: A Database of Existing Bugs in {Python} Programs
               to Enable Controlled Testing and Debugging Studies},
  booktitle = {Proceedings of the 28th ACM Joint Meeting on European
               Software Engineering Conference and Symposium on the
               Foundations of Software Engineering},
  series    = {ESEC/FSE 2020},
  year      = {2020},
  pages     = {1556--1560},
  publisher = {Association for Computing Machinery},
  doi       = {10.1145/3368089.3417943}
}

@article{isella2011crowd,
  author  = {Lorenzo Isella and Juliette Stehl{\'e} and Alain Barrat
             and Ciro Cattuto and Jean-Fran{\c{c}}ois Pinton
             and Wouter {Van den Broeck}},
  title   = {What's in a Crowd? {Analysis} of Face-to-Face Behavioral Networks},
  journal = {Journal of Theoretical Biology},
  volume  = {271},
  number  = {1},
  pages   = {166--180},
  year    = {2011},
  doi     = {10.1016/j.jtbi.2010.11.033}
}

@inproceedings{pang2005seeing,
  author    = {Bo Pang and Lillian Lee},
  title     = {Seeing Stars: Exploiting Class Relationships for Sentiment Categorization with Respect to Rating Scales},
  booktitle = {Proceedings of the 43rd Annual Meeting of the Association for Computational Linguistics},
  year      = {2005},
  pages     = {115--124},
  doi       = {10.3115/1219840.1219855},
  url       = {https://aclanthology.org/P05-1015/}
}

@misc{mtsamples,
  author       = {Harish Nair},
  title        = {{MTSamples}},
  year         = {2024},
  howpublished = {Hugging Face dataset},
  url          = {https://huggingface.co/datasets/harishnair04/mtsamples}
}

@inproceedings{bynum-cho-2025-language,
  author    = {Bynum, Lucius E.J. and Cho, Kyunghyun},
  title     = {Language Models as Causal Effect Generators},
  booktitle = {Proceedings of the 2025 Conference on Empirical Methods in Natural Language Processing},
  year      = {2025},
  pages     = {2096--2115},
  address   = {Suzhou, China},
  publisher = {Association for Computational Linguistics},
  doi       = {10.18653/v1/2025.emnlp-main.107},
  url       = {https://aclanthology.org/2025.emnlp-main.107/}
}

@inproceedings{durand2025causalresponse,
  author    = {Durand, Homer and Varando, Gherardo and Camps-Valls, Gustau},
  title     = {Learning Causal Response Representations through Direct Effect Analysis},
  booktitle = {Proceedings of the Forty-first Conference on Uncertainty in Artificial Intelligence},
  series    = {Proceedings of Machine Learning Research},
  volume    = {286},
  pages     = {1135--1166},
  year      = {2025},
  url       = {https://proceedings.mlr.press/v286/durand25a.html}
}

@inproceedings{jain2026conditional,
  author        = {Jain, Saksham and Luedtke, Alex},
  title         = {Conditional Distributional Treatment Effects: Doubly Robust Estimation and Testing},
  booktitle     = {International Conference on Machine Learning},
  year          = {2026},
  archiveprefix = {arXiv},
  eprint        = {2603.16829},
  note          = {In press},
  url           = {https://arxiv.org/abs/2603.16829}
}

@article{kennedy2023counterfactualdensity,
  author  = {Kennedy, Edward H. and Balakrishnan, Sivaraman and Wasserman, Larry},
  title   = {Semiparametric Counterfactual Density Estimation},
  journal = {Biometrika},
  volume  = {110},
  number  = {4},
  pages   = {875--896},
  year    = {2023},
  doi     = {10.1093/biomet/asad017},
  url     = {https://doi.org/10.1093/biomet/asad017}
}

@misc{wibisono2026unstructured,
  author        = {Wibisono, Kevin Christian and Wang, Yixin},
  title         = {Causal Inference with Unstructured Outcomes},
  year          = {2026},
  archiveprefix = {arXiv},
  eprint        = {2608.03085},
  note          = {Preprint},
  url           = {https://arxiv.org/abs/2608.03085}
}

@misc{zhang2025qwen3embedding,
  author        = {Zhang, Yanzhao and Li, Mingxin and Long, Dingkun and Zhang, Xin and Lin, Huan and Yang, Baosong and Xie, Pengjun and Yang, An and Liu, Dayiheng and Lin, Junyang and Huang, Fei and Zhou, Jingren},
  title         = {{Qwen3} Embedding: Advancing Text Embedding and Reranking Through Foundation Models},
  year          = {2025},
  eprint        = {2506.05176},
  archiveprefix = {arXiv},
  primaryclass  = {cs.CL},
  url           = {https://arxiv.org/abs/2506.05176}
}

@article{newman2002spread,
  title = {Spread of epidemic disease on networks},
  author = {Newman, M. E. J.},
  journal = {Phys. Rev. E},
  volume = {66},
  issue = {1},
  pages = {016128},
  numpages = {11},
  year = {2002},
  month = {Jul},
  publisher = {American Physical Society},
  doi = {10.1103/PhysRevE.66.016128},
  url = {https://link.aps.org/doi/10.1103/PhysRevE.66.016128}
}

@inproceedings{kingma2015adam,
  author    = {Diederik P. Kingma and Jimmy Ba},
  title     = {{Adam}: A Method for Stochastic Optimization},
  booktitle = {International Conference on Learning Representations},
  year      = {2015},
  url       = {https://arxiv.org/abs/1412.6980}
}

@article{luedtke2019omnibus,
  author  = {Luedtke, Alexander R. and Carone, Marco and van der Laan, Mark J.},
  title   = {An Omnibus Non-Parametric Test of Equality in Distribution for Unknown Functions},
  journal = {Journal of the Royal Statistical Society: Series B (Statistical Methodology)},
  volume  = {81},
  number  = {1},
  pages   = {75--99},
  year    = {2019},
  doi     = {10.1111/rssb.12299}
}

@inproceedings{biggs2023mmdfuse,
  title     = {{MMD-Fuse}: Learning and Combining Kernels for Two-Sample Testing Without Data Splitting},
  author    = {Biggs, Felix and Schrab, Antonin and Gretton, Arthur},
  booktitle = {Advances in Neural Information Processing Systems},
  volume    = {36},
  year      = {2023},
  url       = {https://proceedings.neurips.cc/paper_files/paper/2023/hash/edd00cead3425393baf13004de993017-Abstract.html}
}

@article{schrab2023mmdagg,
  title   = {{MMD} Aggregated Two-Sample Test},
  author  = {Schrab, Antonin and Kim, Ilmun and Albert, M{\'e}lisande and Laurent, B{\'e}atrice and Guedj, Benjamin and Gretton, Arthur},
  journal = {Journal of Machine Learning Research},
  volume  = {24},
  number  = {194},
  pages   = {1--81},
  year    = {2023},
  url     = {https://www.jmlr.org/papers/v24/21-1289.html}
}

@inproceedings{tarvainen2017mean,
  author = {Antti Tarvainen and Harri Valpola},
  title = {Mean teachers are better role models: Weight-averaged consistency targets improve semi-supervised deep learning results},
  booktitle = {Advances in Neural Information Processing Systems},
  volume = {30},
  year = {2017},
  url = {https://proceedings.neurips.cc/paper/2017/hash/68053af2923e00204c3ca7c6a3150cf7-Abstract.html}
}

@misc{meunier2024nonparametric,
  title = {Nonparametric Instrumental Regression via Kernel Methods is Minimax Optimal},
  author = {Meunier, Dimitri and Li, Zhu and Christensen, Tim and Gretton, Arthur},
  year = {2024},
  eprint = {2411.19653},
  archivePrefix = {arXiv},
  primaryClass = {stat.ML},
  url = {https://arxiv.org/abs/2411.19653}
}

@inproceedings{singh2019kernel,
  title = {Kernel Instrumental Variable Regression},
  author = {Singh, Rahul and Sahani, Maneesh and Gretton, Arthur},
  booktitle = {Advances in Neural Information Processing Systems},
  volume = {32},
  pages = {4595--4607},
  year = {2019},
}
\clearpage
\appendix

\section*{Appendix}

Appendix~\ref{appendix:notation} recalls the main notation used throughout this paper. Appendix~\ref{appendix:theory-proofs} presents our theoretical results and proofs. Appendix~\ref{appendix:experimental-details} describes the training and testing procedure, data constructions, architectures and hyperparameters. Appendix~\ref{appendix:additional-experiments} presents additional experiments and ablations.

\section{Notations}
\label{appendix:notation}

\medskip
\noindent\textbf{Data, interventions and samples}
\begin{itemize}[label=--,leftmargin=1.5em,labelsep=0.5em,itemsep=0pt,topsep=1.5pt,parsep=0pt,partopsep=0pt]
    \item $O_i=(X_i,A_i,Y_i)\in\mathcal X\times\mathcal A\times\mathcal Y$ denotes an observed covariate-treatment-outcome triplet and $\mathcal D=\{O_i\}_{i=1}^N$ the observed dataset.
    \item $\mathbb P_0$ is the observed-data law, $\mathbb P_X$ its covariate marginal, $\mathbb P_{A,X}$ its treatment-covariate marginal, and $\mathbb E_0$ denotes expectation under $\mathbb P_0$.
    \item $\pi_0(a\mid x)$ is the logging-policy density; $\pi(a\mid x)$ and $\pi'(a\mid x)$ are the target intervention densities, all relative to the $\sigma$-finite treatment reference measure $\rho$.
    \item $U\in\mathcal U$ is the structural variable, with $U\indep A\mid X$; the jointly measurable response $F:\mathcal A\times\mathcal X\times\mathcal U\to\mathcal Y$ defines $Y^a=F(a,X,U)$ and $Y=F(A,X,U)$.
    \item $A^p\mid X\sim p(\cdot\mid X)$, independently of $U$ given $X$; $Y^p=F(A^p,X,U)$ has law $P_Y^p$. We test $H_0:P_Y^\pi=P_Y^{\pi'}$.
    \item $P_{U\mid X}$ is a regular conditional law of $U$ given $X$; $Q(B\mid a,x)=\int\mathbbm 1\{F(a,x,u)\in B\}P_{U\mid X}(\diff u\mid x)$ is the common response kernel.
    \item $r(a,x)=\pi(a\mid x)-\pi'(a\mid x)$ is the policy contrast. Ratios and integrands weighted by $r$ or $|r|$ are set to zero where $r=0$.
    \item $\mathcal D_{\mathrm{learn}}$, $\mathcal D_{\mathrm{val}}$ and $\mathcal D_{\mathrm{test}}$ are the disjoint training, validation and test samples; $\mathcal D^{(1)},\mathcal D^{(2)},\mathcal D^{(3)}$ are the three learning folds.
    \item $n$ is the nuisance-fitting sample size; $\mathcal I_1,\mathcal I_2$ index two evaluation halves of size $m$. In the test analysis, $m$ denotes the size of each test half, also written $m_{\mathrm{test}}$.
\end{itemize}

\medskip
\needspace{4\baselineskip}
\noindent\textbf{Kernels and interventional mean embeddings}
\begin{itemize}[label=--,leftmargin=1.5em,labelsep=0.5em,itemsep=0pt,topsep=1.5pt,parsep=0pt,partopsep=0pt]
    \item $k_\theta$ is an outcome kernel with parameter $\theta\in\Theta\subseteq\mathbb R^{d_\Theta}$, RKHS $\mathcal H_\theta$ and feature map $\varphi_\theta(y)=k_\theta(y,\cdot)$. Inner products and norms are in $\mathcal H_\theta$.
    \item $\chi_\theta(p)=\int\varphi_\theta(y)P_Y^p(\diff y)$ is the mean embedding under intervention $p\in\{\pi,\pi'\}$.
    \item $\Delta_\theta=\chi_\theta(\pi)-\chi_\theta(\pi')$ is the embedding contrast and $D_\theta=\|\Delta_\theta\|_{\mathcal H_\theta}^2$ is the squared MMD. For characteristic $k_\theta$, $H_0$ is equivalent to $D_\theta=0$.
    \item $\mu_\theta(a,x)=\int\varphi_\theta(y)Q(\diff y\mid a,x)$ is the conditional outcome mean embedding, with the null-set convention in \Cref{sec:soft-intervention-identification}; it agrees $\mathbb P_{A,X}$-almost surely with the observational conditional mean.
    \item $\{k_\vartheta:\vartheta\in\Xi\}$, with $\Xi\subseteq\mathbb R^{d_0}$, is the family optimised by each training run. $M$ is the ensemble size and $\theta=(\vartheta_1,\ldots,\vartheta_M)\in\Theta=\Xi^M$ collects its member parameters, with $d_\Theta=Md_0$.
    \item $\zeta_{\vartheta_b}:\mathcal Y\to\mathbb R^d$ is the neural feature map of member $b$ and $\kappa_{\vartheta_b}$ is its Gaussian kernel on the learned features.
    \item $k_{\vartheta_b}^{\mathrm{char}}$ is the characteristic component, with weight $w_b\in[w_{\min},1]$ for $w_{\min}>0$. The ensemble kernel is $k_\theta=M^{-1}\sum_{b=1}^M k_{\vartheta_b}$.
\end{itemize}

\medskip
\needspace{4\baselineskip}
\noindent\textbf{Nuisance estimation, kernel learning and testing}
\begin{itemize}[label=--,leftmargin=1.5em,labelsep=0.5em,itemsep=0pt,topsep=1.5pt,parsep=0pt,partopsep=0pt]
    \item $e$ and $m_\theta$ are candidate logging-policy and conditional-embedding functions; $\hat\pi_0$ and $\hat\mu_\theta$ estimate $\pi_0$ and $\mu_\theta$, respectively.
    \item $g_\theta(O;e,m_\theta)$ is the doubly robust pseudo-feature in \eqref{eq:generic-dr-feature}. Its oracle version is $g_\theta(O)=g_\theta(O;\pi_0,\mu_\theta)$, with $\mathbb E_0[g_\theta(O)]=\Delta_\theta$.
    \item $\widetilde g_\theta=g_\theta(\cdot;\hat\pi_0,\hat\mu_\theta)$ is the fitted pseudo-feature map, $\hat g_{\theta,i}=\widetilde g_\theta(O_i)$, and $\bar g_{\theta,\ell}=m^{-1}\sum_{i\in\mathcal I_\ell}\hat g_{\theta,i}$.
    \item $h_\theta(o,o')=\langle g_\theta(o),g_\theta(o')\rangle_{\mathcal H_\theta}$ is the oracle observed-data kernel; $\widetilde h_\theta$ is its fitted counterpart.
    \item $\widehat D_\theta^{\mathrm{cross}}=\langle\bar g_{\theta,1},\bar g_{\theta,2}\rangle_{\mathcal H_\theta}$ is the cross-MMD estimator. $\hat s_\theta^2$ is the empirical variance of $U_{\theta,j}=\langle\bar g_{\theta,1},\hat g_{\theta,j}\rangle_{\mathcal H_\theta}$ over $j\in\mathcal I_2$, with divisor $m$.
    \item $s_\theta^2=\operatorname{Var}_{\mathbb P_0}(\langle\Delta_\theta,g_\theta(O)\rangle_{\mathcal H_\theta})$ is the oracle projected variance; $J_\theta=D_\theta/s_\theta$ is the fixed-alternative growth coefficient when $s_\theta>0$.
    \item $\lambda>0$ is the selection regulariser. The population and empirical criteria are $J_{\lambda,\theta}=D_\theta/\sqrt{s_\theta^2+\lambda}$ and $\widehat J_{\lambda,\theta}^{\mathrm{DR}}=\widehat D_\theta^{\mathrm{cross}}/\hat s_{\theta,\lambda}$, where $\hat s_{\theta,\lambda}^2=\hat s_\theta^2+\lambda$.
    \item $\widetilde\vartheta_b$ denotes the EMA parameters and $\hat\vartheta_b$ the selected checkpoint of member $b$; $\hat\theta=(\hat\vartheta_1,\ldots,\hat\vartheta_M)$ specifies the selected ensemble.
    \item $T_\theta=\sqrt m\,\widehat D_\theta^{\mathrm{cross}}/\hat s_\theta$ is the test statistic. We reject $H_0$ at level $\alpha$ when $\hat p=1-\Phi(T_\theta)<\alpha$, where $\Phi$ is the standard normal distribution function.
\end{itemize}

\medskip
\needspace{4\baselineskip}
\noindent\textbf{Notation for the theoretical analysis}
\begin{itemize}[label=--,leftmargin=1.5em,labelsep=0.5em,itemsep=0pt,topsep=1.5pt,parsep=0pt,partopsep=0pt]
    \item $\mathcal F_{\mathrm{fit}}$ is the information used for nuisance fitting in the selection analysis; $\mathcal F_{\mathrm{learn}}$ contains all nuisance fitting, kernel training and selection information available before the final test. $\mathcal F_{\mathrm{data}}$ contains the common learning and validation data and shared training randomness, excluding member-specific randomness.
    \item $\widetilde\Delta_\theta$, $\widetilde D_\theta$ and $\widetilde s_\theta^2$ are the conditional population counterparts of $\Delta_\theta$, $D_\theta$ and $s_\theta^2$ formed from $\widetilde g_\theta$, holding the nuisance fits fixed. The label $\mathrm{or}$ marks empirical quantities computed with oracle pseudo-features.
    \item $\epsilon_{\pi,n}$ and $\epsilon_{\mu,n}$ are the conditional root-mean-square nuisance errors weighted by $|r|$: relative propensity error and conditional-embedding error, the latter uniform over $\Theta$, as defined in \eqref{eq:primitive-nuisance-errors}.
    \item $R_\Theta$ bounds the Euclidean radius of $\Theta$; $c_m(\eta)$ is the kernel-class sampling term in \eqref{eq:kernel-class-sampling-term}, with failure probability $\eta\in(0,1)$.
    \item $\kappa,W,A,B_\mu$ bound the kernel diagonal, $|r|/\pi_0$, relative propensity error and fitted embedding norm, respectively. $L$ and $L_h$ are the parameter Lipschitz constants of the fitted and oracle observed-data kernels.
    \item $\Theta_s=\{\theta\in\Theta:s_\theta^2\ge s^2\}$, for $s>0$, is the population comparison class used in the adaptation results; $\lambda_m$ is a positive sequence of selection regularisers.
    \item $b_{\mathrm{test}}=\mathbb E_0[\hat g_{\hat\theta}(O)\mid\mathcal F_{\mathrm{learn}}]$ is the fitted-score bias under $H_0$, and $\xi(O)=\hat g_{\hat\theta}(O)-b_{\mathrm{test}}$ is the centred test score.
    \item $a_{\pi,\mathrm{test}}$ and $b_{\mu,\mathrm{test}}$ are the weighted nuisance errors for the selected kernel; their product bounds $\|b_{\mathrm{test}}\|_{\mathcal H_{\hat\theta}}$ under $H_0$.
\end{itemize}

\section{Theoretical Results and Proofs}
\label{appendix:theory-proofs}

We prove the results of \Cref{chapter:method,section:deep_kernels}, then give further results on nuisance errors and the learned test under $H_1$. We  suppress the fixed-policy superscript on $g_\theta$ when unambiguous and write $O_{\mathbb P}(1)$ for quantities bounded in probability and $o_{\mathbb P}(1)$ for quantities converging to zero in probability.

\subsection{Identification of Soft-Intervention Outcome Laws}
\label{sec:soft-intervention-identification}

We show that the observational conditional outcome law can be used to recover the outcome distribution under a soft intervention. 

\begin{prop}[Soft-intervention identification]
\label{prop:soft-intervention-identification}
Under \Cref{assum:identification}, there exists a measurable probability kernel $Q$ from $\mathcal A\times\mathcal X$ to $\mathcal Y$ that is a version of the observational conditional law of $Y$ given $(A,X)$ and, for every fixed $a\in\mathcal A$, a conditional law of $Y^a$ given $X$.

For any conditional treatment density $p$ relative to $\rho$, let $A^p\mid X$ be drawn from $p(\cdot\mid X)$ independently of $U$ given $X$, and set $Y^p=F(A^p,X,U)$. Then, for every measurable $B\subseteq\mathcal Y$,
\[
P_Y^p(B)
=\int_{\mathcal X}\int_{\mathcal A}
Q(B\mid a,x)p(a\mid x)\rho(\diff a)\mathbb P_X(\diff x).
\]
If $p=0$ wherever $\pi_0=0$, up to a $\mathbb P_X\otimes\rho$-null set, this expression depends only on the observational distribution. In particular, $P_Y^\pi$ and $P_Y^{\pi'}$ are identified.
\end{prop}

\begin{proof}
    Let $P_{U\mid X}(\cdot\mid x)$ denote a conditional distribution of $U$ given $X=x$. For every measurable $B\subseteq\mathcal Y$, define
    \[
    Q(B\mid a,x)
    :=\int_{\mathcal U}\mathbbm 1\{F(a,x,u)\in B\}
    P_{U\mid X}(\diff u\mid x).
    \]
    Since $F$ is jointly measurable, $Q$ is a measurable probability kernel. We show that it describes both the observed outcome response and the fixed-treatment potential outcomes.
    
    First, for any measurable $C\subseteq\mathcal A\times\mathcal X$, consistency and $U\indep A\mid X$ give
    \begin{align*}
    \mathbb P_0\big(Y\in B,(A,X)\in C\big)
    &=\int_C\int_{\mathcal U}
    \mathbbm 1\{F(a,x,u)\in B\}
    P_{U\mid X}(\diff u\mid x)
    \mathbb P_{A,X}(\diff a,\diff x)\\
    &=\int_C Q(B\mid a,x)\mathbb P_{A,X}(\diff a,\diff x).
    \end{align*}
    Hence $Q$ is a version of the observational conditional law of $Y$ given $(A,X)$. For each fixed $a\in\mathcal A$, the definition $Y^a=F(a,X,U)$ similarly gives, for every measurable $D\subseteq\mathcal X$,
    \[
    \mathbb P(Y^a\in B,X\in D)
    =\int_D Q(B\mid a,x)\mathbb P_X(\diff x).
    \]
    Thus $Q(\cdot\mid a,X)$ is also a conditional law of $Y^a$ given $X$.
    
    We next consider a soft intervention $p$. Since $A^p$ is drawn from $p(\cdot\mid X)$ independently of $U$ given $X$, while the distribution of $(X,U)$ is unchanged, we obtain
    \begin{align*}
    P_Y^p(B)
    &=\mathbb E\!\left[\int_{\mathcal A}
    \mathbbm 1\{F(a,X,U)\in B\}
    p(a\mid X)\rho(\diff a)\right]\\
    &=\int_{\mathcal X}\int_{\mathcal A}
    Q(B\mid a,x)p(a\mid x)\rho(\diff a)\mathbb P_X(\diff x),
    \end{align*}
    where the second equality follows by conditioning on $X$ and applying Tonelli's theorem.
    
    Finally, $\mathbb P_{A,X}(\diff a,\diff x)=\pi_0(a\mid x)\rho(\diff a)\mathbb P_X(\diff x)$. If $p=0$ wherever $\pi_0=0$, up to a $\mathbb P_X\otimes\rho$-null set, then the law of $(A^p,X)$ is absolutely continuous with respect to $\mathbb P_{A,X}$. For each measurable $B$, any two versions of the observational conditional probability $Q(B\mid a,x)$ agree $\mathbb P_{A,X}$-almost everywhere, and therefore also under the intervention. They consequently give the same g-formula, which is thus determined by the observational law. The required support condition holds for both $\pi$ and $\pi'$ by \Cref{assum:identification}, completing the proof.
\end{proof}

\paragraph{Mean embeddings and integrability.}
For fixed $\theta$, define $\mu_\theta(a,x)=\int\varphi_\theta(y)Q(\diff y\mid a,x)$ wherever $\int\sqrt{k_\theta(y,y)}Q(\diff y\mid a,x)<\infty$, and set it to zero elsewhere. This exceptional set is measurable and null under both $\mathbb P_{A,X}$ and the two target treatment-covariate measures by \Cref{assum:embedding-integrability} and parts (b)-(d). Separability and strongly measurable approximation under kernel integration give a strongly measurable $\mu_\theta$. Part (b) identifies $\mu_\theta(A,X)$ with the observational Bochner conditional mean, and part (d) with Bochner Fubini gives
\[
\chi_\theta(p)=\int_{\mathcal X}\int_{\mathcal A}\mu_\theta(a,x)p(a\mid x)\rho(\diff a)\mathbb P_X(\diff x).
\]
The observed-law first moment in \Cref{assum:embedding-integrability} is needed for the observational conditional mean, since target moments need not control actions outside the target support. Uniformly bounded kernels satisfy all these moment conditions.

For finite or countable $\mathcal A$ with counting measure $\rho$, the action integrals reduce to sums, giving the usual randomised mixture of atomic interventions.

\subsection{Observed-Data Representation and Double Robustness}
\label{sec:proof-dr-representation}

For Proposition~\ref{prop:dr-representation}, the candidate nuisances are assumed to satisfy
\begin{align*}
\mathbb E_0\left[\frac{|r(A,X)|}{e(A\mid X)}\|\varphi_\theta(Y)-m_\theta(A,X)\|_{\mathcal H_\theta}\right]&<\infty,\\
\mathbb E_0\left[\int_{\mathcal A}|r(a,X)|\|m_\theta(a,X)\|_{\mathcal H_\theta}\rho(\diff a)\right]&<\infty.
\end{align*}
These conditions make both terms of the pseudo-feature Bochner integrable and justify the conditional integration below. The target moments also ensure $\mathbb E_0\int|r(a,X)|\|\mu_\theta(a,X)\|_{\mathcal H_\theta}\rho(\diff a)<\infty$.

\begin{proof}[Proof of Proposition~\ref{prop:dr-representation}]

Recall that $r(a,x)=\pi(a\mid x)-\pi'(a\mid x)$. As in \Cref{chapter:method}, ratios and integrands weighted by $r$ or $|r|$ are set to zero wherever $r=0$.

Under \Cref{assum:identification,assum:embedding-integrability}, Proposition~\ref{prop:soft-intervention-identification} and Bochner Fubini identify each target embedding through the common response kernel. Subtracting the two embeddings gives
\[
\Delta_\theta(\pi,\pi')
=
\mathbb E_{\mathbb P_0}\left[
\int_{\mathcal A}r(a,X)\mu_\theta(a,X)\rho(\diff a)
\right].
\]
The weight $r(A,X)/e(A\mid X)$ is fixed conditional on $(A,X)$, and $\mu_\theta(A,X)$ agrees almost surely with the observational conditional mean. The integrability conditions above therefore give
\begin{align*}
\mathbb E_{\mathbb P_0}\left[
\frac{r(A,X)}{e(A\mid X)}
\{\varphi_\theta(Y)-m_\theta(A,X)\}\mid A,X
\right] =\frac{r(A,X)}{e(A\mid X)}
\{\mu_\theta(A,X)-m_\theta(A,X)\}.
\end{align*}
Averaging over $A\mid X$, with density $\pi_0$, gives
\begin{align*}
&\mathbb E_{\mathbb P_0}\left[g_\theta^{\pi,\pi'}(O;e,m_\theta)\right]\\
&\quad=
\mathbb E_{\mathbb P_0}\left[
\int_{\mathcal A}
\pi_0(a\mid X)\frac{r(a,X)}{e(a\mid X)}
\{\mu_\theta(a,X)-m_\theta(a,X)\}
\rho(\diff a)
\right]\\
&\qquad+
\mathbb E_{\mathbb P_0}\left[
\int_{\mathcal A}r(a,X)m_\theta(a,X)\rho(\diff a)
\right].
\end{align*}
We now subtract the expression for $\Delta_\theta(\pi,\pi')$:
\begin{align*}
&\mathbb E_{\mathbb P_0}\left[g_\theta^{\pi,\pi'}(O;e,m_\theta)\right]
-\Delta_\theta(\pi,\pi')\\
&\quad=
\mathbb E_{\mathbb P_0}\left[
\int_{\mathcal A}
r(a,X)
\left\{\frac{\pi_0(a\mid X)}{e(a\mid X)}-1\right\}
\{\mu_\theta(a,X)-m_\theta(a,X)\}
\rho(\diff a)
\right].
\end{align*}
This expression is zero if either $e=\pi_0$ or $m_\theta=\mu_\theta$ wherever $|r|>0$. Thus, if at least one nuisance function is correct,
\[
\mathbb E_{\mathbb P_0}\left[g_\theta^{\pi,\pi'}(O;e,m_\theta)\right] = \Delta_\theta(\pi,\pi').
\]
Using the true nuisance functions gives $\mathbb E_{\mathbb P_0}[g_\theta^{\pi,\pi'}(O)]=\Delta_\theta(\pi,\pi')$.
\end{proof}

\begin{rem}
The representation above covers continuous actions for soft interventions $\pi,\pi'$, but not hard interventions at a fixed treatment value. Atomic continuous-treatment interventions require localization in treatment space, with the resulting smoothing bias controlled explicitly \citep{Colangelo2020,zenati25a,bhattacharjee2025doubly}. This remains true when the outcome kernel is learned.
\end{rem}

\subsection{Fixed-Alternative Growth Rate}
\label{sec:proof-fixed-alternative}

\begin{proof}[Proof of \Cref{prop:fixed-kernel-efficacy}]
Fix $\theta$. We take all inner products and norms in $\mathcal H_\theta$. We assume the statistics use the true nuisance functions, so $\hat g_{\theta,i}^{\pi,\pi'}=g_\theta^{\pi,\pi'}(O_i)$ and
\[
\bar g_{\theta,\ell}
=\frac1m\sum_{i\in\mathcal I_\ell}g_\theta^{\pi,\pi'}(O_i),
\qquad \ell\in\{1,2\}.
\]
By Proposition~\ref{prop:dr-representation}, $\mathbb E_{\mathbb P_0}[g_\theta^{\pi,\pi'}(O_i)]=\Delta_\theta(\pi,\pi')$. The fourth-moment assumption gives finite first and second moments, so the RKHS law of large numbers gives $\bar g_{\theta,\ell}\xrightarrow{\mathbb P}\Delta_\theta(\pi,\pi')$ for $\ell\in\{1,2\}$. Continuity of the inner product gives
\[
\widehat D_\theta^{\mathrm{cross}}
=\langle\bar g_{\theta,1},\bar g_{\theta,2}\rangle
\xrightarrow{\mathbb P}
\|\Delta_\theta(\pi,\pi')\|^2
=D_\theta.
\]

For the variance estimator, define the real-valued variables
\[
V_j:=\langle\bar g_{\theta,1},g_\theta^{\pi,\pi'}(O_j)\rangle,
\qquad
W_j:=\langle\Delta_\theta(\pi,\pi'),g_\theta^{\pi,\pi'}(O_j)\rangle,
\qquad j\in\mathcal I_2.
\]
The $V_j$ have sample mean $\widehat D_\theta^{\mathrm{cross}}$. Expanding their sample variance gives
\[
\hat s_\theta^2
=\frac1m\sum_{j\in\mathcal I_2}V_j^2
-\left(\widehat D_\theta^{\mathrm{cross}}\right)^2.
\]
We have already shown that $(\widehat D_\theta^{\mathrm{cross}})^2\xrightarrow{\mathbb P}D_\theta^2$. To find the limit of $\frac1m\sum_{j\in\mathcal I_2}V_j^2$, we compare $V_j$ with $W_j$. By $|u^2-w^2|=|u-w||u+w|$ and Cauchy-Schwarz,
\begin{align*}
\left|\frac1m\sum_{j\in\mathcal I_2}(V_j^2-W_j^2)\right|
&\le
\|\bar g_{\theta,1}-\Delta_\theta(\pi,\pi')\|
\bigl(\|\bar g_{\theta,1}\|+\|\Delta_\theta(\pi,\pi')\|\bigr)
\frac1m\sum_{j\in\mathcal I_2}\|g_\theta^{\pi,\pi'}(O_j)\|^2\\
&=o_{\mathbb P}(1).
\end{align*}
Here the first factor converges to zero, the second is bounded in probability, and the last converges to $\mathbb E_{\mathbb P_0}\|g_\theta^{\pi,\pi'}(O)\|^2<\infty$ by the law of large numbers. The same law applied to $W_j^2$ gives
\[
\frac1m\sum_{j\in\mathcal I_2}W_j^2
\xrightarrow{\mathbb P}
\mathbb E_{\mathbb P_0}\left[
\langle\Delta_\theta(\pi,\pi'),g_\theta^{\pi,\pi'}(O)\rangle^2
\right].
\]
We also have $\mathbb E_{\mathbb P_0}[W_j]=\langle\Delta_\theta(\pi,\pi'),\Delta_\theta(\pi,\pi')\rangle=D_\theta$. Combining these gives
\[
\hat s_\theta^2
\xrightarrow{\mathbb P}
\mathbb E_{\mathbb P_0}\left[
\langle\Delta_\theta(\pi,\pi'),g_\theta^{\pi,\pi'}(O)\rangle^2
\right]-D_\theta^2
=s_\theta^2.
\]
Since $s_\theta>0$, we have $\hat s_\theta\xrightarrow{\mathbb P}s_\theta$. By \eqref{eq:test-statistic} and Slutsky's theorem,
\[
\frac{T_\theta}{\sqrt m}
=\frac{\widehat D_\theta^{\mathrm{cross}}}{\hat s_\theta}
\xrightarrow{\mathbb P}
\frac{D_\theta}{s_\theta}
=J_\theta.
\]
We set $T_\theta=0$ when $\hat s_\theta=0$. Since this event has probability tending to zero, it does not affect the limit.
\end{proof}

\subsection{Characteristic Ensemble Safeguard}
\label{sec:proof-characteristic}

This result is the analogue of \citet[Proposition~5]{liu2021learningdeepkernelsnonparametric}. Their kernel in Eq.~(1), however, additionally factors the characteristic component and the deep neural component.

\begin{proof}[Proof of \Cref{lem:characteristic-ensemble}]
Let $Y,Y'\sim P$ and $Z,Z'\sim Q$ be mutually independent. Then
\[
\operatorname{MMD}_k^2(P,Q)
=\mathbb E[k(Y,Y')]+\mathbb E[k(Z,Z')]-2\mathbb E[k(Y,Z)].
\]
The squared MMD is therefore linear in the kernel. For each ensemble member $b$, write
\[
k_{\vartheta_b}^{\mathrm{deep}}(y,y')
:=\kappa_{\vartheta_b}\bigl(\zeta_{\vartheta_b}(y),\zeta_{\vartheta_b}(y')\bigr).
\]
Substituting the ensemble kernel form in \eqref{eq:deep-kernel} yields
\begin{align*}
\operatorname{MMD}_{k_\theta}^2(P,Q)
&=\frac1M\sum_{b=1}^M
\left[
(1-w_b)\operatorname{MMD}_{k_{\vartheta_b}^{\mathrm{deep}}}^2(P,Q)
+w_b\operatorname{MMD}_{k_{\vartheta_b}^{\mathrm{char}}}^2(P,Q)
\right].
\end{align*}
Every squared MMD is nonnegative, and $1-w_b \ge 0$. Dropping the deep kernel terms and using $w_b \ge w_{\min}$ gives
\[
\operatorname{MMD}_{k_\theta}^2(P,Q)
\ge\frac{w_{\min}}M\sum_{b=1}^M
\operatorname{MMD}_{k_{\vartheta_b}^{\mathrm{char}}}^2(P,Q).
\]
Each base kernel is characteristic and $w_{\min}>0$, so the lower bound is strictly positive when $P\neq Q$. When $P=Q$, every squared MMD term is zero. Thus $k_\theta$ is characteristic.
\end{proof}

\subsection{Memberwise Ensemble Aggregation}
\label{sec:memberwise-ensemble}

\begin{prop}[Memberwise ensemble aggregation]\label{prop:memberwise-ensemble}
Suppose \Cref{assum:identification,assum:embedding-integrability} hold. Let $k_{\vartheta_1},\ldots,k_{\vartheta_M}$ be fixed kernel members and $k_\theta=M^{-1}\sum_b k_{\vartheta_b}$. Use the oracle pseudo-features $g_{\vartheta_b}=g_{\vartheta_b}(\cdot;\pi_0,\mu_{\vartheta_b})$ based on the common response kernel $Q$, and assume each projection $\langle\Delta_{\vartheta_b},g_{\vartheta_b}(O)\rangle_{\mathcal H_{\vartheta_b}}$ is square-integrable. For $\lambda\ge0$, define $J_{\lambda,\vartheta_b}=D_{\vartheta_b}/\sqrt{s_{\vartheta_b}^2+\lambda}$ and analogously for the aggregate whenever the denominators are positive. Then
\[
D_\theta=\frac1M\sum_{b=1}^M D_{\vartheta_b},
\qquad
\sqrt{s_\theta^2+\lambda}\le\frac1M\sum_{b=1}^M\sqrt{s_{\vartheta_b}^2+\lambda}.
\]
If all member and aggregate denominators are positive, then
\[
J_{\lambda,\theta}
\ge\frac{\sum_{b=1}^M D_{\vartheta_b}}{\sum_{b=1}^M\sqrt{s_{\vartheta_b}^2+\lambda}}
\ge\min_{1\le b\le M}J_{\lambda,\vartheta_b}.
\]
For learned parameters $\hat\theta=(\hat\vartheta_1,\ldots,\hat\vartheta_M)$, let $\mathcal F_{\mathrm{data}}$ contain the common learning and validation data and training randomness. If the selected member kernels are conditionally independent given $\mathcal F_{\mathrm{data}}$ and the $D_{\hat\vartheta_b}$ have finite conditional second moments, then
\[
\operatorname{Var}(D_{\hat\theta}\mid\mathcal F_{\mathrm{data}})
=\frac1{M^2}\sum_{b=1}^M\operatorname{Var}(D_{\hat\vartheta_b}\mid\mathcal F_{\mathrm{data}}).
\]
If the members are also conditionally identically distributed, this equals $M^{-1}\operatorname{Var}(D_{\hat\vartheta_1}\mid\mathcal F_{\mathrm{data}})$.
\end{prop}

\begin{proof}
In the Hilbert direct sum $\bigoplus_{b=1}^M\mathcal H_{\vartheta_b}$, the averaged kernel has the isometric feature representation
\[
\varphi_\theta(y)=M^{-1/2}\bigl(\varphi_{\vartheta_1}(y),\ldots,\varphi_{\vartheta_M}(y)\bigr).
\]
Linearity of Bochner integration gives $\Delta_\theta=M^{-1/2}(\Delta_{\vartheta_1},\ldots,\Delta_{\vartheta_M})$ and the same identity for the conditional embeddings. Using the common $\pi_0$ and $Q$ in \eqref{eq:generic-dr-feature} therefore gives $g_\theta(O)=M^{-1/2}(g_{\vartheta_1}(O),\ldots,g_{\vartheta_M}(O))$. Taking squared norms proves the discrepancy identity, and
\[
\langle\Delta_\theta,g_\theta(O)\rangle_{\mathcal H_\theta}
=\frac1M\sum_{b=1}^M\langle\Delta_{\vartheta_b},g_{\vartheta_b}(O)\rangle_{\mathcal H_{\vartheta_b}}.
\]
Let $Z_b:=\langle\Delta_{\vartheta_b},g_{\vartheta_b}(O)\rangle_{\mathcal H_{\vartheta_b}}-D_{\vartheta_b}$. By Proposition~\ref{prop:dr-representation}, $\mathbb E_0Z_b=0$, so $s_{\vartheta_b}=\|Z_b\|_{L^2(\mathbb P_0)}$ and $s_\theta=\|M^{-1}\sum_b Z_b\|_{L^2(\mathbb P_0)}$. The triangle inequality in $L^2(\mathbb P_0)\oplus\mathbb R$ yields
\[
\sqrt{s_\theta^2+\lambda}
=\left\|\frac1M\sum_{b=1}^M(Z_b,\sqrt\lambda)\right\|
\le\frac1M\sum_{b=1}^M\|(Z_b,\sqrt\lambda)\|
=\frac1M\sum_{b=1}^M\sqrt{s_{\vartheta_b}^2+\lambda}.
\]
Since $D_{\vartheta_b}\ge0$, the denominator bound proves the first efficacy inequality. The intervening ratio is a weighted average of the member efficacies with positive weights proportional to $\sqrt{s_{\vartheta_b}^2+\lambda}$, proving the second. No independence of the projected scores on a common observation is required. Finally, the discrepancy identity holds for every realised collection of learned kernels; conditional independence eliminates its cross-covariances, and identical conditional distributions give the $1/M$ formula.
\end{proof}

\subsection{Held-Out Null Validity}
\label{sec:proof-heldout-validity}

The sample-splitting principle detailed below also appears in \citet{gretton2012optimal,liu2021learningdeepkernelsnonparametric}. Here we additionally control bias from nuisance estimation.

Let $\mathcal F_{\mathrm{learn}}$ contain the data and randomness used for nuisance fitting, kernel training and selection. Conditional on $\mathcal F_{\mathrm{learn}}$, the fitted map $\hat g_{\hat\theta}(O):=g_{\hat\theta}^{\pi,\pi'}(O;\hat\pi_0,\hat\mu_{\hat\theta})$ is fixed. Write $m$ for the size of each test half. For a new observation $O\sim\mathbb P_0$ independent of $\mathcal F_{\mathrm{learn}}$, define
\begin{align*}
b_{\mathrm{test}}
&:=\mathbb E_{\mathbb P_0}[\hat g_{\hat\theta}(O)\mid\mathcal F_{\mathrm{learn}}],\\
\xi(O)&:=\hat g_{\hat\theta}(O)-b_{\mathrm{test}},
\qquad
\bar\xi_\ell:=\frac1m\sum_{i\in\mathcal I_\ell}\xi(O_i),
\qquad \ell\in\{1,2\}.
\end{align*}
$b_{\mathrm{test}}$ is the bias from nuisance estimation under $H_0$.

Let $\mathcal G_m$ denote the information from the learning procedure $\mathcal F_{\mathrm{learn}}$ and the first test half. $\bar\xi_1$ is therefore fixed when conditioned on $\mathcal G_m$. For a new observation $O\sim\mathbb P_0$ independent of $\mathcal G_m$, define
\[
V(O):=\langle\bar\xi_1,\xi(O)\rangle_{\mathcal H_{\hat\theta}},
\qquad
v_{\mathrm{test}}:=\mathbb E_{\mathbb P_0}[V(O)^2\mid\mathcal G_m].
\]
Since $V(O)$ has conditional mean zero, $v_{\mathrm{test}}$ is its variance conditioned on $\mathcal G_m$.

\begin{assum}[Held-out regularity]
\label{assum:heldout-regularity}
Let $m\to\infty$. Suppose the two test halves are mutually independent i.i.d. samples from $\mathbb P_0$ and are independent of $\mathcal F_{\mathrm{learn}}$. Assume that $\hat g_{\hat\theta}$ is strongly measurable and conditionally Bochner integrable. Under $H_0$, suppose that:
\begin{enumerate}
    \item The conditional fourth moment of the centred pseudo-feature is bounded in probability:
    \[
    \mathbb E_{\mathbb P_0}[\|\xi(O)\|_{\mathcal H_{\hat\theta}}^4\mid\mathcal F_{\mathrm{learn}}]
    =O_{\mathbb P}(1).
    \]
    \item The projected variance is nonzero with probability tending to one, $\mathbb P(v_{\mathrm{test}}>0)\to1$, and
    \[
    (m v_{\mathrm{test}})^{-1}=O_{\mathbb P}(1),
    \qquad
    \frac{\mathbb E_{\mathbb P_0}[V(O)^4\mid\mathcal G_m]}{m v_{\mathrm{test}}^2}
    \xrightarrow{\mathbb P}0.
    \]
    The inverse and ratio are defined to be zero when $v_{\mathrm{test}}=0$.
    \item The conditional mean tends to zero faster than $m^{-1/2}$:
    \[
    \sqrt m\|b_{\mathrm{test}}\|_{\mathcal H_{\hat\theta}}
    \xrightarrow{\mathbb P}0.
    \]
\end{enumerate}
\end{assum}

Double robustness gives a sufficient nuisance-rate condition for condition 3 of \Cref{assum:heldout-regularity}. Under \Cref{assum:identification,assum:embedding-integrability}, define the weighted nuisance errors
\begin{align*}
a_{\pi,\mathrm{test}}^2
&:=\mathbb E_{\mathbb P_0}\left[
\int_{\mathcal A}|r(a,X)|
\left\{\frac{\pi_0(a\mid X)}{\hat\pi_0(a\mid X)}-1\right\}^2
\rho(\diff a)
\middle|\mathcal F_{\mathrm{learn}}
\right],\\
b_{\mu,\mathrm{test}}^2
&:=\mathbb E_{\mathbb P_0}\left[
\int_{\mathcal A}|r(a,X)|
\|\hat\mu_{\hat\theta}(a,X)-\mu_{\hat\theta}(a,X)\|_{\mathcal H_{\hat\theta}}^2
\rho(\diff a)
\middle|\mathcal F_{\mathrm{learn}}
\right].
\end{align*}
Under $H_0$, $\Delta_{\hat\theta}(\pi,\pi')=0$. If $\hat\pi_0$ is positive on the support of $|r|$ and the terms in the double robustness identity are integrable, taking norms in the identity of Proposition~\ref{prop:dr-representation} and applying Cauchy-Schwarz with weight $|r(a,X)|$ gives
\begin{equation}
\label{eq:heldout-bias-product}
\|b_{\mathrm{test}}\|_{\mathcal H_{\hat\theta}}
\le a_{\pi,\mathrm{test}}b_{\mu,\mathrm{test}}.
\end{equation}
$a_{\pi,\mathrm{test}}b_{\mu,\mathrm{test}}=o_{\mathbb P}(m^{-1/2})$ is therefore sufficient for condition 3.

\begin{proof}[Proof of \Cref{thm:heldout-validity}]Write $b=b_{\mathrm{test}}$ and $v=v_{\mathrm{test}}$. All inner products and norms are in $\mathcal H_{\hat\theta}$.

Let $V_j:=V(O_j)$ for $j\in\mathcal I_2$, and define
\begin{align*}
\widehat D^{\mathrm c}
&:=\frac1m\sum_{j\in\mathcal I_2}V_j
=\langle\bar\xi_1,\bar\xi_2\rangle,\\
(\hat s^{\mathrm c})^2
&:=\frac1m\sum_{j\in\mathcal I_2}V_j^2-(\widehat D^{\mathrm c})^2,
\qquad
T^{\mathrm c}:=\frac{\sqrt m\widehat D^{\mathrm c}}{\hat s^{\mathrm c}}.
\end{align*}
Set $T^{\mathrm c}=0$ when $\hat s^{\mathrm c}=0$.

$V_j$ are i.i.d., conditional on $\mathcal G_m$, with mean zero and variance $v$. On $\{v>0\}$, the fourth-moment ratio in condition 2 is the Lyapunov condition, so $\sqrt m\widehat D^{\mathrm c}/\sqrt v$ has a conditional standard normal limit. The sample-variance identity above and conditional Chebyshev and Markov inequalities give, for every $\delta>0$,
\[
\mathbb P\left(
\left|\frac{(\hat s^{\mathrm c})^2}{v}-1\right|>\delta
\middle|\mathcal G_m
\right)
\le\frac{4\mathbb E_{\mathbb P_0}[V(O)^4\mid\mathcal G_m]}{\delta^2 m v^2}
+\frac{2}{m\delta}
\xrightarrow{\mathbb P}0.
\]
Here the two bounds control the sample second moment and the squared sample mean, using $\mathbb E_{\mathbb P_0}[(\widehat D^{\mathrm c})^2/v\mid\mathcal G_m]=1/m$. Applying Slutsky's theorem conditionally gives
\[
d_m:=\sup_{t\in\mathbb R}
\left|\mathbb P(T^{\mathrm c}\le t\mid\mathcal G_m)-\Phi(t)\right|
\xrightarrow{\mathbb P}0.
\]
Since $\mathbb P(v>0)\to1$ and $(\hat s^{\mathrm c})^2/v\xrightarrow{\mathbb P}1$, we have $\mathbb P(\hat s^{\mathrm c}>0)\to1$. On $\{\hat s^{\mathrm c}>0\}$, $(m v)^{-1}=O_{\mathbb P}(1)$ also gives $(\hat s^{\mathrm c})^{-1}=O_{\mathbb P}(\sqrt m)$.

Conditional independence and centring give
\[
\mathbb E_{\mathbb P_0}[\|\bar\xi_\ell\|^2\mid\mathcal F_{\mathrm{learn}}]
=\frac1m\mathbb E_{\mathbb P_0}[\|\xi(O)\|^2\mid\mathcal F_{\mathrm{learn}}]
=O_{\mathbb P}(m^{-1}).
\]
By condition 1 and conditional Markov's inequality,
\[
\|\bar\xi_\ell\|=O_{\mathbb P}(m^{-1/2}),
\qquad
\frac1m\sum_{j\in\mathcal I_2}\|\xi(O_j)\|^2=O_{\mathbb P}(1).
\]
Since the fitted half-sample means are $\bar\xi_\ell+b$, condition 3 and Cauchy-Schwarz give
\begin{align*}
\left|\widehat D_{\hat\theta}^{\mathrm{cross}}-\widehat D^{\mathrm c}\right|
&=\left|\langle b,\bar\xi_1+\bar\xi_2\rangle+\|b\|^2\right|\\
&\le\|b\|\bigl(\|\bar\xi_1\|+\|\bar\xi_2\|\bigr)+\|b\|^2
=o_{\mathbb P}(m^{-1}).
\end{align*}
Applying the reverse triangle inequality to the residual vectors and Cauchy-Schwarz gives
\begin{align*}
|\hat s_{\hat\theta}-\hat s^{\mathrm c}|
&\le\|b\|\left\{
\frac1m\sum_{j\in\mathcal I_2}\|\xi(O_j)-\bar\xi_2\|^2
\right\}^{1/2}\\
&\le\|b\|\left\{
\frac1m\sum_{j\in\mathcal I_2}\|\xi(O_j)\|^2
\right\}^{1/2}
=o_{\mathbb P}(m^{-1/2}).
\end{align*}
On $\{\hat s^{\mathrm c}>0\}$, combining these bounds with $(\hat s^{\mathrm c})^{-1}=O_{\mathbb P}(\sqrt m)$ yields
\[
\frac{\sqrt m(\widehat D_{\hat\theta}^{\mathrm{cross}}-\widehat D^{\mathrm c})}{\hat s^{\mathrm c}}
=o_{\mathbb P}(1),
\qquad
\frac{\hat s_{\hat\theta}}{\hat s^{\mathrm c}}\xrightarrow{\mathbb P}1.
\]
Hence $\mathbb P(\hat s_{\hat\theta}>0)\to1$. Since $T^{\mathrm c}=O_{\mathbb P}(1)$, the fitted statistic satisfies $T_{\hat\theta}-T^{\mathrm c}=o_{\mathbb P}(1)$.

Since $0\le d_m\le1$ and $d_m\xrightarrow{\mathbb P}0$, we have $\mathbb E[d_m]\to0$. Therefore $\mathbb E[d_m\mid\mathcal F_{\mathrm{learn}}]\xrightarrow{\mathbb P}0$. For any fixed $\delta>0$,
\[
q_m(\delta):=\mathbb P\left(
|T_{\hat\theta}-T^{\mathrm c}|>\delta
\mid\mathcal F_{\mathrm{learn}}\right)
\xrightarrow{\mathbb P}0,
\]
by Markov's inequality, since $\mathbb E[q_m(\delta)]=\mathbb P(|T_{\hat\theta}-T^{\mathrm c}|>\delta)\to0$.

Averaging over the first test half gives the bound $\mathbb E[d_m\mid\mathcal F_{\mathrm{learn}}]$ for the normal-approximation error of $T^{\mathrm c}$. The fitted and centred statistics differ by at most $\delta$, except with conditional probability $q_m(\delta)$. Comparing their distribution functions at $t-\delta$ and $t+\delta$ gives
\begin{align*}
&\sup_{t\in\mathbb R}
\left|\mathbb P(T_{\hat\theta}\le t\mid\mathcal F_{\mathrm{learn}})-\Phi(t)\right|\\
&\quad\le\mathbb E[d_m\mid\mathcal F_{\mathrm{learn}}]
+q_m(\delta)+\frac{\delta}{\sqrt{2\pi}},
\end{align*}
where the last term follows because the standard normal density is bounded by $1/\sqrt{2\pi}$. Letting $m\to\infty$ and then $\delta\to0$ proves \eqref{eq:heldout-null-limit}. The left-hand side is bounded by one, so taking expectations gives the unconditional normal limit and asymptotic level $\alpha$. This argument also covers a fixed kernel by taking $\hat\theta=\theta$ deterministically and retaining nuisance-fitting data and randomness in $\mathcal F_{\mathrm{learn}}$.
\end{proof}

\subsection{Uniform Control of the Learning Criterion}\label{sec:uniform-criterion-theory}

We prove \Cref{thm:feasible-criterion-uniform} by separating sampling error from nuisance error. Throughout this subsection, expectations, variances and norms involving fitted functions are conditional on $\mathcal F_{\mathrm{fit}}$, the data and randomness used to fit the nuisance estimators. This fixes the fitted functions for every $\theta$. Write $n$ for the fitting-sample size and $m$ for the size of each selection half. The two halves are mutually independent i.i.d. samples from $\mathbb P_0$ and are independent of $\mathcal F_{\mathrm{fit}}$. For a separable Hilbert space $\mathcal H$, let $\|f\|_{L^2(\mathbb P_0;\mathcal H)}^2:=\mathbb E_0\|f(O)\|_{\mathcal H}^2$.

Write $\mathcal O:=\mathcal X\times\mathcal A\times\mathcal Y$. For the ensemble in \eqref{eq:deep-kernel}, let $\Theta=\Xi^M$, where $\Xi\subseteq\mathbb R^{d_0}$ is the member parameter set and $d_\Theta=Md_0$. For the memberwise optimisation in \Cref{alg:dr-dk-causal-mmd}, we apply the learning bound with $M=1$, $\Theta=\Xi$ and $d_\Theta=d_0$. For $\vartheta\in\Xi$, write $k_\vartheta^{\mathrm{deep}}$ for the learned component and $w_\vartheta$ for its mixture weight, with $w_{\vartheta_b}=w_b$.

Let $\widetilde g_\theta:=g_\theta(\cdot;\hat\pi_0,\hat\mu_\theta)$ denote the fitted pseudo-feature. Its conditional population quantities are
\begin{align*}
\widetilde\Delta_\theta&:=\mathbb E_0[\widetilde g_\theta(O)],
\qquad
\widetilde D_\theta:=\|\widetilde\Delta_\theta\|_{\mathcal H_\theta}^2,\\
\widetilde s_\theta^2&:=\operatorname{Var}_{\mathbb P_0}
\bigl(\langle\widetilde\Delta_\theta,\widetilde g_\theta(O)\rangle_{\mathcal H_\theta}\bigr).
\end{align*}
A subscript $\mathrm{or}$ denotes an empirical quantity computed with oracle pseudo-features.

\paragraph{Measurable versions.}
We use the common response kernel from Proposition~\ref{prop:soft-intervention-identification}:
\begin{equation}
\label{eq:conditional-outcome-embedding}
\mu_\theta(a,x)=\int_{\mathcal Y}\varphi_\theta(y)Q(\diff y\mid a,x).
\end{equation}
The measurable null-set convention in \Cref{sec:soft-intervention-identification} applies whenever only first moments are assumed. Under the uniform diagonal bound below, the integral exists for every $(\theta,a,x)$ and satisfies $\|\mu_\theta(a,x)\|_{\mathcal H_\theta}\le\sqrt\kappa$; no parameter-dependent exceptional sets are needed.

Write the estimators from \Cref{chapter:method} in terms of the fitted observed-data kernel $\widetilde h_\theta(o,o'):=\langle\widetilde g_\theta(o),\widetilde g_\theta(o')\rangle_{\mathcal H_\theta}$. Set $U_{\theta,j}:=m^{-1}\sum_{i\in\mathcal I_1}\widetilde h_\theta(O_i,O_j)$. Then
\begin{equation}
\label{eq:cross-mmd-estimator}
\widehat D_\theta^{\mathrm{cross}}
=\frac1m\sum_{j\in\mathcal I_2}U_{\theta,j},
\end{equation}
and
\begin{equation}
\label{eq:cross-variance-estimator}
\hat s_\theta^2
=\frac1m\sum_{j\in\mathcal I_2}U_{\theta,j}^2
-\bigl(\widehat D_\theta^{\mathrm{cross}}\bigr)^2.
\end{equation}

\begin{assum}[Selection regularity]\label{assum:selection-regularity}
The set $\Theta\subset\mathbb R^{d_\Theta}$ is compact and lies in a Euclidean ball of radius $R_\Theta$. There are fixed $\kappa,W,A,B_\mu,L<\infty$ such that, almost surely over the nuisance-fitting sample, $k_\theta(y,y)\le\kappa$ for every $(\theta,y)$ and, outside one common $(\mathbb P_X\otimes\rho)$-null set, whenever $|r(a,x)|>0$, $\pi_0(a\mid x),\hat\pi_0(a\mid x)>0$ and, for every $\theta\in\Theta$,
\begin{equation}
\label{eq:main-selection-regularity}
 \frac{|r(a,x)|}{\pi_0(a\mid x)}\le W,
\quad
\left|\frac{\pi_0(a\mid x)}{\hat\pi_0(a\mid x)}-1\right|\le A,
\quad
\|\hat\mu_\theta(a,x)\|_{\mathcal H_\theta}\le B_\mu.
\end{equation}
For every $\theta$, $\hat\mu_\theta$ and $\widetilde g_\theta$ are strongly measurable. Moreover, $\widetilde h_\theta(o,o'):=\langle\widetilde g_\theta(o),\widetilde g_\theta(o')\rangle_{\mathcal H_\theta}$ is jointly measurable and, almost surely over the nuisance-fitting sample, on a common $\mathbb P_0^2$-full set,
\begin{equation}
\label{eq:main-fitted-class-condition}
|\widetilde h_\theta(o,o')-\widetilde h_{\theta'}(o,o')|
\le L\|\theta-\theta'\|_2.
\end{equation}
\end{assum}

Conditional on $\mathcal F_{\mathrm{fit}}$, define the nuisance errors
\begin{equation}
\label{eq:primitive-nuisance-errors}
\begin{aligned}
\epsilon_{\pi,n}^2
&:=
\mathbb E_0\left[
\int_{\mathcal A}|r(a,X)|
\left\{
\frac{\pi_0(a\mid X)}{\hat\pi_0(a\mid X)}-1
\right\}^{2}
\rho(\diff a)
\right],\\
\epsilon_{\mu,n}^2
&:=
\sup_{\theta\in\Theta}
\mathbb E_0\left[
\int_{\mathcal A}|r(a,X)|
\|\hat\mu_\theta(a,X)-\mu_\theta(a,X)\|_{\mathcal H_\theta}^2
\rho(\diff a)
\right].
\end{aligned}
\end{equation}
We assume that $\epsilon_{\mu,n}^2$ is measurable with respect to $\mathcal F_{\mathrm{fit}}$.

For $\lambda>0$, the empirical and fitted-population criteria are continuous on compact $\Theta$ under the measurability and Lipschitz assumptions on $\widetilde h_\theta$. We can take the supremum of their absolute difference over a countable dense subset of $\Theta$, so it is measurable. Measurability of the error relative to $J_{\lambda,\theta}$ follows from continuity of $f\mapsto\sup_{\theta\in\Theta}|f(\theta)-J_{\lambda,\theta}|$ on $C(\Theta)$. This map is Lipschitz in the supremum norm because $J_{\lambda,\theta}$ is deterministic and bounded.

For a failure probability $\eta\in(0,1)$, define the kernel-class sampling term
\begin{equation}
\label{eq:kernel-class-sampling-term}
c_m(\eta)
:=
\frac{
\sqrt{d_\Theta\log(1+2R_\Theta\sqrt m)+\log(4/\eta)}+1
}{\sqrt m}.
\end{equation}

\subsubsection{Kernel-Class Regularity}

We give sufficient conditions for regularity of $k_\theta$, the oracle kernel $h_\theta$ and the fitted kernel $\widetilde h_\theta$.

For $o,o'\in\mathcal O$, define the oracle induced observed-data kernel
\begin{equation}
\label{eq:induced-causal-kernel}
h_\theta(o,o')
:=
\langle
g_\theta^{\pi,\pi'}(o),
g_\theta^{\pi,\pi'}(o')
\rangle_{\mathcal H_\theta}.
\end{equation}

\begin{assum}[Oracle induced-kernel class]\label{assum:oracle-induced-class}
The parameter set $\Theta$ is compact and contained in the Euclidean ball of radius $R_\Theta$. There exist a measurable set $\mathcal O_0\subseteq\mathcal O$ with $\mathbb P_0(\mathcal O_0)=1$ and constants $B_g,L_h<\infty$ such that, simultaneously for every $o,o'\in\mathcal O_0$ and $\theta,\theta'\in\Theta$,
\[
\|g_\theta^{\pi,\pi'}(o)\|_{\mathcal H_\theta}\le B_g,
\qquad
|h_\theta(o,o')-h_{\theta'}(o,o')|
\le
L_h\|\theta-\theta'\|_2.
\]
The scalar map $(\theta,o,o')\mapsto h_\theta(o,o')$ is jointly measurable on $\Theta\times\mathcal O_0^2$. Consequently, $|h_\theta(o,o')|\le B_g^2$ on this set; write $B_h:=B_g^2$.
\end{assum}

\begin{rem}[Sufficient conditions for neural kernels]
The uniform learning bound requires the outcome kernel to be bounded and Lipschitz in its parameters. For a Gaussian kernel on learned features, these properties hold if the feature map is uniformly Lipschitz in its parameters and the bandwidth is Lipschitz and bounded away from zero (a fixed positive bandwidth is a special case). For a feedforward network of fixed architecture, uniform parameter Lipschitzness follows, for example, from bounded inputs and parameters and Lipschitz activations; see \citet[Propositions 21 and 23]{liu2021learningdeepkernelsnonparametric}. These conditions control the neural component. Lemma~B.7 extends them to the kernel mixture and ensemble, while Propositions~B.8 and B.9 give the additional conditions needed for the oracle and fitted observed-data kernels.
\end{rem}

\begin{lem}[Regularity of the deep-kernel ensemble]\label{lem:additive-kernel-regularity}
Suppose there exist finite constants
\[
K_{\mathrm{deep}},\ K_{\mathrm{char}},\
L_{\mathrm{deep}},\ L_{\mathrm{char}},\ L_w<\infty
\]
such that, for all $\vartheta,\vartheta'\in\Xi$ and $y,y'\in\mathcal Y$,
\[
|k_\vartheta^{\mathrm{deep}}(y,y')|\le K_{\mathrm{deep}},
\qquad
|k_\vartheta^{\mathrm{char}}(y,y')|\le K_{\mathrm{char}},
\]
\[
|k_\vartheta^{\mathrm{deep}}(y,y')-k_{\vartheta'}^{\mathrm{deep}}(y,y')|
\le L_{\mathrm{deep}}\|\vartheta-\vartheta'\|_2,
\qquad
|k_\vartheta^{\mathrm{char}}(y,y')-k_{\vartheta'}^{\mathrm{char}}(y,y')|
\le L_{\mathrm{char}}\|\vartheta-\vartheta'\|_2,
\]
and
\[
|w_{\vartheta}-w_{\vartheta'}|
\le L_w\|\vartheta-\vartheta'\|_2.
\]
Writing $L_0:=L_{\mathrm{deep}}+L_{\mathrm{char}}+L_w(K_{\mathrm{deep}}+K_{\mathrm{char}})$, for $\theta=(\vartheta_1,\ldots,\vartheta_M)$ and $\theta'=(\vartheta'_1,\ldots,\vartheta'_M)$,
\[
|k_\theta(y,y')-k_{\theta'}(y,y')|
\le
\frac{L_0}{\sqrt M}\|\theta-\theta'\|_2.
\]
Moreover,
\[
\sup_{\theta,y}k_\theta(y,y)
\le
\max\{K_{\mathrm{deep}},K_{\mathrm{char}}\}.
\]
\end{lem}

\begin{proof}
For one member, add and subtract $(1-w_{\vartheta})k_{\vartheta'}^{\mathrm{deep}}+w_{\vartheta}k_{\vartheta'}^{\mathrm{char}}$. Since $0\le w_{\vartheta}\le1$,
\begin{align*}
&\big|
(1-w_{\vartheta})k_\vartheta^{\mathrm{deep}}
+w_{\vartheta}k_\vartheta^{\mathrm{char}}
-(1-w_{\vartheta'})k_{\vartheta'}^{\mathrm{deep}}
-w_{\vartheta'}k_{\vartheta'}^{\mathrm{char}}
\big|\\
&\le
(1-w_{\vartheta})|k_\vartheta^{\mathrm{deep}}-k_{\vartheta'}^{\mathrm{deep}}|
+
w_{\vartheta}|k_\vartheta^{\mathrm{char}}-k_{\vartheta'}^{\mathrm{char}}|
+
|w_{\vartheta}-w_{\vartheta'}|
|k_{\vartheta'}^{\mathrm{char}}-k_{\vartheta'}^{\mathrm{deep}}|\\
&\le
L_0\|\vartheta-\vartheta'\|_2.
\end{align*}
Average this bound over members and apply Cauchy-Schwarz:
\[
|k_\theta-k_{\theta'}|
\le
\frac{L_0}{M}\sum_{b=1}^M
\|\vartheta_b-\vartheta'_b\|_2
\le
\frac{L_0}{\sqrt M}\|\theta-\theta'\|_2.
\]
The diagonal bound follows first from convexity within each member and then from averaging the members.
\end{proof}

\begin{prop}[Primitive conditions for the induced kernel]\label{prop:induced-kernel-regularity}
Suppose $\Theta$ is compact and contained in a Euclidean ball of finite radius, \Cref{assum:identification} holds, the map $(\theta,y,y')\mapsto k_\theta(y,y')$ is jointly measurable, and, for finite $\kappa,L_k$,
\[
\sup_{\theta\in\Theta}\sup_{y\in\mathcal Y}k_\theta(y,y)\le\kappa,
\qquad
\sup_{y,y'\in\mathcal Y}|k_\theta(y,y')-k_{\theta'}(y,y')|
\le
L_k\|\theta-\theta'\|_2.
\]
Let
\[
W
:=
\operatorname*{ess\ sup}_{\mathbb P_X\otimes\rho}
\frac{|r(a,x)|}{\pi_0(a\mid x)}.
\]
If $W<\infty$, then there exists a measurable set $\mathcal O_0\subseteq\mathcal O$ with $\mathbb P_0(\mathcal O_0)=1$ such that
\[
\sup_{\theta\in\Theta}
\|g_\theta^{\pi,\pi'}(o)\|_{\mathcal H_\theta}
\le
B_g:=2(W+1)\sqrt\kappa
\qquad\text{for every }o\in\mathcal O_0,
\]
and the induced kernels in \eqref{eq:induced-causal-kernel} satisfy \Cref{assum:oracle-induced-class} with
\[
B_h=B_g^2=4(W+1)^2\kappa,
\qquad
L_h=4(W+1)^2L_k.
\]
\end{prop}

\begin{proof}
Choose a common full-measure set on which the contrast-overlap bound defining $W$ holds and whose $x$-sections have full $\rho$-measure, and let $\mathcal O_0$ be the corresponding full-measure subset of observed triples.
Positive definiteness and the diagonal bound imply
\[
|k_\theta(y,y')|
\le
\sqrt{k_\theta(y,y)k_\theta(y',y')}
\le\kappa.
\]
Hence $\|\varphi_\theta(y)\|_{\mathcal H_\theta}\le\sqrt\kappa$ and, by Jensen's inequality,
$\|\mu_\theta(a,x)\|_{\mathcal H_\theta}\le\sqrt\kappa$. Since
$\int|r(a,x)|\rho(\diff a)\le2$,
\begin{align*}
\|g_\theta^{\pi,\pi'}(o)\|_{\mathcal H_\theta}
&\le
\frac{|r(a,x)|}{\pi_0(a\mid x)}
\|\varphi_\theta(y)-\mu_\theta(a,x)\|_{\mathcal H_\theta}
+
\int|r(\widetilde a,x)|
\|\mu_\theta(\widetilde a,x)\|_{\mathcal H_\theta}
\rho(\diff\widetilde a)\\
&\le
2W\sqrt\kappa+2\sqrt\kappa
=
2(W+1)\sqrt\kappa.
\end{align*}
Cauchy-Schwarz gives $|h_\theta(o,o')|\le B_g^2=B_h$.

To compare $h_\theta$ across $\theta$, define for $o=(x,a,y)$ the finite signed measure
\begin{align*}
M_o(B)
:={}&
\frac{r(a,x)}{\pi_0(a\mid x)}
\{\mathbbm 1(y\in B)-Q(B\mid a,x)\}\\
&+
\int_{\mathcal A}r(\widetilde a,x)
Q(B\mid\widetilde a,x)
\rho(\diff\widetilde a).
\end{align*}
The common kernel $Q$ makes $M_o$ independent of $\theta$. Its total variation satisfies $\|M_o\|_{\mathrm{TV}}\le2(W+1)$. Moreover,
\[
g_\theta^{\pi,\pi'}(o)
=
\int_{\mathcal Y}\varphi_\theta(u)M_o(\diff u),
\]
so
\[
h_\theta(o,o')
=
\int_{\mathcal Y}\int_{\mathcal Y}
k_\theta(u,v)M_o(\diff u)M_{o'}(\diff v).
\]
Measurability of $Q$ and joint measurability of the kernel family imply that $(\theta,o,o')\mapsto h_\theta(o,o')$ is jointly measurable on $\Theta\times\mathcal O_0^2$.
Therefore
\begin{align*}
|h_\theta(o,o')-h_{\theta'}(o,o')|
&\le
\|M_o\|_{\mathrm{TV}}\|M_{o'}\|_{\mathrm{TV}}
\sup_{u,v}|k_\theta(u,v)-k_{\theta'}(u,v)|\\
&\le
4(W+1)^2L_k\|\theta-\theta'\|_2.
\end{align*}
\end{proof}

\begin{prop}[Primitive conditions for the fitted CME family]\label{prop:fitted-cme-regularity}
Condition on $\mathcal F_{\mathrm{fit}}$. Retain the compact parameter set and overlap and fitted propensity-ratio bounds of \Cref{assum:selection-regularity}, with a measurable $\hat\pi_0$ independent of $\theta$. Suppose $k_\theta$ is jointly measurable, has diagonal bound $\kappa$ and uniform parameter Lipschitz constant $L_k$ as in Proposition~\ref{prop:induced-kernel-regularity}, and each $\varphi_\theta$ is strongly measurable. Let
\[
\hat\mu_\theta(a,x)=\sum_{i=1}^n\beta_i(a,x)\varphi_\theta(Y_i),
\qquad
\sup_{a,x}\sum_{i=1}^n|\beta_i(a,x)|\le B_\beta,
\]
where the fitting outcomes $Y_i$ and measurable coefficient functions $\beta_i$ are fixed conditional on $\mathcal F_{\mathrm{fit}}$ and the coefficients do not depend on $\theta$. Suppose these bounds hold almost surely with fixed constants. Then \Cref{assum:selection-regularity} holds with
\[
B_\mu=B_\beta\sqrt\kappa,\qquad
L=C_\beta^2L_k,\qquad
C_\beta:=W(1+A)(1+B_\beta)+2B_\beta,
\]
and $|\widetilde h_\theta(o,o')|\le\kappa C_\beta^2$ on a common $\mathbb P_0^2$-full set.
\end{prop}
\begin{proof}
The CME bound follows by the triangle inequality. Define the finite signed measure
\[
\widehat M_o
=\frac{r(a,x)}{\hat\pi_0(a\mid x)}
\left(\delta_y-\sum_i\beta_i(a,x)\delta_{Y_i}\right)
+\sum_i\left\{\int r(t,x)\beta_i(t,x)\rho(\diff t)\right\}\delta_{Y_i},
\]
with the ratio set to zero where $r=0$. Since $|r|/\hat\pi_0\le W(1+A)$ and $\int|r(t,x)|\rho(\diff t)\le2$, its total variation is at most $C_\beta$ on a common full set. It is independent of $\theta$, and
\[
\widetilde g_\theta(o)=\int\varphi_\theta(u)\widehat M_o(\diff u),\qquad
\widetilde h_\theta(o,o')=\iint k_\theta(u,v)\widehat M_o(\diff u)\widehat M_{o'}(\diff v).
\]
These representations give the envelope and Lipschitz bounds by $|k_\theta|\le\kappa$ and the assumed bound on $k_\theta-k_{\theta'}$. The finite expansions and measurable coefficient integrals give strong measurability of the fitted functions and joint measurability of the scalar kernel.
\end{proof}

For kernel ridge regression with a fixed input kernel and ridge, the coefficients $\beta_i(a,x)$ do not depend on $\theta$, although the fitted embedding still changes through $\varphi_\theta$. With the $K+n\lambda_\mu I$ normalisation, a bounded input kernel and $\lambda_\mu$ bounded away from zero give the uniform coefficient bound required above.

\subsubsection{Oracle Concentration and Adaptation}

Let $(\lambda_m)_{m\ge1}$ be a deterministic sequence of positive selection regularisers. Let two independent oracle selection samples each contain $m$ observations. Construct the oracle versions of \eqref{eq:cross-mmd-estimator}, \eqref{eq:cross-variance-estimator}, and \eqref{eq:selection-criterion} by replacing the fitted pseudo-features with $g_\theta^{\pi,\pi'}$, and denote the resulting quantities by
\[
\widehat D_{\theta,\mathrm{or}}^{\mathrm{cross}},
\qquad
\widehat s_{\theta,\mathrm{or}}^2,
\qquad
\widehat J_{\lambda_m,\theta}^{\mathrm{or}}.
\]
For $\delta\in(0,1)$, define
\begin{align}
r_{D,m}(\delta)
&:=
C\left\{B_h c_m(\delta)+\frac{L_h}{\sqrt m}\right\},
\label{eq:rD-definition}\\
r_{s,m}(\delta)
&:=
C\left\{
B_h^2 c_m(\delta)+\frac{B_hL_h}{\sqrt m}
+
\frac{B_h^2}{m}
\right\},
\label{eq:rs-definition}
\end{align}
where $C>0$ is a universal constant. For $s>0$, let
\[
\Theta_s
:=
\{\theta\in\Theta:s_\theta^2\ge s^2\}.
\]
Membership in $\Theta_s$ depends on unknown population variances.

\begin{lem}[Uniform oracle concentration]\label{lem:oracle-concentration}
Under \Cref{assum:identification,assum:embedding-integrability} and \Cref{assum:oracle-induced-class}, with probability at least $1-\delta$,
\[
\sup_{\theta\in\Theta}
|\widehat D_{\theta,\mathrm{or}}^{\mathrm{cross}}-D_\theta|
\le
r_{D,m}(\delta),
\qquad
\sup_{\theta\in\Theta}
|\widehat s_{\theta,\mathrm{or}}^2-s_\theta^2|
\le
r_{s,m}(\delta).
\]
\end{lem}

\begin{proof}
We adapt the covering and bounded-difference arguments of \citet[Propositions~15--16 and Lemmas~17--19]{liu2021learningdeepkernelsnonparametric} to the cross-statistic. Earlier uniform bounds for kernel mixtures appear in \citet[Supplement, Appendix~A]{gretton2012optimal}.
Let $\tau_m=m^{-1/2}$. Since $\Theta$ lies in a Euclidean ball of radius $R_\Theta$, it admits a $\tau_m$-net $\mathcal N_{\tau_m}$ with
\[
|\mathcal N_{\tau_m}|
\le
\{1+2R_\Theta\sqrt m\}^{d_\Theta}.
\]

For fixed $\theta$,
\[
\widehat D_\theta
=
\frac1{m^2}\sum_{i=1}^m\sum_{j=1}^m
h_\theta(O_i^{(1)},O_j^{(2)}),
\qquad
\mathbb E_0 \left[ \widehat D_\theta \right] =D_\theta.
\]
Replacing one observation in either half changes $\widehat D_\theta$ by at most $2B_h/m$. McDiarmid's inequality therefore gives
\[
\mathbb P\{|\widehat D_\theta-D_\theta|>t\}
\le
2\exp\left(-\frac{mt^2}{4B_h^2}\right).
\]
A union bound over $\mathcal N_{\tau_m}$ controls the maximum over the net by a universal constant times
\[
B_h c_m(\delta).
\]
Since both $\widehat D_\theta$ and $D_\theta$ are $L_h$-Lipschitz, extending from the net adds at most $2L_h/\sqrt m$. This proves the numerator bound.

For the variance estimator, write
\[
\bar G_1:=\frac1m\sum_{i=1}^mg_\theta(O_i^{(1)}),
\qquad
\bar G_2:=\frac1m\sum_{j=1}^mg_\theta(O_j^{(2)}),
\]
and
\[
Z_j:=\langle\bar G_1,g_\theta(O_j^{(2)})-\bar G_2\rangle_{\mathcal H_\theta}.
\]
Then $\widehat s_\theta^2=m^{-1}\sum_jZ_j^2$ and $|Z_j|\le2B_h$. Replacing one observation in either half changes $\widehat s_\theta^2$ by at most $CB_h^2/m$ for a universal constant $C$. Hence
\[
\mathbb P\{|\widehat s_\theta^2-\mathbb E_0\widehat s_\theta^2|>t\}
\le
2\exp\left(-c\frac{mt^2}{B_h^4}\right)
\]
for a universal $c>0$.

We now control the bias term. Let $H=g_\theta(O)$, $\Delta=\mathbb E_0H$, $\xi=H-\Delta$, and let $\Sigma:=\mathbb E_0[\xi\otimes\xi]$. Conditional on $\bar G_1$, $\widehat s_\theta^2$ is the sample variance, with denominator $m$, of $\langle\bar G_1,H_j\rangle$. Therefore
\[
\mathbb E_0[\widehat s_\theta^2\mid\bar G_1]
=
\frac{m-1}{m}\langle\bar G_1,\Sigma\bar G_1\rangle.
\]
Since $\bar G_1=\Delta+m^{-1}\sum_i\xi_i$,
\[
\mathbb E_0\langle\bar G_1,\Sigma\bar G_1\rangle
=
s_\theta^2+
\frac1m\mathbb E_0\langle\xi,\Sigma\xi\rangle.
\]
Thus
\[
\mathbb E_0\widehat s_\theta^2-s_\theta^2
=
-\frac{s_\theta^2}{m}
+
\frac{m-1}{m^2}\mathbb E_0\langle\xi,\Sigma\xi\rangle,
\]
and boundedness of $h_\theta$ yields
\[
|\mathbb E_0\widehat s_\theta^2-s_\theta^2|
\le
C\frac{B_h^2}{m}.
\]
Finally, $\theta\mapsto\widehat s_\theta^2$ is $8B_hL_h$-Lipschitz and $\theta\mapsto s_\theta^2$ is $4B_hL_h$-Lipschitz. Applying the preceding concentration inequality on the net, adding the bias, and extending from the net proves the variance bound.
\end{proof}

\begin{lem}[Stable ratio bound]\label{lem:stable-ratio}
Let $D,\widehat D\in\mathbb R$, $v,\widehat v\in\mathbb R_+$, and $s,B>0$. Suppose
\[
0\le D\le B,
\qquad
v\ge s^2,
\qquad
|\widehat v-v|\le s^2/2.
\]
Then, for every $\lambda\ge0$,
\begin{equation}
\label{eq:deterministic-stable-ratio}
\left|
\frac{\widehat D}{\sqrt{\widehat v+\lambda}}
-
\frac{D}{\sqrt v}
\right|
\le
\frac{\sqrt2}{s}|\widehat D-D|
+
\frac{\sqrt2B}{s^3}\{|\widehat v-v|+\lambda\}.
\end{equation}
\end{lem}

\begin{proof}
The variance condition gives $\widehat v\ge s^2/2$. Hence
\begin{align*}
\left|
\frac{\widehat D}{\sqrt{\widehat v+\lambda}}
-
\frac{D}{\sqrt v}
\right|
&\le
\frac{|\widehat D-D|}{\sqrt{\widehat v+\lambda}}
+
D\left|(\widehat v+\lambda)^{-1/2}-v^{-1/2}\right|\\
&\le
\frac{\sqrt2}{s}|\widehat D-D|
+
\frac{\sqrt2D}{s^3}|\widehat v+\lambda-v|,
\end{align*}
where the second line uses the mean-value theorem on $[s^2/2,\infty)$. Since $D\le B$, \eqref{eq:deterministic-stable-ratio} follows.
\end{proof}

\begin{thm}[Variance-stable oracle adaptation]\label{thm:oracle-adaptation}
Let $s>0$ satisfy $\Theta_s\neq\varnothing$, and suppose \Cref{assum:identification,assum:embedding-integrability} and \Cref{assum:oracle-induced-class} hold. If $r_{s,m}(\delta)\le s^2/2$, then, with probability at least $1-\delta$,
\begin{equation}
\label{eq:oracle-uniform-efficacy-bound}
\sup_{\theta\in\Theta_s}
|\widehat J_{\lambda_m,\theta}^{\mathrm{or}}-J_\theta|
\le
\epsilon_m(\delta),
\end{equation}
where
\begin{equation}
\label{eq:oracle-epsilon-definition}
\epsilon_m(\delta)
:=
\frac{\sqrt2}{s}r_{D,m}(\delta)
+
\frac{\sqrt2B_h}{s^3}
\{r_{s,m}(\delta)+\lambda_m\}.
\end{equation}
If $B_h,L_h,s$ are fixed,
\[
\frac{d_\Theta\log(1+2R_\Theta\sqrt m)+1}{m}\longrightarrow0,
\qquad
\lambda_m
=
O\left(
\sqrt{\frac{d_\Theta\log(1+2R_\Theta\sqrt m)+1}{m}}
\right),
\]
then
\[
\sup_{\theta\in\Theta_s}
|\widehat J_{\lambda_m,\theta}^{\mathrm{or}}-J_\theta|
=
O_{\mathbb P}\left(
\sqrt{\frac{d_\Theta\log(1+2R_\Theta\sqrt m)+1}{m}}
\right).
\]
\end{thm}

\begin{proof}[Proof of \Cref{thm:oracle-adaptation}]
On the event in \Cref{lem:oracle-concentration}, every $\theta\in\Theta_s$ satisfies
\[
\widehat s_{\theta,\mathrm{or}}^2
\ge
s_\theta^2-r_{s,m}(\delta)
\ge
s^2/2.
\]
Moreover, $0\le D_\theta=\mathbb E_0h_\theta(O_1,O_2)\le B_h$. Apply \Cref{lem:stable-ratio} pointwise with
\[
D=D_\theta,
\quad
\widehat D=\widehat D_{\theta,\mathrm{or}}^{\mathrm{cross}},
\quad
v=s_\theta^2,
\quad
\widehat v=\widehat s_{\theta,\mathrm{or}}^2,
\quad
B=B_h,
\]
and take the supremum over $\Theta_s$. This gives \eqref{eq:oracle-uniform-efficacy-bound} with $\epsilon_m(\delta)$ as defined in \eqref{eq:oracle-epsilon-definition}. The stochastic rate follows from the definitions of $r_{D,m}$ and $r_{s,m}$.
\end{proof}

Theorem 1 of \citet{gretton2012optimal} and Theorem~11 of \citet{liu2021learningdeepkernelsnonparametric} choose a ridge of order $t^{-1/3}$ for training size $t$. Their ratio bounds amplify root-$t$ estimation errors, up to logarithmic and class-complexity factors, by $\lambda_t^{-1/2}$ and balance this against $O(\lambda_t)$ bias. Both assume a positive population variance lower bound. Here, for sufficiently large samples, the uniform concentration event bounds the empirical variance away from zero on $\Theta_s$, so \Cref{lem:stable-ratio} avoids this amplification. The same refinement applies on the variance-bounded subclass of Liu et al. using their Propositions 15-16 and $\lambda_t=O(\sqrt{\log t/t})$ for fixed class constants and a fixed positive variance lower bound. Thus the cube-root bounds do not establish an intrinsic rate limitation.

\subsubsection{Nuisance Perturbation}

We now transfer the oracle concentration bounds to fitted pseudo-features.

\begin{lem}[Score envelopes from selection regularity]\label{lem:fitted-score-envelope}
Under \Cref{assum:selection-regularity}, on a common full-measure set,
\[
\sup_{\theta\in\Theta}
\left\{
\|g_\theta^{\pi,\pi'}(o)\|_{\mathcal H_\theta}
\vee
\|\widetilde g_\theta(o)\|_{\mathcal H_\theta}
\right\}
\le M_g,
\]
where
\[
M_g
:=
\max\left\{
2(W+1)\sqrt\kappa,
W(1+A)(\sqrt\kappa+B_\mu)+2B_\mu
\right\}.
\]
Consequently $|\widetilde h_\theta(o,o')|\le M_g^2$.
\end{lem}

\begin{proof}
The diagonal kernel bound and Jensen's inequality give
$\|\varphi_\theta(y)\|\le\sqrt\kappa$ and
$\|\mu_\theta(a,x)\|\le\sqrt\kappa$. Also
$\int|r(a,x)|\rho(\diff a)\le2$. The oracle bound therefore follows from
\[
\|g_\theta(o)\|
\le
\frac{|r(a,x)|}{\pi_0(a\mid x)}
\|\varphi_\theta(y)-\mu_\theta(a,x)\|
+\int|r(\widetilde a,x)|\|\mu_\theta(\widetilde a,x)\|\rho(\diff\widetilde a).
\]
On the support of $|r|$,
\[
\frac{|r|}{\hat\pi_0}
=
\frac{|r|}{\pi_0}\frac{\pi_0}{\hat\pi_0}
\le W(1+A),
\]
which gives the fitted bound in the same way. Cauchy-Schwarz gives the final claim.
\end{proof}

Write $\widetilde B_h:=M_g^2$. Let $\widetilde r_{D,m}(\delta)$ and $\widetilde r_{s,m}(\delta)$ be defined by \eqref{eq:rD-definition}-\eqref{eq:rs-definition}, with $(B_h,L_h)$ replaced by $(\widetilde B_h,L)$.

\begin{lem}[Exact nuisance decomposition]\label{lem:sharp-nuisance-decomposition}
Under \Cref{assum:identification,assum:selection-regularity}, define
\[
\alpha(a,x)
:=
\frac{\pi_0(a\mid x)}{\hat\pi_0(a\mid x)}-1,
\qquad
d_\theta(a,x)
:=
\hat\mu_\theta(a,x)-\mu_\theta(a,x),
\qquad
q(a,x):=\frac{r(a,x)}{\pi_0(a\mid x)},
\]
where $q:=0$ off $\{|r|>0\}$.
Writing $\xi_\theta(O):=\varphi_\theta(Y)-\mu_\theta(A,X)$ and
$t_{\theta,d}(X):=\int r(a,X)d_\theta(a,X)\rho(\diff a)$, the fitted-score error has the exact decomposition
\begin{equation}
\label{eq:exact-fitted-score-decomposition}
\widetilde g_\theta-g_\theta
=
q\alpha\xi_\theta
+
\{t_{\theta,d}-qd_\theta\}
-
q\alpha d_\theta.
\end{equation}
The first term has conditional mean zero given $(A,X)$, and the second has conditional mean zero given $X$. Hence, for
$B_\theta:=\widetilde\Delta_\theta-\Delta_\theta$,
\begin{equation}
\label{eq:exact-fitted-mean-bias}
B_\theta
=
-\mathbb E_0\left[
\int_{\mathcal A}r(a,X)\alpha(a,X)d_\theta(a,X)\rho(\diff a)
\right],
\qquad
\sup_{\theta\in\Theta}\|B_\theta\|
\le
\epsilon_{\pi,n}\epsilon_{\mu,n}.
\end{equation}
Moreover, the squared $L^2(\mathbb P_0;\mathcal H_\theta)$ norms of the three terms in \eqref{eq:exact-fitted-score-decomposition} are bounded, respectively, by
\begin{equation}
\label{eq:score-component-bounds}
4\kappa W\epsilon_{\pi,n}^2,
\qquad
W\epsilon_{\mu,n}^2,
\qquad
W(B_\mu+\sqrt\kappa)^2\epsilon_{\pi,n}^2.
\end{equation}
\end{lem}

\begin{proof}
Because $r/\hat\pi_0=q(1+\alpha)$ and $\hat\mu_\theta=\mu_\theta+d_\theta$, subtracting the oracle score from the fitted score gives
\begin{align*}
\widetilde g_\theta-g_\theta
&=
q(1+\alpha)(\xi_\theta-d_\theta)
+\int r(\mu_\theta+d_\theta)\diff\rho
-q\xi_\theta-\int r\mu_\theta\diff\rho\\
&=
q\alpha\xi_\theta
+t_{\theta,d}-qd_\theta
-q\alpha d_\theta,
\end{align*}
where $\int f\diff\rho$ abbreviates integration over $\mathcal A$. Since $\mathbb E_0[\xi_\theta\mid A,X]=0$, the first conditional-mean statement follows. Also
\[
\mathbb E_0[qd_\theta(A,X)\mid X]
=
\int r(a,X)d_\theta(a,X)\rho(\diff a)
=t_{\theta,d}(X),
\]
which proves the second. Only the last term contributes to the unconditional mean, and conditioning on $X$ yields the identity for $B_\theta$. Cauchy-Schwarz with respect to the finite measure
$|r(a,x)|\mathbb P_X(\diff x)\rho(\diff a)$ proves the product bound.

For the component bounds, $\|\xi_\theta\|\le2\sqrt\kappa$ and
$r^2/\pi_0\le W|r|$, whence
\[
\mathbb E_0\|q\alpha\xi_\theta\|^2
\le4\kappa W\epsilon_{\pi,n}^2.
\]
The middle term is the negative of $qd_\theta$ centred conditionally on $X$, so conditional variance contraction gives
\[
\mathbb E_0\|t_{\theta,d}-qd_\theta\|^2
\le
\mathbb E_0\|qd_\theta\|^2
\le W\epsilon_{\mu,n}^2.
\]
Finally, $\|d_\theta\|\le B_\mu+\sqrt\kappa$ and the same change of measure gives the third bound in \eqref{eq:score-component-bounds}.
\end{proof}

\begin{lem}[Sharp population nuisance perturbation]\label{lem:nuisance-population-gap}
Under \Cref{assum:identification,assum:selection-regularity}, for every $\theta\in\Theta$,
\begin{equation}
\label{eq:sharp-numerator-perturbation}
\widetilde D_\theta-D_\theta
=
2\langle\Delta_\theta,B_\theta\rangle+\|B_\theta\|^2,
\qquad
|\widetilde D_\theta-D_\theta|
\le
4\sqrt\kappa\epsilon_{\pi,n}\epsilon_{\mu,n}
+\epsilon_{\pi,n}^2\epsilon_{\mu,n}^2.
\end{equation}
There is a finite $C$, depending only on $(\kappa,W,A,B_\mu)$, such that
\begin{equation}
\label{eq:sharp-variance-perturbation}
\sup_{\theta\in\Theta}|\widetilde s_\theta^2-s_\theta^2|
\le
C\{\epsilon_{\pi,n}+\epsilon_{\pi,n}^2
+\epsilon_{\pi,n}\epsilon_{\mu,n}+\epsilon_{\mu,n}^2\}
\le
C\{\epsilon_{\pi,n}+\epsilon_{\mu,n}^2\}.
\end{equation}
If $\hat\pi_0=\pi_0$ on the support of $|r|$, then $B_\theta=0$ and
\begin{equation}
\label{eq:known-logging-variance-inflation}
\widetilde D_\theta=D_\theta,
\qquad
\widetilde s_\theta^2
=
s_\theta^2
+\mathbb E_0\left[
\langle\Delta_\theta,\widetilde g_\theta(O)-g_\theta(O)\rangle^2
\right],
\end{equation}
so $0\le\widetilde s_\theta^2-s_\theta^2\le4\kappa W\epsilon_{\mu,n}^2$.
\end{lem}

\begin{proof}
The numerator identity follows from
$\widetilde\Delta_\theta=\Delta_\theta+B_\theta$. Since
$\|\Delta_\theta\|\le2\sqrt\kappa$, \eqref{eq:exact-fitted-mean-bias} gives the stated bound.

To bound the variance error, we treat the propensity and outcome components separately. Suppress $\theta$ and write
\[
G:=g-\Delta,
\qquad
e:=\widetilde g-g,
\qquad
\bar e:=e-B,
\qquad
Z:=\langle\Delta,G\rangle.
\]
The centred fitted projection is
\[
\widetilde Z
:=
\langle\Delta+B,G+\bar e\rangle
=
Z+R,
\qquad
R:=\langle\Delta,\bar e\rangle+\langle B,G+\bar e\rangle.
\]
Both $Z$ and $R$ are centred. Therefore
\begin{equation}
\label{eq:variance-perturbation-expansion}
\widetilde s^2-s^2
=
2\mathbb E_0[ZR]+\mathbb E_0[R^2].
\end{equation}

Let $p:=q\alpha\xi$, $u:=t_d-qd$, and $v:=-q\alpha d$ denote the three terms in \eqref{eq:exact-fitted-score-decomposition}. Also set
$t(X):=\int r(a,X)\mu(a,X)\rho(\diff a)$. Then
\[
G=q\xi+t(X)-\Delta.
\]
The conditional-embedding term has the exact cancellation
\begin{equation}
\label{eq:outcome-linear-cancellation}
\mathbb E_0\left[Z\langle\Delta,u\rangle\right]=0.
\end{equation}
Indeed, $u$ is $(A,X)$-measurable and $\mathbb E_0[\xi\mid A,X]=0$, so the covariance with the residual part $q\xi$ vanishes. The remaining part $t(X)-\Delta$ is $X$-measurable, while $\mathbb E_0[u\mid X]=0$. Thus there is no term linear in $\epsilon_{\mu,n}$.

The same calculation does not cancel the propensity contribution. Instead,
\begin{equation}
\label{eq:propensity-linear-term}
\mathbb E_0\left[Z\langle\Delta,p\rangle\right]
=
\mathbb E_0\left[
q(A,X)^2\alpha(A,X)
\langle\Delta,\xi(O)\rangle^2
\right],
\end{equation}
which is not necessarily zero. Since $r^2/\pi_0\le W|r|$, bounded features and Cauchy-Schwarz give
\[
\left|\mathbb E_0[Z\langle\Delta,p\rangle]\right|
\le
C\mathbb E_0\int|r(a,X)||\alpha(a,X)|\rho(\diff a)
\le C\epsilon_{\pi,n},
\]
where $\mathbb E_0\int|r|\diff\rho\le2$.

For the mixed term,
\[
\mathbb E_0[v\mid X]
=
-\int r(a,X)\alpha(a,X)d(a,X)\rho(\diff a).
\]
The residual part $q\xi$ is again orthogonal to $v$, while Cauchy-Schwarz gives
\[
\mathbb E_0\left\|\int r\alpha d\diff\rho\right\|
\le\epsilon_{\pi,n}\epsilon_{\mu,n}.
\]
Since $t(X)$, $\Delta$, and the scores are uniformly bounded by \Cref{lem:fitted-score-envelope},
\[
\left|\mathbb E_0[Z\langle\Delta,v-B\rangle]\right|
+
\left|\mathbb E_0[Z\langle B,G+\bar e\rangle]\right|
\le C\epsilon_{\pi,n}\epsilon_{\mu,n}.
\]

We now control the quadratic term in \eqref{eq:variance-perturbation-expansion}. From \eqref{eq:score-component-bounds},
\[
\mathbb E_0\|\bar e\|^2
\le
\mathbb E_0\|e\|^2
\le C(\epsilon_{\pi,n}^2+\epsilon_{\mu,n}^2).
\]
Using $\|B\|\le\epsilon_{\pi,n}\epsilon_{\mu,n}$ and the score envelope,
\[
\mathbb E_0[R^2]
\le
C\{\epsilon_{\pi,n}^2+\epsilon_{\mu,n}^2
+\epsilon_{\pi,n}^2\epsilon_{\mu,n}^2\}.
\]
Substituting into \eqref{eq:variance-perturbation-expansion} proves the first inequality in \eqref{eq:sharp-variance-perturbation}. Under \Cref{assum:selection-regularity}, both nuisance errors are bounded by constants: $|\alpha|\le A$, $\|d\|\le B_\mu+\sqrt\kappa$, and the $|r|$-weighted measure has mass at most two. Hence the higher-order and mixed terms are absorbed into $C(\epsilon_{\pi,n}+\epsilon_{\mu,n}^2)$, proving the second inequality.

If $\hat\pi_0=\pi_0$, then $p=v=B=0$ and $e=u$. Equation~\eqref{eq:outcome-linear-cancellation} yields
\[
\widetilde s^2
=
\mathbb E_0\{Z+\langle\Delta,u\rangle\}^2
=s^2+\mathbb E_0\langle\Delta,u\rangle^2.
\]
The middle component bound in \eqref{eq:score-component-bounds} and $\|\Delta\|^2\le4\kappa$ give the final inequality.
\end{proof}

\subsubsection{Uniform Criterion Bound and Consequences}

\begin{proof}[Proof of \Cref{thm:feasible-criterion-uniform}]
Conditional on $\mathcal F_{\mathrm{fit}}$, \Cref{lem:fitted-score-envelope} and the covering argument in \Cref{lem:oracle-concentration} give, with probability at least $1-\eta$,
\begin{equation}
\label{eq:fitted-score-concentration}
\sup_{\theta\in\Theta}
|\widehat D_\theta^{\mathrm{cross}}-\widetilde D_\theta|
\le\widetilde r_{D,m}(\eta),
\qquad
\sup_{\theta\in\Theta}
|\widehat s_\theta^2-\widetilde s_\theta^2|
\le\widetilde r_{s,m}(\eta).
\end{equation}
The definitions of these two radii and of $c_m$ imply
$\widetilde r_{D,m}(\eta)+\widetilde r_{s,m}(\eta)\le Cc_m(\eta)$ for fixed class constants. Combining \eqref{eq:fitted-score-concentration} with \eqref{eq:sharp-numerator-perturbation} and \eqref{eq:sharp-variance-perturbation} gives the componentwise bounds
\begin{align}
\sup_{\theta\in\Theta}|\widehat D_\theta^{\mathrm{cross}}-D_\theta|
&\le C\{c_m(\eta)+\epsilon_{\pi,n}\epsilon_{\mu,n}\},
\label{eq:uniform-numerator-rate}\\
\sup_{\theta\in\Theta}|\hat s_\theta^2-s_\theta^2|
&\le C\{c_m(\eta)+\epsilon_{\pi,n}+\epsilon_{\mu,n}^2\}.
\label{eq:uniform-variance-rate}
\end{align}
Here the bounded nuisance envelopes absorb $(\epsilon_{\pi,n}\epsilon_{\mu,n})^2$ into the product term.

For the ratio, decompose
\begin{align*}
|\widehat J_{\lambda,\theta}^{\mathrm{DR}}-J_{\lambda,\theta}|
&\le
\frac{|\widehat D_\theta^{\mathrm{cross}}-D_\theta|}{\sqrt\lambda}
+D_\theta\left|
(\widehat s_\theta^2+\lambda)^{-1/2}
-(s_\theta^2+\lambda)^{-1/2}
\right|\\
&\le
\frac{|\widehat D_\theta^{\mathrm{cross}}-D_\theta|}{\sqrt{\lambda_0}}
+\frac{2\kappa}{\lambda_0^{3/2}}
|\widehat s_\theta^2-s_\theta^2|,
\end{align*}
where $D_\theta\le4\kappa$ and the derivative of $v\mapsto(v+\lambda)^{-1/2}$ has magnitude at most $(2\lambda_0^{3/2})^{-1}$. Taking the supremum and using \eqref{eq:uniform-numerator-rate} to \eqref{eq:uniform-variance-rate} proves \eqref{eq:regularized-uniform-bound}.

For a cyclic rotation $c$, apply the conditional result given its own nuisance-fitting fold and write its errors as $\epsilon_{\pi,n}^{(c)}$ and $\epsilon_{\mu,n}^{(c)}$. Integrating each conditional failure probability and applying an unconditional union bound yields simultaneous control with probability at least $1-3\eta$. Averaging the three inequalities gives, for example,
\[
\sup_{\theta\in\Theta}
\left|
\frac13\sum_{c=1}^3\widehat J_{\lambda,\theta}^{\mathrm{DR},(c)}
-J_{\lambda,\theta}
\right|
\le
C\left[
c_m(\eta)
+\frac13\sum_{c=1}^3
\left\{
\epsilon_{\pi,n}^{(c)}\epsilon_{\mu,n}^{(c)}
+\epsilon_{\pi,n}^{(c)}
+(\epsilon_{\mu,n}^{(c)})^2
\right\}
\right].
\]
No independence among the three rotation events is required.
\end{proof}

\paragraph{Joint product-class complexity.}
If $\Xi\subseteq\{v:\|v\|_2\le R_0\}$, then $d_\Theta=Md_0$ and $R_\Theta\le\sqrt M R_0$. Hence
\[
c_m(\eta)
=O\left(
\sqrt{\frac{Md_0\log(1+2R_0\sqrt{Mm})+\log(1/\eta)+1}{m}}
\right).
\]
For fixed class constants and nuisance rates $\epsilon_{\pi,n}=O_{\mathbb P}(a_{\pi,n})$ and $\epsilon_{\mu,n}=O_{\mathbb P}(a_{\mu,n})$, \Cref{thm:feasible-criterion-uniform} gives criterion error
\[
O_{\mathbb P}\left(
\sqrt{\frac{Md_0\log(1+2R_0\sqrt{Mm})+1}{m}}
+a_{\pi,n}+a_{\mu,n}^2
\right).
\]
If $M$ increases, the sampling term vanishes when $Md_0\log(1+2R_0\sqrt{Mm})=o(m)$, provided the regularity constants remain bounded. The memberwise procedure in \Cref{alg:dr-dk-causal-mmd} uses the $M=1$ bound for each trained member.

\begin{thm}[Approximation of unregularised efficacy on a variance-stable class]\label{thm:feasible-adaptation}
Let $\Theta_s:=\{\theta\in\Theta:s_\theta^2\ge s^2\}$ be nonempty for $s>0$, and suppose \Cref{assum:identification,assum:selection-regularity} hold. There exist finite constants $C_v,C_s$, where $C_v$ depends only on the constants in \Cref{assum:selection-regularity} and $C_s$ may additionally depend on $s$, such that the following holds. On every $\mathcal F_{\mathrm{fit}}$ realisation satisfying
\[
C_v\{c_m(\eta)+\epsilon_{\pi,n}+\epsilon_{\mu,n}^2\}
\le s^2/2,
\]
for every $\lambda_m>0$, with conditional probability at least $1-\eta$,
\begin{equation}
\label{eq:feasible-adaptation-bound}
\sup_{\theta\in\Theta_s}
|\widehat J_{\lambda_m,\theta}^{\mathrm{DR}}-J_\theta|
\le
C_s\{c_m(\eta)+\epsilon_{\pi,n}\epsilon_{\mu,n}
+\epsilon_{\pi,n}+\epsilon_{\mu,n}^2+\lambda_m\}.
\end{equation}
Suppose also that deterministic sequences $a_{\pi,n},a_{\mu,n}\to0$ satisfy $\epsilon_{\pi,n}=O_{\mathbb P}(a_{\pi,n})$ and $\epsilon_{\mu,n}=O_{\mathbb P}(a_{\mu,n})$, and that
\[
\frac{d_\Theta\log(1+2R_\Theta\sqrt m)+1}{m}\longrightarrow0,
\qquad
\lambda_m
=O\left(
\sqrt{\frac{d_\Theta\log(1+2R_\Theta\sqrt m)+1}{m}}
\right).
\]
Then the margin holds with probability tending to one for fixed $s$ and fixed $\eta$, and the left-hand side of \eqref{eq:feasible-adaptation-bound} is
\[
O_{\mathbb P}\left(
\sqrt{\frac{d_\Theta\log(1+2R_\Theta\sqrt m)+1}{m}}
+a_{\pi,n}+a_{\mu,n}^2
\right).
\]
\end{thm}

\begin{proof}
The proof of \eqref{eq:uniform-variance-rate} supplies the constant $C_v$ independently of the ridge. On its concentration event, the assumed inequality gives
$\widehat s_\theta^2\ge s^2/2$ throughout $\Theta_s$. Apply \Cref{lem:stable-ratio} with $D=D_\theta$, $\widehat D=\widehat D_\theta^{\mathrm{cross}}$, $v=s_\theta^2$, $\widehat v=\widehat s_\theta^2$, and $B=4\kappa$. Equations~\eqref{eq:uniform-numerator-rate} to \eqref{eq:uniform-variance-rate} then give \eqref{eq:feasible-adaptation-bound}. The rate statement follows from \eqref{eq:kernel-class-sampling-term}.
\end{proof}

Under the preceding ridge choice, for $n\asymp m$ and fixed class constants, sufficient conditions for a uniform approximation rate of $O_{\mathbb P}(\sqrt{\log m/m})$ for the unregularised criterion $J_\theta$ on $\Theta_s$ are $\epsilon_{\pi,n}=O_{\mathbb P}(\sqrt{\log m/m})$ and $\epsilon_{\mu,n}=O_{\mathbb P}((\log m/m)^{1/4})$. These learner rates, including uniformity over $\theta$ for the regression, are assumed rather than proved here.

\begin{cor}[Known logging policy]\label{cor:known-logging}
Under the conditions of \Cref{thm:feasible-criterion-uniform}, if $\hat\pi_0=\pi_0$ on the support of $|r|$, then, for every $\lambda\ge\lambda_0$, with conditional probability at least $1-\eta$,
\[
\sup_{\theta\in\Theta}
|\widehat J_{\lambda,\theta}^{\mathrm{DR}}-J_{\lambda,\theta}|
\le
C\{c_m(\eta)+\epsilon_{\mu,n}^2\}.
\]
If, in addition, $\Theta_s$ is nonempty and
\[
C_v\{c_m(\eta)+\epsilon_{\mu,n}^2\}\le s^2/2,
\]
then, for every $\lambda_m>0$, with conditional probability at least $1-\eta$,
\[
\sup_{\theta\in\Theta_s}
|\widehat J_{\lambda_m,\theta}^{\mathrm{DR}}-J_\theta|
\le
C_s\{c_m(\eta)+\epsilon_{\mu,n}^2+\lambda_m\}.
\]
\end{cor}

\begin{proof}
With known logging, $\epsilon_{\pi,n}=0$, the population nuisance perturbation in the numerator vanishes, and \eqref{eq:known-logging-variance-inflation} bounds the population variance perturbation by $4\kappa W\epsilon_{\mu,n}^2$. Combining these facts with \eqref{eq:fitted-score-concentration} proves the fixed-ridge result. The stable-ratio argument in the proof of \Cref{thm:feasible-adaptation} proves the second result.
\end{proof}

\begin{rem}[Approximate optimisation and held-out inference]\label{rem:approximate-optimisation} The comparison below follows \citet[Corollary~13]{liu2021learningdeepkernelsnonparametric}, with an optimisation error as allowed in their Corollary~12, footnote~7.
Let $\lambda>0$. On an event where $\sup_{\theta\in\Theta}|\widehat J_{\lambda,\theta}^{\mathrm{DR}}-J_{\lambda,\theta}|\le\epsilon$, suppose a selected $\hat\theta\in\Theta$ satisfies
\[
\widehat J_{\lambda,\hat\theta}^{\mathrm{DR}}
\ge\sup_{\theta\in\Theta}\widehat J_{\lambda,\theta}^{\mathrm{DR}}-\delta_{\mathrm{opt}},
\qquad \delta_{\mathrm{opt}}\ge0.
\]
Then
\begin{align*}
J_{\lambda,\hat\theta}
&\ge\widehat J_{\lambda,\hat\theta}^{\mathrm{DR}}-\epsilon\\
&\ge\sup_{\theta\in\Theta}\widehat J_{\lambda,\theta}^{\mathrm{DR}}-\delta_{\mathrm{opt}}-\epsilon\\
&\ge\sup_{\theta\in\Theta}J_{\lambda,\theta}-2\epsilon-\delta_{\mathrm{opt}}.
\end{align*}
\Cref{thm:feasible-criterion-uniform} bounds $\epsilon$ for $\lambda\ge\lambda_0>0$, but we do not bound the optimisation error $\delta_{\mathrm{opt}}$ of the training procedure. Separate member training gives no guarantee of joint optimisation over $\Xi^M$.

For $\lambda_m\to0$, \Cref{thm:feasible-adaptation} gives the same comparison for $J_\theta$ on $\Theta_s$, provided $\hat\theta\in\Theta_s$ and the empirical optimisation inequality holds over $\Theta_s$. Held-out validity under \Cref{thm:heldout-validity} does not require global optimisation.
\end{rem}

\subsection{Null Second-Order Structure}
\label{sec:null-second-order-structure}

We use the conditional notation and nuisance errors from \Cref{sec:uniform-criterion-theory}.

\begin{prop}[Null second-order structure]\label{prop:null-second-order-structure}
Fix $\theta$ and assume all terms below are integrable. Consider a regular perturbation $t\mapsto\mathbb P_t$ through $\mathbb P_0$ for which the embedding contrast satisfies
$\Delta_{\theta,t}=\Delta_{\theta,0}+t\dot\Delta_\theta+o(t)$ in $\mathcal H_\theta$. Under $H_0$,
\begin{equation}
\label{eq:null-local-quadratic-expansion}
D_{\theta,t}
=
t^2\|\dot\Delta_\theta\|_{\mathcal H_\theta}^2+o(t^2).
\end{equation}
If the perturbation has score $\ell\in L_0^2(\mathbb P_0)$ and
$\dot\Delta_\theta=\mathbb E_0[g_\theta(O)\ell(O)]$, then the quadratic form in \eqref{eq:null-local-quadratic-expansion} is induced by
\[
h_{\theta,0}(o_1,o_2)
:=
\langle g_\theta(o_1),g_\theta(o_2)\rangle_{\mathcal H_\theta}.
\]
Equivalently, $D_{\theta,0}''[\ell,\ell]=2\mathbb E_0^{\otimes2}[h_{\theta,0}(O_1,O_2)\ell(O_1)\ell(O_2)]$; thus, under the conventional $t^2/2$ Taylor normalisation, the order-two coefficient is $2h_{\theta,0}$, whereas $h_{\theta,0}$ is the kernel averaged by the corresponding U-statistic. Moreover,
\[
\mathbb E_0[h_{\theta,0}(O_1,O_2)\mid O_1]
=
\mathbb E_0[h_{\theta,0}(O_1,O_2)\mid O_2]
=0
\quad\text{almost surely},
\]
so $h_{\theta,0}$ is one-degenerate.

For the fitted cross-product, suppose \Cref{assum:identification,assum:selection-regularity} hold, let
$\xi_{\theta,i}:=\widetilde g_\theta(O_i)-\widetilde\Delta_\theta$ and let $\bar\xi_{\theta,\ell}$ be its mean over selection half $\ell$. Under $H_0$, $\widetilde\Delta_\theta=B_\theta$ and
\begin{align}
\widehat D_\theta^{\mathrm{cross}}
&=
\|B_\theta\|^2
+\langle B_\theta,\bar\xi_{\theta,1}+\bar\xi_{\theta,2}\rangle
+\langle\bar\xi_{\theta,1},\bar\xi_{\theta,2}\rangle,
\label{eq:null-cross-decomposition}\\
\widehat s_\theta^2
&=
\frac1m\sum_{j\in\mathcal I_2}
\left\langle
B_\theta+\bar\xi_{\theta,1},
\xi_{\theta,j}-\bar\xi_{\theta,2}
\right\rangle^2.
\label{eq:null-variance-decomposition}
\end{align}
Consequently,
$\mathbb E_0[\widehat D_\theta^{\mathrm{cross}}\mid\mathcal F_{\mathrm{fit}}]=\|B_\theta\|^2\le\epsilon_{\pi,n}^2\epsilon_{\mu,n}^2$ and, for this fixed $\theta$,
\begin{equation}
\label{eq:null-second-order-rate}
\widehat D_\theta^{\mathrm{cross}}
=
O_{\mathbb P}\left(
\{\epsilon_{\pi,n}\epsilon_{\mu,n}+m^{-1/2}\}^2
\middle|\mathcal F_{\mathrm{fit}}
\right),
\qquad
\widehat s_\theta^2
=
O_{\mathbb P}\left(
\{\epsilon_{\pi,n}\epsilon_{\mu,n}+m^{-1/2}\}^2
\middle|\mathcal F_{\mathrm{fit}}
\right).
\end{equation}
\end{prop}

\begin{proof}
Under $H_0$, $\Delta_{\theta,0}=0$. Squaring the Hilbert-space expansion gives
\[
\|t\dot\Delta_\theta+o(t)\|^2
=t^2\|\dot\Delta_\theta\|^2+o(t^2),
\]
which proves \eqref{eq:null-local-quadratic-expansion}. If
$\dot\Delta_\theta=\mathbb E_0[g_\theta\ell]$, then
\[
\|\dot\Delta_\theta\|^2
=
\mathbb E_0^{\otimes2}\left[
h_{\theta,0}(O_1,O_2)\ell(O_1)\ell(O_2)
\right].
\]
Multiplying this identity by two gives the asserted second-derivative convention. Since $\mathbb E_0g_\theta=\Delta_{\theta,0}=0$,
\[
\mathbb E_0[h_{\theta,0}(O_1,O_2)\mid O_1]
=\langle g_\theta(O_1),\mathbb E_0g_\theta(O_2)\rangle=0
\quad\text{almost surely},
\]
and similarly conditioning on $O_2$. This is the asserted one-degeneracy.

For the empirical identities, write the two fitted half-means as
$B_\theta+\bar\xi_{\theta,1}$ and $B_\theta+\bar\xi_{\theta,2}$ and expand their inner product. Subtracting the second-half projected mean from each projected observation gives \eqref{eq:null-variance-decomposition}. Independence and centring of the two halves then give the conditional expectation. Finally, the Hilbert-space second-moment bound yields
$\bar\xi_{\theta,\ell}=O_{\mathbb P}(m^{-1/2}\mid\mathcal F_{\mathrm{fit}})$, while
$m^{-1}\sum_{j\in\mathcal I_2}\|\xi_{\theta,j}-\bar\xi_{\theta,2}\|^2=O_{\mathbb P}(1\mid\mathcal F_{\mathrm{fit}})$. Combining these facts with \eqref{eq:exact-fitted-mean-bias} proves \eqref{eq:null-second-order-rate}.
\end{proof}

This expansion parallels \citet[Theorem~2]{luedtke2019omnibus}, whose U-statistic uses all off-diagonal pairs. Here the cross-product uses independent halves, with conditional studentisation as in \citet{kim2024dimension,shekhar2022permutation}. With fitted nuisances, only the centred term in \eqref{eq:null-cross-decomposition} is degenerate; the $B_\theta$ terms account for nuisance bias.

\subsection{Held-Out Fixed-Alternative Scaling}
\label{sec:heldout-alternative}

\Cref{prop:fixed-kernel-efficacy} concerns a fixed kernel and true nuisance functions. We now extend its scaling limit to the selected kernel with fitted nuisances, conditional on the learning procedure. Write $m_{\mathrm{test}}$ for the size of each test half.

\begin{thm}[Held-out fixed-alternative scaling]\label{thm:heldout-alternative}
Suppose \Cref{assum:identification,assum:embedding-integrability} hold. Consider a joint sequence with $m_{\mathrm{test}}\to\infty$ under a fixed alternative. Assume that the final fitted score is $\mathcal F_{\mathrm{learn}}$-measurable and conditionally Bochner integrable, and that the two test halves consist of i.i.d. observations from $\mathbb P_0$, are mutually independent, and are independent of $\mathcal F_{\mathrm{learn}}$. For a random fitted score $f:\mathcal O\to\mathcal H_{\hat\theta}$, write
\[
\|f\|_{2,\mathrm{learn}}^2
:=
\mathbb E_0[\|f(O)\|_{\mathcal H_{\hat\theta}}^2
\mid\mathcal F_{\mathrm{learn}}].
\]
For some fixed $s>0$, suppose $\mathbb P(\hat\theta\in\Theta_s)\to1$, define $J_{\hat\theta}:=0$ on $\{\hat\theta\notin\Theta_s\}$, and assume
\[
\|\widehat g_{\hat\theta}-g_{\hat\theta}^{\pi,\pi'}\|_{2,\mathrm{learn}}
=
o_{\mathbb P}(1),
\qquad
\mathbb E_0[\|\widehat g_{\hat\theta}(O)\|_{\mathcal H_{\hat\theta}}^4
\mid\mathcal F_{\mathrm{learn}}]
=O_{\mathbb P}(1).
\]
Then
\begin{equation}
\label{eq:heldout-alternative-limit}
\frac{T_{\hat\theta}}{\sqrt{m_{\mathrm{test}}}}
-
J_{\hat\theta}
\xrightarrow{\mathbb P}0.
\end{equation}
\end{thm}

\begin{proof}[Proof of \Cref{thm:heldout-alternative}]
Condition on $\mathcal F_{\mathrm{learn}}$ and write
\[
G:=\widehat g_{\hat\theta},
\qquad
g:=g_{\hat\theta}^{\pi,\pi'},
\qquad
\mu_G:=\mathbb E_0[G(O)\mid\mathcal F_{\mathrm{learn}}],
\qquad
\Delta:=\Delta_{\hat\theta}.
\]
The conditional fourth-moment bound and conditional laws of large numbers
on the two independent test halves give
\begin{align*}
\widehat D_{\hat\theta}^{\mathrm{cross}}
-\|\mu_G\|_{\mathcal H_{\hat\theta}}^2
&=o_{\mathbb P}(1),\\
\hat s_{\hat\theta}^2
-\operatorname{Var}_{\mathbb P_0}
\left(
\langle\mu_G,G(O)\rangle_{\mathcal H_{\hat\theta}}
\middle|\mathcal F_{\mathrm{learn}}
\right)
&=o_{\mathbb P}(1).
\end{align*}
For the second line, first condition additionally on the first test half;
the second-half studentiser is a scalar sample variance. The conditional
fourth-moment bound makes its law of large numbers uniform along the joint
sequence, and replacing the first-half mean by $\mu_G$ costs
$o_{\mathbb P}(1)$.

Since $\mathbb E_0[g(O)\mid\mathcal F_{\mathrm{learn}}]=\Delta$, conditional Jensen's inequality and the assumed score convergence give
\[
\|\mu_G-\Delta\|_{\mathcal H_{\hat\theta}}
\le
\|G-g\|_{2,\mathrm{learn}}
=o_{\mathbb P}(1).
\]
Moreover,
\begin{align*}
&\left\|
\langle\mu_G,G(O)\rangle_{\mathcal H_{\hat\theta}}
-\langle\Delta,g(O)\rangle_{\mathcal H_{\hat\theta}}
\right\|_{L^2(\mathbb P_0\mid\mathcal F_{\mathrm{learn}})}\\
&\qquad\le
\|\mu_G-\Delta\|_{\mathcal H_{\hat\theta}}
\|G\|_{2,\mathrm{learn}}
+
\|\Delta\|_{\mathcal H_{\hat\theta}}
\|G-g\|_{2,\mathrm{learn}}
=o_{\mathbb P}(1).
\end{align*}
Here $\|G\|_{2,\mathrm{learn}}=O_{\mathbb P}(1)$ follows from the
conditional fourth-moment bound, and
$\|\Delta\|=O_{\mathbb P}(1)$ follows from the preceding display
and $\|\mu_G\|\le\|G\|_{2,\mathrm{learn}}$. Thus the fitted-score numerator and variance differ from
$D_{\hat\theta}$ and $s_{\hat\theta}^2$, respectively, by
$o_{\mathbb P}(1)$. On $\{\hat\theta\in\Theta_s\}$, division is stable
because $s_{\hat\theta}\ge s$. In particular,
$\mathbb P(\hat s_{\hat\theta}=0)\to0$, and the zero-denominator convention
in \eqref{eq:test-statistic} does not affect the limit. It follows that
\[
\frac{T_{\hat\theta}}{\sqrt{m_{\mathrm{test}}}}
=J_{\hat\theta}+o_{\mathbb P}(1),
\]
which proves \eqref{eq:heldout-alternative-limit}.
\end{proof}

\FloatBarrier
\section{Experimental Details}\label{appendix:experimental-details}

We now detail our experimental setting. We first describe the training and testing procedure, then specify the construction of each experiment. Finally, we detail the neural network architectures and hyperparameters used for the optimisation.

\subsection{Training and Testing Procedure}
\label{sec:training-testing-details}

We describe the training and testing procedure in \Cref{alg:dr-dk-causal-mmd}. Each ensemble member is trained separately by maximising its own learning criterion, using the same learning sample $\cD_{\mathrm{learn}}$ but a different random initialisation. We estimate the outcome embedding using kernel ridge regression and the propensity estimate using logistic regression for binary treatments and a Gaussian logging-policy model for continuous treatments. For each member, we average the learning criteria across three rotations of the learning folds and optimise $\vartheta_b$ by gradient ascent. At each epoch, we update an exponential moving average (EMA) of the kernel parameters and compute its unregularised test statistic on $\cD_{\mathrm{val}}$, using nuisance estimators fitted on $\cD_{\mathrm{learn}}$. We retain the EMA checkpoint with the highest validation statistic for each member, then average the selected kernels as in \eqref{eq:deep-kernel}. Finally, we refit the nuisance estimators on $\cD_{\mathrm{learn}}\cup\cD_{\mathrm{val}}$ for testing.

We keep the selected kernels and fitted nuisance estimators fixed and evaluate all methods on the same 200 fresh test datasets $\cD_{\mathrm{test}}$ at $\alpha=0.05$. We perform 20 independent outer repetitions, drawing new learning and validation samples and repeating kernel training, checkpoint selection and nuisance estimation. We report the mean rejection rate $\pm$ one standard error across outer repetitions. The reported sample size is $N := |\cD_{\mathrm{learn}}| + |\cD_{\mathrm{val}}| + |\cD_{\mathrm{test}}|$ where each sample contains $N/3$ observations, up to integer rounding. Here, $|\cD_{\mathrm{test}}|$ denotes the size of each test dataset.

For experiments with binary treatments, we compare two deterministic interventions $\pi(1\mid x)=1$ and $\pi'(0\mid x)=1$.

For DR-KTE, we use the Gaussian outcome kernel $k(y,y')=\exp\{-\|\phi(y)-\phi(y')\|^2/\sigma^2\}$ where $\phi(y)$ is a numerical representation of the structured outcome. We use flattened pixel values for images, token IDs for text, node-label and parent-index encodings for graphs. We set $\sigma$ to the median pairwise Euclidean distance over the combined learning and validation outcomes, except on Causal-Blob, where we use the fifth percentile.

\begin{algorithm}[h]
\caption{Training and testing procedure of DR-KTE-Deep}
\label{alg:dr-dk-causal-mmd}
\begin{algorithmic}[1]
\Require Logged data $\cD$, target interventions $\pi,\pi'$, ensemble size $M$, member kernel family $\{k_\vartheta:\vartheta\in\Xi\}$, selection regularisation $\lambda>0$, epochs $T_{\max}$, EMA decay $\gamma$
\State Split $\cD$ into disjoint samples $\cD_{\mathrm{learn}}$, $\cD_{\mathrm{val}}$, and $\cD_{\mathrm{test}}$
\State Split $\cD_{\mathrm{learn}}$ into three equal folds $\cD^{(1)},\cD^{(2)},\cD^{(3)}$
\For{$b=1,\dots,M$}
    \State Initialise $\vartheta_b\in\Xi$ independently across members and set $\widetilde\vartheta_b\gets\vartheta_b$
    \For{$t=1,\dots,T_{\max}$}
        \For{$r=1,2,3$}
            \State Assign the folds cyclically to $\cD_{\mathrm{fit}}^{(r)},\cD_{\mathrm{sel},1}^{(r)},\cD_{\mathrm{sel},2}^{(r)}$
            \State Fit $\hat\pi_0^{(r)}$ and $\hat\mu_{\vartheta_b}^{(r)}$ on $\cD_{\mathrm{fit}}^{(r)}$
            \State Compute pseudo-features on $\cD_{\mathrm{sel},1}^{(r)}$ and $\cD_{\mathrm{sel},2}^{(r)}$
            \State Compute $\widehat J_{\lambda,\vartheta_b}^{\mathrm{DR},(r)}$ using \eqref{eq:selection-criterion}
        \EndFor
        \State Update $\vartheta_b$ by gradient ascent on $\frac{1}{3}\sum_{r=1}^3\widehat J_{\lambda,\vartheta_b}^{\mathrm{DR},(r)}$
        \State Update the EMA parameters $\widetilde\vartheta_b$
        \State Fit nuisances on $\cD_{\mathrm{learn}}$ and evaluate $T_{\widetilde\vartheta_b}$ on $\cD_{\mathrm{val}}$
    \EndFor
    \State Set $\hat\vartheta_b$ to the EMA checkpoint with the largest validation statistic on $\cD_{\mathrm{val}}$
\EndFor
\State Set $\hat\theta=(\hat\vartheta_1,\ldots,\hat\vartheta_M)$ and form $k_{\hat\theta}$ using \eqref{eq:deep-kernel}
\State Refit the nuisance estimators on the combined learning and validation samples $\cD_{\mathrm{learn}}\cup\cD_{\mathrm{val}}$
\State Split $\cD_{\mathrm{test}}$ into two halves $\mathcal I_1, \mathcal I_2$ of size $m_{\mathrm{test}}$
\State Compute the test pseudo-features and form $T_{\hat\theta}$ using \eqref{eq:test-statistic}
\State \Return $T_{\hat\theta}$ and the asymptotic $p$-value $1-\Phi(T_{\hat\theta})$
\end{algorithmic}
\end{algorithm}

\subsection{Image Outcomes}\label{appendix:medical-image-details}

We detail the construction of experiments on image outcomes from \Cref{sec:image-outcomes}. We extend the OCTMNIST experiments of \citet{zenati2026semiparametric}. We apply semi-synthetic perturbations to baseline images sampled from MedMNIST v2 \citep{medmnistv2}, released under CC BY-NC 4.0 for DermaMNIST and CC BY 4.0 for the remaining datasets.

For each unit, we sample a synthetic covariate
\begin{equation}
    X_i \in \mathbb{R}^{d_X}, \quad X_i \sim \mathcal{N}(0, I_{d_X}),
    \label{eq:medical-image-covariate-generation}
\end{equation}
and assign treatment as
\begin{equation}
    A_i \mid X_i \sim \operatorname{Bernoulli}\{\pi_0(1 \mid X_i)\}, \quad \pi_0(1 \mid X_i) = \operatorname{clip}\{ \sigma(\boldsymbol{\alpha}^\top X_i), 0.07, 0.93 \}.
    \label{eq:medical-image-treatment-assignment}
\end{equation}
The potential outcomes are generated as
\[
    Y^a_i = B_i + R_i(a), \quad a \in \{0, 1\},
\]
where $B_i$ is a baseline image sampled from a healthy or normal-labelled class and $R_i(a)$ is a treatment dependent residual. We generate a potential outcome $Y^a_i$ for each treatment $a$ and retain the outcome under the assigned treatment, $Y_i = Y^{A_i}_i$. 

To construct the residual, let
\[
    q_i = q_{\min} + (q_{\max} - q_{\min}) \sigma(\boldsymbol\beta^\top X_i)
\]
where $q_i \in (q_{\min}, q_{\max})$ determines the probability of a severe case. We sample
\[
    S_i \mid X_i \sim \operatorname{Bernoulli}(q_i),
\]
define mean-one multiplicative jitter factors
\[
    J_{ia} = \exp(\tau U_{ia} - \tau^2 / 2), \quad U_{ia} \sim \mathcal{N}(0, 1),
\]
and independent pixel noise $\varepsilon_{ia} \sim \mathcal{N}(0, \sigma_\varepsilon^2 I)$. The residuals under both treatments are then
\[
    R_i(0) = q_i \delta J_{i0} M_i + \varepsilon_{i0}, \quad R_i(1) = S_i \delta J_{i1} M_i + \varepsilon_{i1},
\]
where $M_i$ is a mask describing one of three perturbations: a localised lesion, scattered deposits, and a morphological change. Localised perturbations add a single blob to the baseline image, as in Chest-Effusion, Chest-Mass, OCT-DME, OCT-CNV and Pneumonia. Scattered perturbations add several small blobs to plausible locations, as in OCT-Drusen and Retina-Severe. Morphological perturbations change an existing structure; Chest-Cardiomegaly enlarges the cardiac silhouette, while Derma-Melanoma distorts and desaturates a mole. Finally, Breast-Tumor distorts an added localised lesion, combining a localised and morphological perturbation.

We use $d_X=6$, $\boldsymbol{\alpha}=(0.90,-0.75,0.55,-0.40,0.30,-0.20)^\top$ and $\boldsymbol{\beta}=(0.60,-0.50,0.40,-0.30,0.20,-0.10)^\top$ with $q_{\min}=0.10$, $q_{\max}=0.40$, $\tau=0.20$ and $\sigma_\varepsilon=0.025$. Baseline images are sampled with replacement from the corresponding training, validation or test split.

For \Cref{tab:medmnist-power} under $H_1$, we use $\delta=0.7$ for Breast-Tumor, $1.4$ for Chest-Cardiomegaly, $1.2$ for Chest-Effusion, $1.0$ for Chest-Mass, $1.4$ for Derma-Melanoma, $0.8$ for OCT-CNV, $0.4$ for OCT-DME and OCT-Drusen, and $1.0$ for Pneumonia and Retina-Severe. We set $\delta=0$ under $H_0$.

\subsection{Text Outcomes}\label{appendix:text-constructions}

We now detail the generation of text outcomes in \Cref{sec:text-outcome-experiments}. Given some covariate $X_i$, we sample a treatment $A_i$ according to \Cref{eq:medical-image-treatment-assignment} and generate pseudo-outcomes
\begin{equation}
    \label{eq:text-pseudo-outcomes}
    \tilde Y_i^0 = \boldsymbol{\beta}^\top X_i + \varepsilon_{i0}, \quad \tilde Y_i^1 = \boldsymbol{\beta}^\top X_i + b_i(\delta) + \varepsilon_{i1}, \quad \varepsilon_{i0},\varepsilon_{i1}
    \overset{\mathrm{i.i.d.}}{\sim}\mathcal{N}(0,\sigma_\varepsilon^2).
\end{equation}
where $b_i(\delta) = \delta(2Z_i - 1)$, with $Z_i \sim_{\mathrm{i.i.d.}} \operatorname{Bernoulli}(0.5)$. Each pseudo-outcome is then discretised into one of three tones. The tone is used as conditioning to generate the outcome text $Y^a$ for each treatment $a$ from a large language model. We retain the potential outcome under the assigned treatment, $Y_i = Y_i^{A_i}$.

\paragraph{Oncology Notes.}
We sample a five-dimensional covariate $X_i \sim \mathcal{N}(0,I_5)$. The covariate is rendered as a breast-cancer patient profile containing age, cancer stage, ECOG performance status, a biomarker score, and a comorbidity index. For each profile, we prompt the language model to write three post-treatment clinical notes describing respectively disease progression, stabilised disease and positive treatment response.

\paragraph{Movie Reviews.}
For each unit, we first sample a movie review from the Rotten Tomatoes dataset \citep{pang2005seeing} and prompt the language model to rewrite the review as positive, negative and neutral. We construct the covariate $X_i\in\R^5$ by applying a fixed Gaussian random projection to the sampled review's embedding.

\paragraph{Clinical Notes.}
For each unit, we first sample a clinical note from the MTSamples dataset \citep{mtsamples}. We then append to the note two language model generated sentences describing a positive, negative or neutral clinical outcome. As with movie reviews, we construct $X_i\in\R^5$ from a fixed Gaussian random projection of the note's embedding.

We use $\boldsymbol{\alpha}=(0.90,-0.75,0.55,-0.40,0.30)^\top$, $\boldsymbol{\beta}=(0.60,-0.50,0.40,-0.30,0.20)^\top$,
and $\sigma_\varepsilon=0.1$. Pseudo-outcomes below $-1$ select the negative variant, in $[-1,1)$ select the neutral variant, and at or above $1$ select the positive text variant. For $H_1$, we use $\delta=1.0$ for Oncology Notes and Movie Reviews and $\delta=2.6$ for Clinical Notes. For $H_0$, we use $\delta=0$.

We generate the text variants with Qwen3-8B, with thinking disabled, temperature $0.7$ and $p=0.9$. Outcomes in \Cref{appendix:text-representation-ablations} and text covariates are embedded using Qwen3-Embedding-0.6B.

\subsection{Graph Outcomes}\label{appendix:graph-constructions}

Finally, we detail the graph outcome constructions of \Cref{sec:graph-outcome-experiments}. For a pre-treatment graph $G_i$, we extract and standardise a scalar feature $c_i$ from $G_i$ to obtain the covariate $X_i$. As before, treatment is assigned according to \Cref{eq:medical-image-treatment-assignment} with $\alpha = 1$. Here, treatment modifies $G_i$, from which we construct the outcome graph $Y_i$.

\paragraph{Contact Networks.}

Each unit's pre-treatment graph $G_i$ is a contact network obtained by sampling a connected subgraph from the Infectious SocioPatterns dataset \citep{isella2011crowd}. Each node is an individual, and each edge is weighted by the number of 20-second intervals during which two individuals were in close proximity. We use mean excess degree as a proxy for epidemic risk, motivated by its role in random-network epidemic thresholds \citep{newman2002spread}:
\[
    c_i
    =
    \frac{\sum_v d_v^2}{\sum_v d_v}-1,
\]
where $d_v$ is the number of distinct contacts of individual $v$. Contact networks with a larger epidemic risk are therefore more likely to receive the intervention.

We sample an infection source with probability proportional to each node's weighted degree. Under $A_i=0$, the contact network is unchanged. Under $A_i=1$, we remove 20\% of the contact edges, prioritising edges with more neighbours, greater contact weight, and greater weighted degree. We then sample potential transmission edges, with probability increasing with contact weight, and construct a transmission tree from each network. The two potential outcomes use the same infection source and sampled potential transmission edges. Under $H_0$, the transmission trees are generated from the unchanged contact network for both potential outcomes. 

\paragraph{LLM Code Repairs.}

Each unit's pre-treatment graph $G_i$ is a buggy Python function from the BugsInPy dataset \citep{widyasari2020bugsinpy}, represented as its abstract syntax tree (AST). 

Under $A_i=0$, we retain the original AST, $Y_i^0=G_i$. Under $A_i=1$, the outcome $Y_i^1$ is the AST of the corresponding LLM-generated repair attempt from \citet{ehsani2025hierarchical}. We choose $c_i$ as the log-size
\[
    c_i=\log(1+|\mathcal{V}_i|),
\]
where $|\mathcal{V}_i|$ is the number of nodes in the original AST. Longer functions are therefore more likely to receive an LLM-generated repair.

Under $H_0$, we set both potential outcomes to the original AST, $Y_i^0=Y_i^1=G_i$, while retaining the same covariates and treatment-assignment mechanism. Under $H_1$, the treated potential outcome is the corresponding repair candidate.

Each outcome tree is encoded as a padded sequence of node labels and parent indices. This representation is injective for the labelled tree, although constructing the tree may discard information from the original program or contact network. The WL subtree kernel is not generally characteristic, and therefore does not provide the same omnibus guarantees as the Gaussian kernel tests.

\FloatBarrier
\subsection{Architectures and Hyperparameters}\label{sec:architectures-hyperparameters}

We describe the neural network architectures and hyperparameters used for DR-KTE-Deep. For scalar outcomes, Causal-Blob, and text outcomes when represented as an embedding, $\zeta_\theta$ is a four-layer fully connected network with softplus activations (\Cref{fig:architecture-mlp}). Text outcomes, when represented as a sequence of tokens, use a two-layer transformer (\Cref{fig:architecture-transformer}). For image outcomes, we use a four-layer convolutional network with LeakyReLU activations (\Cref{fig:architecture-cnn}). Finally for graph outcomes, we use a three-layer message-passing graph neural network (GNN) (\Cref{fig:architecture-gnn}).

We use an Adam optimiser \citep{kingma2015adam}, with separate learning rates for the neural features and the remaining parameters. Hyperparameters are summarised in \Cref{tab:shared-hyperparameters,tab:experiment-hyperparameters}.

DR-KTE-Opt shares the same optimiser, regularisation parameters and number of epochs, but only optimises the lengthscale of its kernel, using the same learning rate as the non-neural parameters of DR-KTE-Deep.

\begin{table}[!htbp]
    \centering
    \small
    \setlength{\tabcolsep}{4pt}
    \begin{tabular}{@{}ll@{}}
        \toprule
        Hyperparameter & Value \\
        \midrule
        Training variance regularisation $\lambda$ & $10^{-4}$ \\
        Validation and test variance regularisation & $0$ \\
        CME regularisation $\lambda_\mu$ & $10^{-3}$ \\
        EMA decay $\gamma$ & $0.95$ \\
        Minimum mixture weight $w_{\min}$ & $10^{-9}$ \\
        Learning-rate schedule & Cosine decay to $1\%$ of initial rate \\
        Gradient norm clipping & $1.0$ \\
        \bottomrule
    \end{tabular}
    \caption{Shared hyperparameters for our experiments.}
    \label{tab:shared-hyperparameters}
\end{table}

\begin{table}[!htbp]
    \centering
    \small
    \setlength{\tabcolsep}{4pt}
    \begin{tabular}{@{}lrrrr@{}}
        \toprule
        Experiment & Epochs & Ensemble Size $M$ & Feature LR & Other LR \\
        \midrule
        Image outcomes & 500 & 3 & $2\times10^{-2}$ & $10^{-3}$ \\
        Text outcomes (token sequences) & 1000 & 3 & $10^{-3}$ & $10^{-4}$ \\
        Graph outcomes & 1000 & 3 & $2\times10^{-2}$ & $10^{-3}$ \\
        Text outcomes (embedded) & 1000 & 3 & $10^{-5}$ & $10^{-4}$ \\
        Scalar outcomes & 1000 & 5 & $2\times10^{-2}$ & $10^{-3}$ \\
        Causal-Blob & 1000 & 5 & $2\times10^{-2}$ & $10^{-2}$ \\
        \bottomrule
    \end{tabular}
    \caption{Experiment-specific hyperparameters.}
    \label{tab:experiment-hyperparameters}
\end{table}

\begin{figure}[!htbp]
    \centering
    \includegraphics[width=\linewidth]{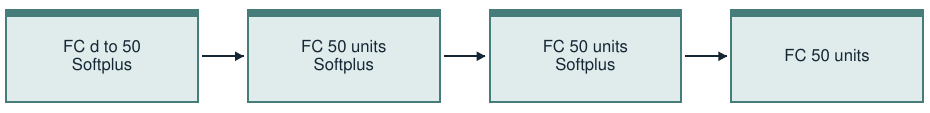}
    \caption{Four-layer fully connected feature extractor for scalar outcomes, Causal-Blob, and text outcomes under an embedded representation.}
    \label{fig:architecture-mlp}
\end{figure}

\begin{figure}[!htbp]
    \centering
    \includegraphics[width=\linewidth]{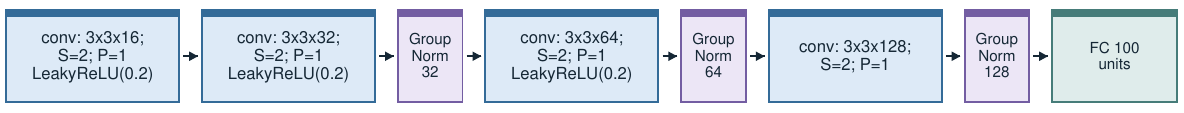}
    \caption{Four-layer convolutional feature extractor for image outcomes.}
    \label{fig:architecture-cnn}
\end{figure}

\begin{figure}[!htbp]
    \centering
    \includegraphics[width=\linewidth]{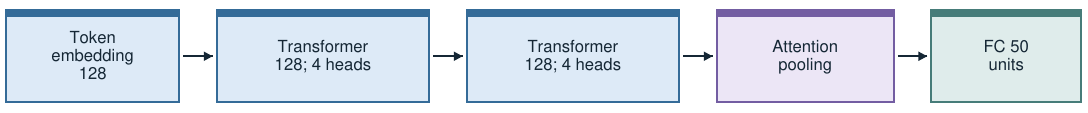}
    \caption{Two-layer transformer feature extractor for text outcomes represented as a sequence of tokens.}
    \label{fig:architecture-transformer}
\end{figure}

\begin{figure}[!htbp]
    \centering
    \includegraphics[width=\linewidth]{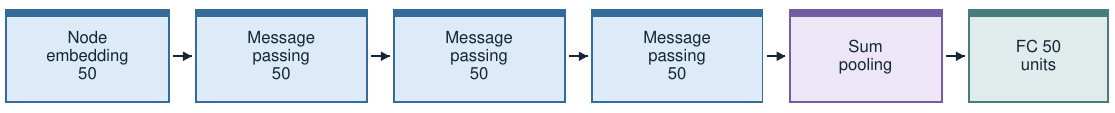}
    \caption{Three-layer message-passing feature extractor for graph outcomes.}
    \label{fig:architecture-gnn}
\end{figure}

\FloatBarrier
\section{Additional Experiments}\label{appendix:additional-experiments}

\subsection{Additional Results on Image Outcomes}\label{sec:additional-image-experiments}

\paragraph{Medical-Image Outcomes.}\label{sec:additional-medical-image-results} We run additional simulations on medical-image outcomes, varying both sample size and effect strength. On OCT-DME, deep-kernel power is already near one throughout the sample-size sweep, while the Gaussian baselines have low power (\Cref{fig:oct-dme-sweep}). On Derma-Melanoma, increasing $N$ substantially raises deep-kernel power, with smaller gains for the Gaussian baselines (\Cref{fig:derma-melanoma-sweep}).

\begin{figure}[!htbp]
    \centering
    \includegraphics[width=\linewidth]{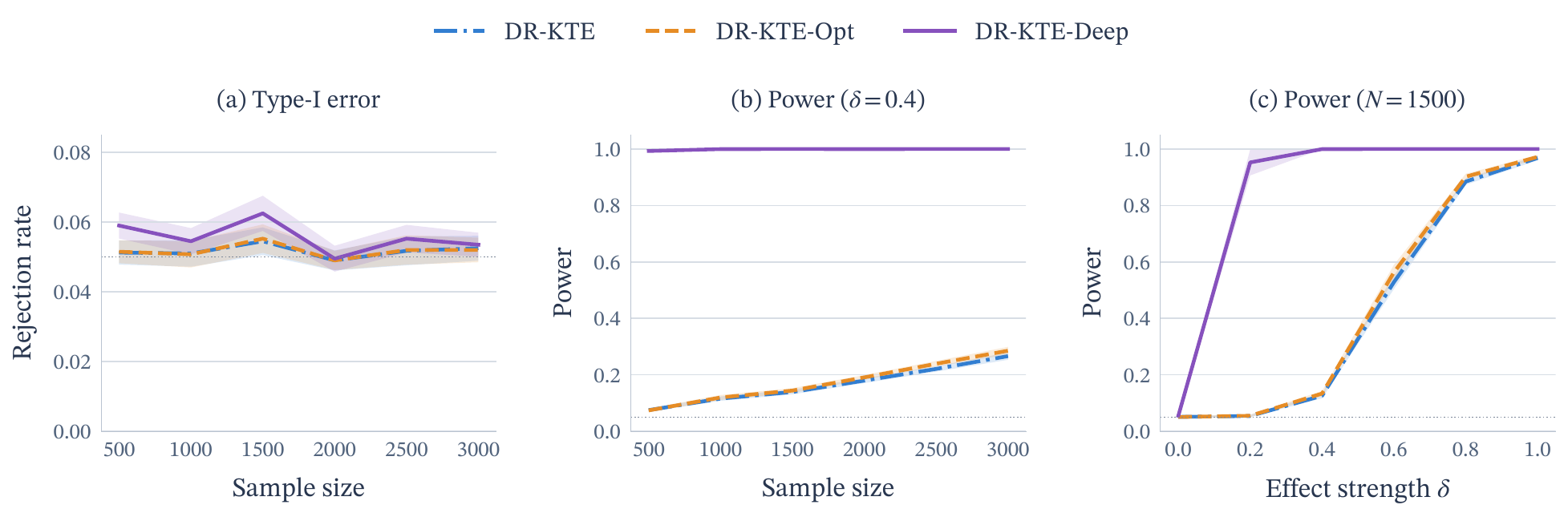}
    \caption{Results for OCT-DME. From left to right: (a) Type-I error under $H_0$ ($\delta = 0$), (b) average test power when varying $N$ while keeping $\delta=0.4$ fixed, (c) average test power when varying $\delta$ while keeping $N=1500$ fixed. Shaded bands show $\pm$ one standard error across 20 independent outer repetitions. The dotted grey line marks $\alpha=0.05$.}
    \label{fig:oct-dme-sweep}
\end{figure}

\begin{figure}[!htbp]
    \centering
    \includegraphics[width=\linewidth]{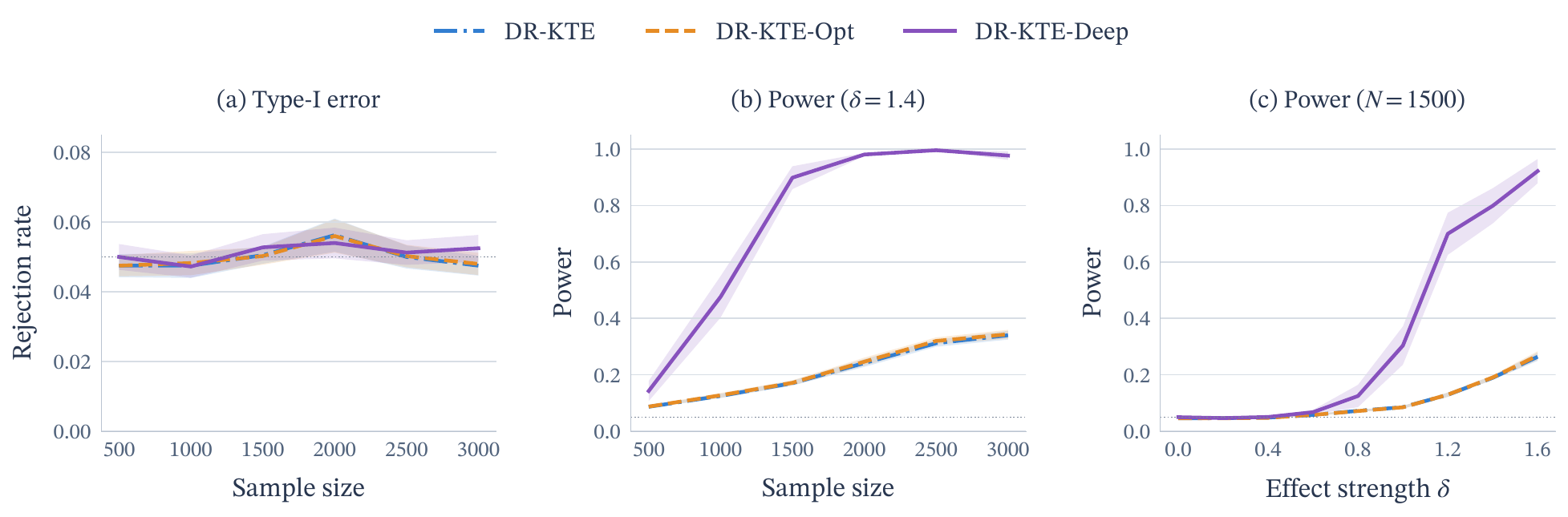}
    \caption{Results for Derma-Melanoma. From left to right: (a) Type-I error under $H_0$ ($\delta=0$), (b) average test power when varying $N$ with $\delta=1.4$ fixed, (c) average test power when varying $\delta$ with $N=1500$ fixed. Shaded bands show $\pm$ one standard error across 20 independent outer repetitions. The dotted grey line marks $\alpha=0.05$.}
    \label{fig:derma-melanoma-sweep}
\end{figure}

\FloatBarrier

\paragraph{MNIST, RetinaMNIST, BUSI.} We next study a construction where treatment changes the sampling frequency of image labels of the MNIST \citep{lecun1998gradient}, RetinaMNIST \citep{medmnistv1,medmnistv2,retinamnist}, and BUSI \citep{breastmnist} datasets.

Covariates and treatments are sampled as in \Cref{eq:medical-image-covariate-generation} and \Cref{eq:medical-image-treatment-assignment}. We define the scalar pseudo-outcomes
\begin{equation}
\label{eq:pseudo-outcomes}
    \tilde Y_i^0 = \boldsymbol{\beta}^\top X_i + \varepsilon_{i0}, \quad \tilde Y_i^1 = \boldsymbol{\beta}^\top X_i + b_i(\delta) + \varepsilon_{i1}, \quad \varepsilon_{i0},\varepsilon_{i1}
    \overset{\mathrm{i.i.d.}}{\sim}\mathcal{N}(0,\sigma_\varepsilon^2),
\end{equation}
where $d_X=5$, $\boldsymbol{\alpha}=(0.90,-0.75,0.55,-0.40,0.30)^\top$, $\boldsymbol{\beta}=(0.60,-0.50,0.40,-0.30,0.20)^\top$ and $\sigma_\varepsilon=0.1$.

We consider the three scenarios of \citet{muandet2021counterfactual}. For Scenario I, $b_i(\delta) = 0$. For Scenario II, $b_i(\delta) = \delta$. For Scenario III, $b_i(\delta) = \delta(2Z_i - 1)$, where $Z_i \sim_\mathrm{i.i.d.} \operatorname{Bernoulli}(0.5)$. Scenario I thus tests $H_0$, while Scenarios II and III test $H_1$. The pseudo-outcomes are converted to discrete image labels $\ell_i^a$ within each dataset's range. $Y_i^a$ is then generated by sampling an image with label $\ell_i^a$. We retain the outcome under the assigned treatment, $Y_i=Y_i^{A_i}$. The generative process is summarised in \Cref{fig:image-label-dag}, and \Cref{fig:image-label-samples} illustrates the construction.

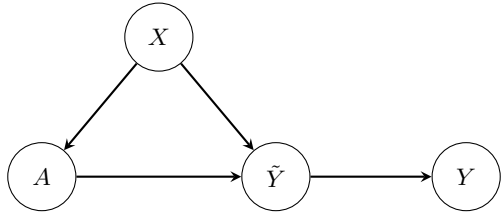
\begin{figure}[!htb]
    \centering
    \begin{tikzpicture}[
        >=stealth,
        node/.style={draw, circle, minimum size=9mm, inner sep=1pt, font=\small},
        every edge/.style={draw, ->, thick}
    ]
    \node[node] (X)  {$X$};
    \node[node] (A)  [below left=12mm and 9mm of X]  {$A$};
    \node[node] (Yt) [below right=12mm and 9mm of X] {$\tilde{Y}$};
    \node[node] (Y)  [right=16mm of Yt] {$Y$};
    \draw (X)  edge (A);
    \draw (X)  edge (Yt);
    \draw (A)  edge (Yt);
    \draw (Yt) edge node[above, font=\scriptsize] {} (Y);
    \end{tikzpicture}
    \caption{Data-generating process for MNIST, RetinaMNIST and BUSI. The covariate $X$ and treatment $A$ generate a scalar pseudo-outcome $\tilde{Y}$, which is discretised to select the label of the sampled image outcome $Y$.}
    \label{fig:image-label-dag}
\end{figure}

\begin{figure}[!htb]
    \centering
    \includegraphics[width=\linewidth]{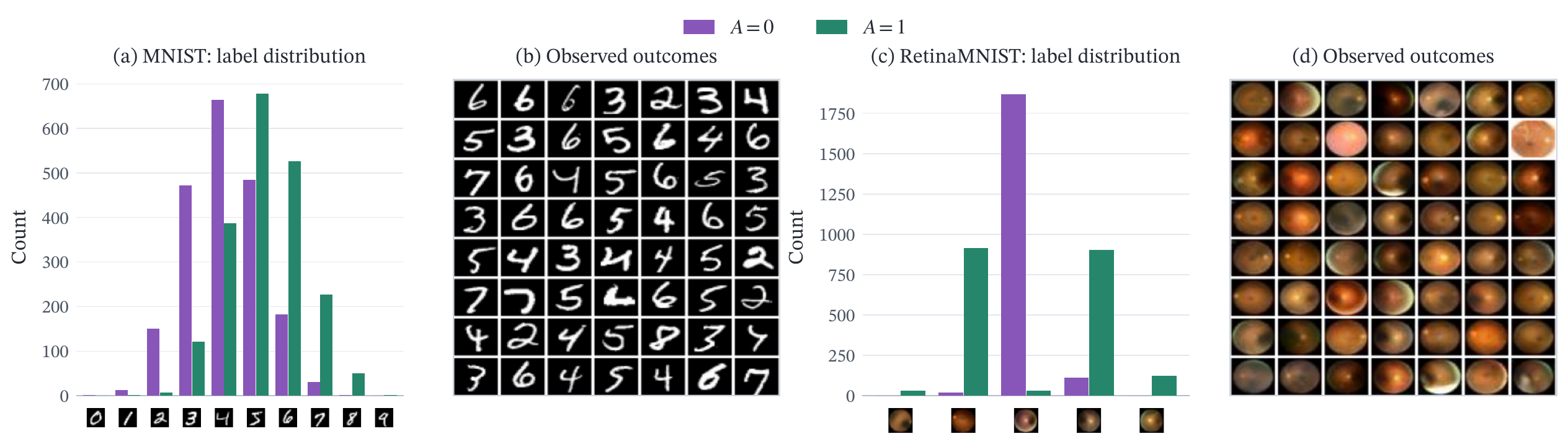}
    \caption{MNIST (Scenario II, left) and RetinaMNIST (Scenario III, right). Panels (a) and (c) show label distributions under each treatment. Panels (b) and (d) show sampled image outcomes.}
    \label{fig:image-label-samples}
\end{figure}

We report average Type-I error for Scenario I and average power for Scenarios II and III in \Cref{tab:image-label-results} for the three datasets at $N=1500$ and $N=3000$. Type-I error ranges from $0.045$ to $0.054$, close to the nominal rate of $\alpha=0.05$. Under the alternatives, DR-KTE-Deep has higher estimated power than the baselines in every setting, with the largest improvement on RetinaMNIST in Scenario III.

\begin{table}[!htb]
    \centering
    \begin{tabular}{lccccc}
    \toprule
    Dataset & $N$ & Scenario & DR-KTE & DR-KTE-Opt & DR-KTE-Deep \\
    \midrule
    \multirow{6}{*}{MNIST}
    & \multirow{3}{*}{1500} & I   & $0.046\pmerr{0.003}$ & $0.046\pmerr{0.003}$ & $0.051\pmerr{0.003}$ \\
    &                       & II  & $0.731\pmerr{0.026}$ & $0.731\pmerr{0.026}$ & $\mathbf{0.870\pmerr{0.024}}$ \\
    &                       & III & $0.721\pmerr{0.024}$ & $0.721\pmerr{0.024}$ & $\mathbf{0.803\pmerr{0.023}}$ \\
    \cmidrule(l){2-6}
    & \multirow{3}{*}{3000} & I   & $0.049\pmerr{0.003}$ & $0.049\pmerr{0.003}$ & $0.049\pmerr{0.004}$ \\
    &                       & II  & $0.948\pmerr{0.008}$ & $0.948\pmerr{0.008}$ & $\mathbf{0.986\pmerr{0.002}}$ \\
    &                       & III & $0.938\pmerr{0.010}$ & $0.937\pmerr{0.010}$ & $\mathbf{0.973\pmerr{0.006}}$ \\
    \midrule
    \multirow{6}{*}{RetinaMNIST}
    & \multirow{3}{*}{1500} & I   & $0.053\pmerr{0.003}$ & $0.052\pmerr{0.003}$ & $0.051\pmerr{0.004}$ \\
    &                       & II  & $0.585\pmerr{0.017}$ & $0.587\pmerr{0.017}$ & $\mathbf{0.647\pmerr{0.022}}$ \\
    &                       & III & $0.474\pmerr{0.020}$ & $0.486\pmerr{0.022}$ & $\mathbf{0.732\pmerr{0.035}}$ \\
    \cmidrule(l){2-6}
    & \multirow{3}{*}{3000} & I   & $0.053\pmerr{0.004}$ & $0.054\pmerr{0.003}$ & $0.050\pmerr{0.003}$ \\
    &                       & II  & $0.798\pmerr{0.013}$ & $0.799\pmerr{0.013}$ & $\mathbf{0.881\pmerr{0.011}}$ \\
    &                       & III & $0.743\pmerr{0.014}$ & $0.759\pmerr{0.014}$ & $\mathbf{0.939\pmerr{0.009}}$ \\
    \midrule
    \multirow{4}{*}{BUSI}
    & \multirow{2}{*}{1500} & I   & $0.048\pmerr{0.003}$ & $0.049\pmerr{0.003}$ & $0.046\pmerr{0.004}$ \\
    &                       & III & $0.878\pmerr{0.014}$ & $0.880\pmerr{0.014}$ & $\mathbf{0.911\pmerr{0.013}}$ \\
    \cmidrule(l){2-6}
    & \multirow{2}{*}{3000} & I   & $0.045\pmerr{0.002}$ & $0.046\pmerr{0.002}$ & $0.049\pmerr{0.004}$ \\
    &                       & III & $0.954\pmerr{0.006}$ & $0.956\pmerr{0.006}$ & $\mathbf{0.973\pmerr{0.005}}$ \\
    \bottomrule
    \end{tabular}
    \caption{Average rejection rates for MNIST, RetinaMNIST, and BUSI image outcomes at $\alpha=0.05$. Scenario I reports Type-I error; Scenarios II and III report power. Effect strengths are $\delta=1.0$ and $1.5$ for MNIST II and III, $\delta=2.0$ and $4.0$ for RetinaMNIST II and III, and $\delta=4.0$ for BUSI III. Values are means $\pm$ one standard error across 20 independent outer repetitions.}
    \label{tab:image-label-results}
\end{table}

\FloatBarrier
\paragraph{Fake-MNIST.}

We now consider a treatment that changes the source dataset of the sampled image, adapting the comparison of \citet{liu2021learningdeepkernelsnonparametric} between real MNIST images and images generated by a generative adversarial network (GAN). We use a single digit covariate $X_i\in\{0,\dots,9\}$, sampled uniformly, and assign treatment according to
\[
    A_i \mid X_i \sim \operatorname{Bernoulli}\{\pi_0(1 \mid X_i)\},
    \qquad
    \pi_0(1 \mid X_i)
    = \operatorname{clip}\left\{\sigma\left(s_\pi\frac{X_i-4.5}{\sqrt{8.25}}\right),0.07,0.93\right\},
\]
where $4.5$ and $\sqrt{8.25}$ are the respective mean and standard deviation of a uniformly sampled digit, and $s_\pi$ is a scaling hyperparameter.

Under $H_0$, we sample real MNIST images of digit $X_i$ under both interventions. Under $H_1$, we sample a GAN-generated image of digit $X_i$ with probability
\[
    r(A_i,X_i) = A_i\frac{X_i+1}{s_r},
\]
and a real MNIST image otherwise, where $s_r$ is a scaling hyperparameter. Outcomes under $A_i=0$ are thus always real images, while outcomes under $A_i=1$ can be real or GAN-generated. Larger digits are more likely to be assigned $A_i=1$ and, under $A_i=1$, to be replaced by a GAN-generated image.

Average Type-I error and power at $N=1500$ and $N=3000$ are reported in \Cref{tab:fake-mnist-results}. Under $H_1$, DR-KTE-Deep exceeds all baselines at both sample sizes. DR-KTE-Opt slightly improves power over DR-KTE. Under $H_0$, rejection rates range from $0.050$ to $0.054$ across all methods and sample sizes, close to the nominal level.

\begin{table}[!htb]
    \centering
    \begin{tabular}{clccc}
    \toprule
    $N$ & Scenario & DR-KTE & DR-KTE-Opt & DR-KTE-Deep \\
    \midrule
    \multirow{2}{*}{1500}
         & $H_0$ & $0.051\pmerr{0.003}$ & $0.052\pmerr{0.003}$ & $0.050\pmerr{0.004}$ \\
         & $H_1$ & $0.225\pmerr{0.008}$ & $0.229\pmerr{0.008}$ & $\mathbf{0.380\pmerr{0.050}}$ \\
    \midrule
    \multirow{2}{*}{3000}
         & $H_0$ & $0.054\pmerr{0.004}$ & $0.054\pmerr{0.004}$ & $0.050\pmerr{0.003}$ \\
         & $H_1$ & $0.502\pmerr{0.012}$ & $0.514\pmerr{0.012}$ & $\mathbf{0.961\pmerr{0.019}}$ \\
    \bottomrule
    \end{tabular}
    \caption{Type-I error ($H_0$) and power ($H_1$) for Fake-MNIST at $\alpha=0.05$. Values are means $\pm$ one standard error across 20 independent outer repetitions.}
    \label{tab:fake-mnist-results}
\end{table}

\FloatBarrier
\subsection{Causal-Blob}
\label{sec:causal-blob-experiments}

We evaluate our method on a Gaussian mixture dataset adapted from the Blob dataset of \citet{liu2021learningdeepkernelsnonparametric} to a causal setting with covariates. For each unit, we sample a mode index $M_i \in \{1, \dots, 9\}$ uniformly with centres $\{\mu_m\}^9_{m=1} = \{(j, k) : j, k \in \{0, 1, 2\}\}$. The covariate is then a noisy version of the mode location:
\[
    X_i = \mu_{M_i} + \varepsilon_X, \quad \varepsilon_X \sim \mathcal{N}(0, \sigma_X^2 I_2).
\]
Treatment is sampled as in \Cref{eq:medical-image-treatment-assignment}. The potential outcomes follow the mode-centred Gaussians
\[
    Y_i^0 \mid M_i = m \sim \mathcal{N}(\mu_m, \Sigma_P), \qquad Y_i^1 \mid M_i = m \sim \mathcal{N}(\mu_m, \Sigma_{Q,m}),
\]
where $\Sigma_P = 0.03 I_2$ is isotropic and shared across modes, while $\Sigma_{Q,m}$ has the same diagonal entries as $\Sigma_P$ but a mode-dependent off-diagonal entry $\rho_m$. We use $\sigma_X=0.3$ and $\boldsymbol{\alpha}=(1.2,-0.8)^\top$. The two outcome distributions thus have the same mean but differ in higher-order moments under $H_1$.

\Cref{tab:blob-n1000} shows that DR-KTE-Deep achieves the highest power on Causal-Blob at both sample sizes, with the largest gain over DR-KTE-Opt at $N=3000$. DR-KTE-Opt also improves power over DR-KTE. Type-I error for all three tests remains close to the nominal level at both sample sizes.

\begin{table}[!htbp]
    \centering
    \small
    \setlength{\tabcolsep}{4pt}
    \resizebox{\linewidth}{!}{%
    \begin{tabular}{@{}lccc@{}}
    \toprule
     & \multicolumn{3}{c}{Power/Type-I error ($H_1/H_0$)} \\
    \cmidrule(lr){2-4}
    $N$ & DR-KTE & DR-KTE-Opt & DR-KTE-Deep \\
    \midrule
    1500 & $0.322\pmerr{0.014}/0.049\pmerr{0.004}$ & $0.351\pmerr{0.016}/0.051\pmerr{0.003}$ & $\mathbf{0.382}\pmerr{0.024}/0.050\pmerr{0.004}$ \\
    3000 & $0.710\pmerr{0.020}/0.046\pmerr{0.003}$ & $0.780\pmerr{0.022}/0.048\pmerr{0.003}$ & $\mathbf{0.891}\pmerr{0.019}/0.045\pmerr{0.003}$ \\
    \bottomrule
    \end{tabular}
    }
    \caption{Power and Type-I error for Causal-Blob at $\alpha=0.05$. Each cell reports $H_1/H_0$, with means $\pm$ one standard error across 20 independent outer repetitions.}
    \label{tab:blob-n1000}
\end{table}

\FloatBarrier
\subsection{Scalar Outcomes}
\label{sec:low-dimensional-outcomes}

We also evaluate our method on scalar outcomes. We adapt the scalar-outcome scenarios of \citet{muandet2021counterfactual} with binary treatments. Covariates and treatments are sampled as in \Cref{eq:medical-image-covariate-generation,eq:medical-image-treatment-assignment}, and we use the scalar outcomes of \Cref{eq:pseudo-outcomes} directly. Scenarios I, II and III cover the null hypothesis, a mean shift and a mean-preserving distributional shift respectively. We additionally consider Scenario IV of \citet{zenati2025cpme}, with $b_i(\delta)=4\delta Z_i$, where $Z_i\sim\operatorname{Bernoulli}(0.5)$. This represents a change in both the mean and variance of the outcome distributions.

\Cref{tab:continuous-n1000} reports average Type-I error for Scenario I and average power for Scenarios II-IV. Type-I error is close to the nominal level for the kernel tests at both sample sizes. Power is close to one in Scenarios II and IV, with little improvement from learning a deep kernel. Scenario III is more challenging, although power increases substantially with sample size. DR-KTE-Opt achieves the highest estimated power in Scenario III at both sample sizes.

\begin{table}[!htbp]
    \centering
    \small
    \setlength{\tabcolsep}{4pt}
    \resizebox{\linewidth}{!}{%
    \begin{tabular}{@{}clccc@{}}
    \toprule
     & & \multicolumn{3}{c}{Power/Type-I error ($H_1/H_0$)} \\
    \cmidrule(lr){3-5}
    $N$ & Scenario & DR-KTE & DR-KTE-Opt & DR-KTE-Deep \\
    \midrule
    \multirow{3}{*}{1500}
         & II  & $0.999\pmerr{0.000}/0.055\pmerr{0.004}$ & $\mathbf{1.000}\pmerr{0.000}/0.059\pmerr{0.004}$ & $0.995\pmerr{0.002}/0.055\pmerr{0.004}$ \\
         & III & $0.611\pmerr{0.028}/0.055\pmerr{0.004}$ & $\mathbf{0.653}\pmerr{0.027}/0.059\pmerr{0.004}$ & $0.614\pmerr{0.028}/0.055\pmerr{0.004}$ \\
         & IV  & $0.980\pmerr{0.007}/0.055\pmerr{0.004}$ & $\mathbf{0.985}\pmerr{0.005}/0.059\pmerr{0.004}$ & $0.981\pmerr{0.004}/0.055\pmerr{0.004}$ \\
    \midrule
    \multirow{3}{*}{3000}
         & II  & $\mathbf{1.000}\pmerr{0.000}/0.048\pmerr{0.003}$ & $\mathbf{1.000}\pmerr{0.000}/0.046\pmerr{0.003}$ & $\mathbf{1.000}\pmerr{0.000}/0.048\pmerr{0.003}$ \\
         & III & $0.903\pmerr{0.013}/0.048\pmerr{0.003}$ & $\mathbf{0.927}\pmerr{0.012}/0.046\pmerr{0.003}$ & $0.914\pmerr{0.013}/0.048\pmerr{0.003}$ \\
         & IV  & $\mathbf{0.994}\pmerr{0.002}/0.048\pmerr{0.003}$ & $\mathbf{0.994}\pmerr{0.002}/0.046\pmerr{0.003}$ & $0.990\pmerr{0.002}/0.048\pmerr{0.003}$ \\
    \bottomrule
    \end{tabular}
    }
    \caption{Power and Type-I error for scalar outcomes at $\alpha=0.05$, with $\delta=0.8$ and $M=5$. Each cell reports $H_1/H_0$, with means $\pm$ one standard error across 20 independent outer repetitions.}
    \label{tab:continuous-n1000}
\end{table}

\subsection{Estimator Comparison}
\label{sec:ablation-estimator-comparison}

We study how the choice of estimator affects test power and Type-I error, using the same learned outcome kernel for all methods. We consider the continuous-treatment construction of \citet{zenati2025cpme}. We compare doubly robust estimation with inverse-propensity weighting (IPW) and a kernel policy test. For each unit, we sample the covariate as
\[
    X_i\sim\mathcal N(0,I_5),
\]
and treatment is assigned as
\[
    A_i\mid X_i\sim\mathcal N(X_i^\top w,1).
\]
The outcome is then
\[
    Y_i=X_i^\top\beta+A_i+0.1\varepsilon_i, \quad \varepsilon_i\sim\mathcal N(0,1),
\]
where $w,\beta\in\mathbb R^5$ are fixed coefficient vectors.

We compare two interventions $\pi$ and $\pi'$ and set $\pi:=\pi_0$. We consider four policy-test scenarios adapted from \citet{zenati2025cpme} with continuous treatments. Scenario I tests the null hypothesis with $\pi'(\cdot\mid x)=\mathcal N(x^\top w_{\mathrm I},1)$, where $w_{\mathrm I}\neq w$ and $\|\beta+w_{\mathrm I}\|=\|\beta+w\|$, so the two interventions induce the same marginal outcome distribution. In Scenario II, $\pi'(\cdot\mid x)=\mathcal N(x^\top(w+2\mathbbm{1}_5),1)$. In Scenario III, $\pi'(\cdot\mid x)$ is an equal mixture of the distributions $\mathcal N(x^\top(w+\mathbbm{1}_5),1)$ and $\mathcal N(x^\top(w-\mathbbm{1}_5),1)$. Finally, in Scenario IV, $\pi'(\cdot\mid x)$ is an equal mixture of the distributions $\mathcal N(x^\top(w+2\mathbbm{1}_5),1)$ and $\mathcal N(x^\top w,1)$.

We compare our doubly robust estimator, DR-KTE-Deep, to a cross-U IPW estimator, IPW-Cross, obtained by setting $\hat \mu_\theta(a,x) := 0$ in the plug-in estimator of \Cref{eq:generic-dr-feature}. We additionally compare our method to a kernel policy test, KPT, following \citet{zenati2025cpme}. All methods use the same deep outcome kernel, learned with the doubly robust criterion. We divide the data into four equal samples for learning, validation, nuisance fitting and testing. For this comparison, checkpoint selection uses variance regularisation $10^{-4}$, and the final cross-U statistics are unregularised. KPT uses 10,000 permutations to calibrate its test. These choices hold the learned kernel and data budget fixed while comparing the test estimators.

\Cref{tab:ablation-estimator-comparison} shows that DR-KTE-Deep achieves higher estimated power than IPW-Cross and KPT in all three alternative scenarios, with the largest gains in Scenario III. Under the null hypothesis, its rejection rate is close to the nominal level of $\alpha=0.05$, while IPW-Cross and KPT have higher Type-I error.

\begin{table}[!htbp]
    \centering
    \small
    \setlength{\tabcolsep}{4pt}
    \begin{tabular}{@{}lccc@{}}
    \toprule
     & \multicolumn{3}{c}{Power/Type-I error ($H_1/H_0$)} \\
    \cmidrule(lr){2-4}
    Scenario & DR-KTE-Deep & IPW-Cross & KPT \\
    \midrule
    Scenario II & $\mathbf{0.961}\pmerr{0.007}/0.048\pmerr{0.003}$ & $0.920\pmerr{0.012}/0.085\pmerr{0.009}$ & $0.638\pmerr{0.022}/0.070\pmerr{0.013}$ \\
    Scenario III & $\mathbf{0.739}\pmerr{0.011}/0.048\pmerr{0.003}$ & $0.574\pmerr{0.008}/0.085\pmerr{0.009}$ & $0.291\pmerr{0.017}/0.070\pmerr{0.013}$ \\
    Scenario IV & $\mathbf{0.945}\pmerr{0.009}/0.048\pmerr{0.003}$ & $0.903\pmerr{0.013}/0.085\pmerr{0.009}$ & $0.607\pmerr{0.027}/0.070\pmerr{0.013}$ \\
    \bottomrule
    \end{tabular}
    \caption{Power and Type-I error at $\alpha=0.05$, $N=1600$ and $M=1$ for different choices of estimator. Each cell reports $H_1/H_0$, with means $\pm$ one standard error across 20 independent outer repetitions.}
    \label{tab:ablation-estimator-comparison}
\end{table}

\FloatBarrier

\subsection{Robustness to Outcome Model Misspecification}
\label{sec:ablation-outcome-model-misspecification}

We analyse the contribution of fitting the conditional outcome embedding with the logging policy known. We compare the fitted embedding with $\hat\mu_\theta(a,x) \equiv 0$, while maintaining $\hat \pi_0(a \mid x) \equiv \pi_0(a \mid x)$ and a learned deep outcome kernel. We evaluate the comparison on Causal-Blob (\Cref{sec:causal-blob-experiments}) and the four scalar-outcome scenarios of \Cref{sec:low-dimensional-outcomes}.

\Cref{tab:ablation-outcome-model-misspecification} shows similar Type-I error for the fitted and zero outcome embeddings, with rejection rates close to the nominal level. Using the fitted outcome embedding improves power on Causal-Blob and in Scenarios II and III, with the largest gain in Scenario III. Both estimators achieve power close to one in Scenario IV.

\begin{table}[!htbp]
    \centering
    \small
    \setlength{\tabcolsep}{4pt}
    \begin{tabular}{@{}lcc@{}}
    \toprule
     & \multicolumn{2}{c}{Power/Type-I error ($H_1/H_0$)} \\
    \cmidrule(lr){2-3}
    Environment & Fitted $\hat\mu_\theta$ & $\hat\mu_\theta\equiv0$ \\
    \midrule
    Causal-Blob & $0.318\pmerr{0.024}/0.047\pmerr{0.006}$ & $0.247\pmerr{0.017}/0.047\pmerr{0.007}$ \\
    \midrule
    Scenario II & $0.999\pmerr{0.001}/0.056\pmerr{0.004}$ & $0.927\pmerr{0.012}/0.057\pmerr{0.003}$ \\
    Scenario III & $0.616\pmerr{0.015}/0.056\pmerr{0.004}$ & $0.349\pmerr{0.013}/0.057\pmerr{0.003}$ \\
    Scenario IV & $1.000\pmerr{0.000}/0.056\pmerr{0.004}$ & $0.999\pmerr{0.001}/0.057\pmerr{0.003}$ \\
    \bottomrule
    \end{tabular}
    \caption{Power and Type-I error at $\alpha=0.05$, $N=1500$ and $M=1$. Each cell reports $H_1/H_0$, with means $\pm$ one standard error across 20 independent outer repetitions.}
    \label{tab:ablation-outcome-model-misspecification}
\end{table}

\FloatBarrier
\subsection{Training Objective Comparison}
\label{sec:ablation-representation-objective}

We test whether our power learning objective improves detection compared to a reconstruction objective. On OCT-Drusen and Fake-MNIST, we compare DR-KTE-Deep with a frozen randomly initialised CNN and a pixel autoencoder. All variants share the same encoder architecture, initial encoder weights, learning and validation observations, and test datasets.

DR-KTE-Deep has the highest power on both constructions under $H_1$ (\Cref{tab:ablation-representation-objective}), while Type-I error remains close to the nominal level for all methods.

\begin{table}[!htbp]
    \centering
    \small
    \setlength{\tabcolsep}{4pt}
    \begin{tabular}{@{}llccc@{}}
    \toprule
    Scenario &  & \shortstack{Random Weights} & \shortstack{Pixel Autoencoder} & DR-KTE-Deep \\
    \midrule
    \multirow{2}{*}{OCT-Drusen}
    & $H_1$ & $0.057 \pm 0.004$ & $0.052 \pm 0.004$ & $\mathbf{0.371 \pm 0.080}$ \\
    & $H_0$ & $0.054 \pm 0.004$ & $0.050 \pm 0.003$ & $0.058 \pm 0.005$ \\
    \midrule
    \multirow{2}{*}{Fake-MNIST}
    & $H_1$ & $0.204 \pm 0.008$ & $0.200 \pm 0.012$ & $\mathbf{0.380 \pm 0.050}$ \\
    & $H_0$ & $0.053 \pm 0.004$ & $0.055 \pm 0.004$ & $0.050 \pm 0.004$ \\
    \bottomrule
    \end{tabular}
    \caption{Power under $H_1$ and Type-I error under $H_0$ at $N=1500$. Values are means and one standard error across 20 independent outer repetitions.}
    \label{tab:ablation-representation-objective}
\end{table}

\FloatBarrier
\subsection{Text Representation Ablation}
\label{appendix:text-representation-ablations}

We consider an alternative representation for the text outcomes described in \Cref{sec:text-outcome-experiments}, mapping each outcome to a fixed-length vector using the Qwen3-Embedding-0.6B \citep{zhang2025qwen3embedding} embedding model. Since a language-model embedding can capture semantic information, a simple Gaussian kernel may be well suited to detecting differences in outcome distributions. 

\Cref{tab:text-embedding-results} shows that all three methods achieve high, similar power using this representation. Type-I error remains close to the nominal level for Oncology Notes and Clinical Notes, while all three methods are conservative on Movie Reviews.

However, this embedding is not guaranteed to be injective: distinct texts can collapse to the same embedding. A Gaussian kernel, even if characteristic on the embedding space, might therefore not detect every difference between the original text outcome distributions. Additionally, this representation can be more costly than a sequence of tokens, as it requires inference over an embedding model.

\begin{table}[!htbp]
    \centering
    \small
    \setlength{\tabcolsep}{4pt}
    \begin{tabular}{@{}lccc@{}}
    \toprule
    & \multicolumn{3}{c}{Power/Type-I error ($H_1/H_0$)} \\
    \cmidrule(lr){2-4}
    Scenario & DR-KTE & DR-KTE-Opt & DR-KTE-Deep \\
    \midrule
    Oncology Notes & $\mathbf{0.828}\pmerr{0.035}/0.057\pmerr{0.010}$ & $0.827\pmerr{0.035}/0.058\pmerr{0.010}$ & $\mathbf{0.828}\pmerr{0.036}/0.058\pmerr{0.010}$ \\
    Movie Reviews & $0.892\pmerr{0.009}/0.029\pmerr{0.004}$ & $\mathbf{0.895}\pmerr{0.008}/0.029\pmerr{0.004}$ & $\mathbf{0.895}\pmerr{0.008}/0.029\pmerr{0.004}$ \\
    Clinical Notes & $\mathbf{0.894}\pmerr{0.027}/0.049\pmerr{0.007}$ & $\mathbf{0.894}\pmerr{0.027}/0.049\pmerr{0.007}$ & $\mathbf{0.894}\pmerr{0.027}/0.050\pmerr{0.007}$ \\
    \bottomrule
    \end{tabular}
    \caption{Power ($H_1$) and Type-I error ($H_0$) at $N=1500$ and $\alpha=0.05$ for text outcomes represented as text embeddings.}
    \label{tab:text-embedding-results}
\end{table}

\FloatBarrier
\subsection{Effect of Ensemble Size}
\label{sec:ablation-ensemble-size}

We evaluate the effect of memberwise averaging by varying the number of separately trained kernels included in the ensemble. We perform five independent batches, each consisting of 20 independent outer repetitions. In each repetition, we train five members and form ensembles with $M\in\{1,2,3,5\}$, using the first $M$ members. All ensemble sizes share the same observations, nuisance estimates and member kernels. We use Causal-Blob and OCT-Drusen at $N=1500$ and $\delta=1.0$, following the procedure in \Cref{sec:training-testing-details}.

For each outer repetition and ensemble size, we estimate power as the rejection rate over 200 fresh test datasets under $H_1$, at $\alpha=0.05$. We then compute the mean and sample standard deviation of the 20 power estimates within each batch. For each statistic, we report its mean and one standard error across five independent batches.

On OCT-Drusen, power increases slightly with ensemble size, and its standard deviation is lower at $M=5$ than at $M=1$. On Causal-Blob, average power increases from $0.350$ to $0.394$, while variability does not decrease consistently (\Cref{tab:ablation-ensemble-size} and \Cref{fig:ablation-ensemble-size}).

While Proposition~\ref{prop:memberwise-ensemble} establishes variance reduction for the population discrepancy with the training data fixed, it does not imply that empirical power becomes less variable across repetitions.

\begin{table}[!htbp]
    \centering
    \begin{tabular}{llcccc}
        \toprule
        Dataset & Statistic & $M=1$ & $M=2$ & $M=3$ & $M=5$ \\
        \midrule
        Causal-Blob & Power & $0.350\pmerr{0.003}$ & $0.373\pmerr{0.007}$ & $0.383\pmerr{0.007}$ & $0.394\pmerr{0.008}$ \\
        & SD & $0.150\pmerr{0.011}$ & $0.157\pmerr{0.015}$ & $0.150\pmerr{0.014}$ & $0.151\pmerr{0.014}$ \\
        \midrule
        OCT-Drusen & Power & $0.993\pmerr{0.000}$ & $0.995\pmerr{0.000}$ & $0.995\pmerr{0.000}$ & $0.996\pmerr{0.001}$ \\
        & SD & $0.011\pmerr{0.001}$ & $0.009\pmerr{0.002}$ & $0.009\pmerr{0.002}$ & $0.007\pmerr{0.001}$ \\
        \bottomrule
    \end{tabular}
    \caption{Mean and standard deviation (SD) of empirical power at $N=1500$ and $\delta=1.0$. Each statistic is computed across 20 independent outer repetitions within each batch; values are means $\pm$ one standard error across five independent batches.}
    \label{tab:ablation-ensemble-size}
\end{table}

\begin{figure}[!htbp]
    \centering
    \includegraphics[width=\linewidth]{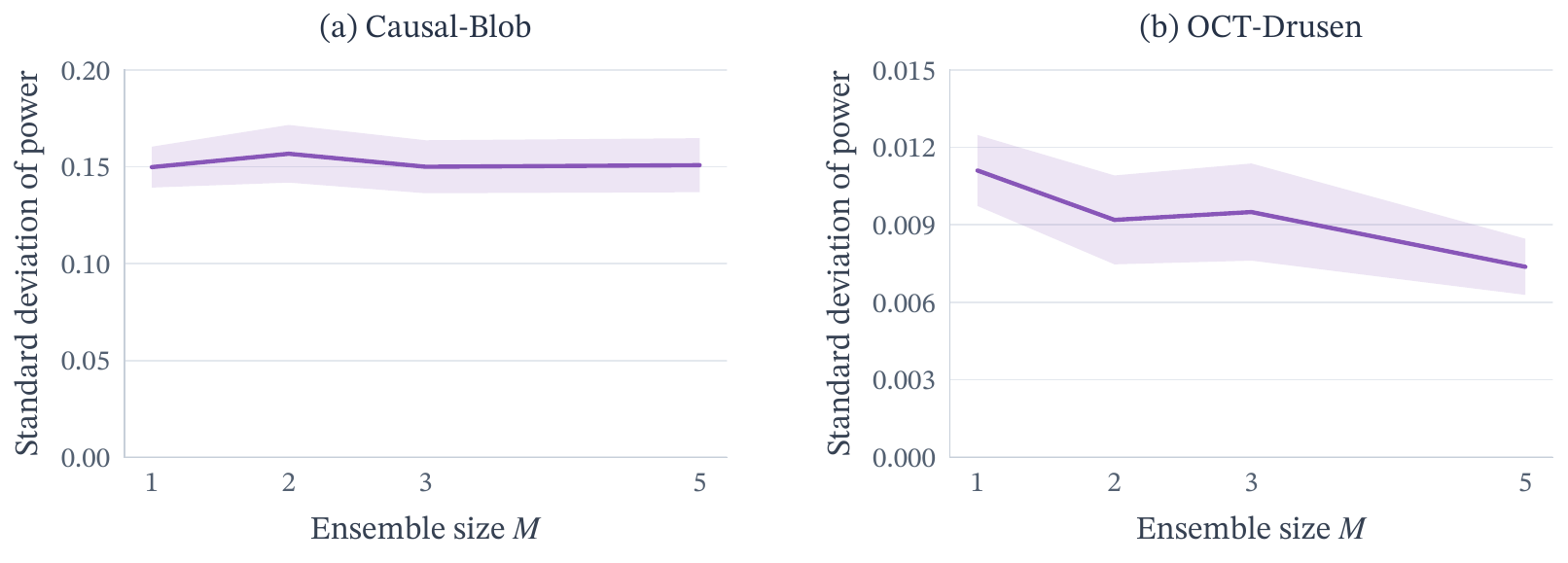}
    \caption{Average standard deviation of empirical power against ensemble size at $N=1500$ and $\delta=1.0$. Shaded bands show $\pm$ one standard error across five independent batches.}
    \label{fig:ablation-ensemble-size}
\end{figure}

\FloatBarrier
\subsection{Weisfeiler-Lehman Outcome Kernel}
\label{sec:ablation-wl-kernel}

We consider an alternative base kernel for the graph-outcome experiments. We replace the Gaussian kernel used as $k^\mathrm{char}$ with a Weisfeiler-Lehman (WL) subtree kernel \citep{shervashidze2011weisfeiler} and denote the resulting tests by DR-KTE-WL and DR-KTE-WL-Deep. Since the WL kernel is not generally characteristic, these methods do not retain the same omnibus guarantees as the Gaussian-based ones.

\Cref{tab:ablation-wl-kernel} reports average power under $H_1$ and average Type-I error under $H_0$ for both WL-based methods. Learning the outcome kernel increases power from $0.743$ to $0.964$ on LLM Code Repairs. On Contact Networks, the two variants achieve similar power, with a slightly higher average for the learned kernel. Type-I error remains close to the nominal level for both constructions and methods.

\begin{table}[!htbp]
    \centering
    \small
    \setlength{\tabcolsep}{4pt}
    \begin{tabular}{@{}lcc@{}}
    \toprule
     & \multicolumn{2}{c}{Power/Type-I error ($H_1/H_0$)} \\
    \cmidrule(lr){2-3}
    Scenario & DR-KTE-WL & DR-KTE-WL-Deep \\
    \midrule
    LLM Code Repairs & $0.743\pmerr{0.023}/0.059\pmerr{0.006}$ & $\mathbf{0.964}\pmerr{0.008}/0.050\pmerr{0.004}$ \\
    Contact Networks & $0.971\pmerr{0.005}/0.057\pmerr{0.003}$ & $\mathbf{0.972}\pmerr{0.005}/0.056\pmerr{0.004}$ \\
    \bottomrule
    \end{tabular}
    \caption{Power ($H_1$) and Type-I error ($H_0$) for graph outcomes using Weisfeiler-Lehman base kernels at $N=1500$. Each cell reports $H_1/H_0$, with means $\pm$ one standard error across 20 outer repetitions. Gaussian-kernel results are given in \Cref{tab:graph-results}.}
    \label{tab:ablation-wl-kernel}
\end{table}

\FloatBarrier
\subsection{Linear Outcome Kernel}
\label{sec:ablation-linear-kernel}

We compare DR-KTE-Deep with DR-ATE, using a linear outcome kernel $k(y,y')=\phi(y)^\top \phi(y')$, where $\phi(y)$ is the numerical representation of the structured outcome. DR-ATE therefore only tests for average treatment effects.

For the medical-image experiments, Causal-Blob and the scalar outcomes Scenario III, which construct a mean-matched but distributional effect, DR-KTE-Deep has high power, while DR-ATE rejects around the nominal rate (\Cref{tab:ablation-linear-kernel-images,tab:ablation-linear-kernel-blob,tab:ablation-linear-kernel-scalar}). In the remaining experiments, where the construction introduces a mean-shift in interventional outcomes, DR-ATE often has high power, but DR-KTE-Deep always achieves higher rejection rates (\Cref{tab:ablation-linear-kernel-image-scenarios,tab:ablation-linear-kernel-text,tab:ablation-linear-kernel-graphs}).

\begin{table}[!htbp]
    \centering
    \small
    \setlength{\tabcolsep}{4pt}
    \begin{tabular}{@{}lcc@{}}
    \toprule
     & \multicolumn{2}{c}{Power/Type-I error ($H_1/H_0$)} \\
    \cmidrule(lr){2-3}
    Scenario & DR-ATE & DR-KTE-Deep \\
    \midrule
    Breast-Tumor & $0.052\pmerr{0.003}/0.047\pmerr{0.004}$ & $\mathbf{0.989}\pmerr{0.002}/0.054\pmerr{0.003}$ \\
    Chest-Cardiomegaly & $0.044\pmerr{0.003}/0.043\pmerr{0.003}$ & $\mathbf{0.928}\pmerr{0.046}/0.051\pmerr{0.004}$ \\
    Chest-Effusion & $0.046\pmerr{0.005}/0.049\pmerr{0.003}$ & $\mathbf{0.999}\pmerr{0.000}/0.048\pmerr{0.004}$ \\
    Chest-Mass & $0.052\pmerr{0.003}/0.051\pmerr{0.004}$ & $\mathbf{1.000}\pmerr{0.000}/0.056\pmerr{0.004}$ \\
    Derma-Melanoma & $0.045\pmerr{0.004}/0.048\pmerr{0.002}$ & $\mathbf{0.898}\pmerr{0.040}/0.053\pmerr{0.004}$ \\
    OCT-CNV & $0.051\pmerr{0.004}/0.056\pmerr{0.003}$ & $\mathbf{1.000}\pmerr{0.000}/0.053\pmerr{0.003}$ \\
    OCT-DME & $0.052\pmerr{0.003}/0.059\pmerr{0.004}$ & $\mathbf{1.000}\pmerr{0.000}/0.063\pmerr{0.005}$ \\
    OCT-Drusen & $0.046\pmerr{0.003}/0.053\pmerr{0.005}$ & $\mathbf{0.371}\pmerr{0.080}/0.058\pmerr{0.005}$ \\
    Pneumonia & $0.053\pmerr{0.003}/0.047\pmerr{0.002}$ & $\mathbf{0.998}\pmerr{0.001}/0.053\pmerr{0.004}$ \\
    Retina-Severe & $0.056\pmerr{0.003}/0.044\pmerr{0.003}$ & $\mathbf{0.959}\pmerr{0.008}/0.052\pmerr{0.003}$ \\
    \bottomrule
    \end{tabular}
    \caption{Comparison with DR-ATE for medical-image outcomes at $N=1500$. Each cell reports $H_1/H_0$, with means $\pm$ one standard error across 20 outer repetitions.}
    \label{tab:ablation-linear-kernel-images}
\end{table}

\begin{table}[!htbp]
    \centering
    \small
    \setlength{\tabcolsep}{4pt}
    \begin{tabular}{@{}llcc@{}}
    \toprule
    Dataset & Scenario & DR-ATE & DR-KTE-Deep \\
    \midrule
    \multirow{3}{*}{MNIST}
        & I   & $0.045\pmerr{0.003}$ & $0.051\pmerr{0.003}$ \\
        & II  & $0.677\pmerr{0.027}$ & $\mathbf{0.870}\pmerr{0.024}$ \\
        & III & $0.695\pmerr{0.023}$ & $\mathbf{0.803}\pmerr{0.023}$ \\
    \midrule
    \multirow{3}{*}{RetinaMNIST}
        & I   & $0.052\pmerr{0.003}$ & $0.051\pmerr{0.004}$ \\
        & II  & $0.434\pmerr{0.016}$ & $\mathbf{0.647}\pmerr{0.022}$ \\
        & III & $0.235\pmerr{0.010}$ & $\mathbf{0.732}\pmerr{0.035}$ \\
    \midrule
    \multirow{2}{*}{BUSI}
        & I   & $0.054\pmerr{0.003}$ & $0.046\pmerr{0.004}$ \\
        & III & $0.771\pmerr{0.016}$ & $\mathbf{0.911}\pmerr{0.013}$ \\
    \midrule
    \multirow{2}{*}{Fake-MNIST}
        & $H_0$ & $0.052\pmerr{0.003}$ & $0.050\pmerr{0.004}$ \\
        & $H_1$ & $0.114\pmerr{0.004}$ & $\mathbf{0.380}\pmerr{0.050}$ \\
    \bottomrule
    \end{tabular}
    \caption{Comparison with DR-ATE for MNIST, RetinaMNIST, BUSI and Fake-MNIST at $N=1500$. Scenario I and Fake-MNIST $H_0$ report Type-I error; Scenarios II and III and Fake-MNIST $H_1$ report power. Values are means $\pm$ one standard error across 20 outer repetitions.}
    \label{tab:ablation-linear-kernel-image-scenarios}
\end{table}

\begin{table}[!htbp]
    \centering
    \small
    \setlength{\tabcolsep}{4pt}
    \begin{tabular}{@{}lcc@{}}
    \toprule
     & \multicolumn{2}{c}{Power/Type-I error ($H_1/H_0$)} \\
    \cmidrule(lr){2-3}
    Scenario & DR-ATE & DR-KTE-Deep \\
    \midrule
    Oncology Notes & $0.375\pmerr{0.015}/0.051\pmerr{0.003}$ & $\mathbf{0.809}\pmerr{0.023}/0.055\pmerr{0.003}$ \\
    Movie Reviews & $0.141\pmerr{0.009}/0.048\pmerr{0.003}$ & $\mathbf{0.449}\pmerr{0.026}/0.053\pmerr{0.004}$ \\
    Clinical Notes & $0.053\pmerr{0.004}/0.053\pmerr{0.003}$ & $\mathbf{0.939}\pmerr{0.013}/0.058\pmerr{0.004}$ \\
    \bottomrule
    \end{tabular}
    \caption{Comparison with DR-ATE for text outcomes at $N=1500$. Each cell reports $H_1/H_0$, with means $\pm$ one standard error across 20 outer repetitions.}
    \label{tab:ablation-linear-kernel-text}
\end{table}

\begin{table}[!htbp]
    \centering
    \small
    \setlength{\tabcolsep}{4pt}
    \begin{tabular}{@{}lcc@{}}
    \toprule
     & \multicolumn{2}{c}{Power/Type-I error ($H_1/H_0$)} \\
    \cmidrule(lr){2-3}
    Scenario & DR-ATE & DR-KTE-Deep \\
    \midrule
    LLM Code Repairs & $0.874\pmerr{0.019}/0.056\pmerr{0.004}$ & $\mathbf{1.000}\pmerr{0.000}/0.051\pmerr{0.003}$ \\
    Contact Networks & $0.923\pmerr{0.009}/0.059\pmerr{0.004}$ & $\mathbf{0.962}\pmerr{0.006}/0.051\pmerr{0.002}$ \\
    \bottomrule
    \end{tabular}
    \caption{Comparison with DR-ATE for graph outcomes at $N=1500$. Each cell reports $H_1/H_0$, with means $\pm$ one standard error across 20 outer repetitions.}
    \label{tab:ablation-linear-kernel-graphs}
\end{table}

\begin{table}[!htbp]
    \centering
    \small
    \setlength{\tabcolsep}{4pt}
    \begin{tabular}{@{}lcc@{}}
    \toprule
     & \multicolumn{2}{c}{Power/Type-I error ($H_1/H_0$)} \\
    \cmidrule(lr){2-3}
    Scenario & DR-ATE & DR-KTE-Deep \\
    \midrule
    Causal-Blob & $0.049\pmerr{0.003}/0.043\pmerr{0.002}$ & $\mathbf{0.382}\pmerr{0.024}/0.050\pmerr{0.004}$ \\
    \bottomrule
    \end{tabular}
    \caption{Comparison with DR-ATE on Causal-Blob at $N=1500$. Each cell reports $H_1/H_0$, with means $\pm$ one standard error across 20 outer repetitions.}
    \label{tab:ablation-linear-kernel-blob}
\end{table}

\begin{table}[!htbp]
    \centering
    \small
    \setlength{\tabcolsep}{4pt}
    \begin{tabular}{@{}lcc@{}}
    \toprule
     & \multicolumn{2}{c}{Power/Type-I error ($H_1/H_0$)} \\
    \cmidrule(lr){2-3}
    Scenario & DR-ATE & DR-KTE-Deep \\
    \midrule
    Scenario II & $0.939\pmerr{0.009}/0.097\pmerr{0.007}$ & $\mathbf{0.995}\pmerr{0.002}/0.055\pmerr{0.004}$ \\
    Scenario III & $0.050\pmerr{0.004}/0.097\pmerr{0.007}$ & $\mathbf{0.614}\pmerr{0.028}/0.055\pmerr{0.004}$ \\
    Scenario IV & $0.972\pmerr{0.006}/0.097\pmerr{0.007}$ & $\mathbf{0.981}\pmerr{0.004}/0.055\pmerr{0.004}$ \\
    \bottomrule
    \end{tabular}
    \caption{Comparison with DR-ATE for scalar outcomes at $N=1500$. Each alternative is paired with the shared Scenario I null. Each cell reports $H_1/H_0$, with means $\pm$ one standard error across 20 outer repetitions.}
    \label{tab:ablation-linear-kernel-scalar}
\end{table}

\FloatBarrier
\subsection{Scaling with Outcome Dimension}
\label{sec:ablation-outcome-dimension}

We compare DR-ATE, DR-KTE, DR-KTE-Opt and DR-KTE-Deep as the outcome dimension increases. We use Causal-Blob (\Cref{sec:causal-blob-experiments}), where we vary the outcome dimension by appending independent noise coordinates to the two-dimensional outcomes. For $d>2$, the augmented outcome is
\[
    Y_i^{(d)} = \begin{pmatrix} Y_i^{(2)} \\ Z_i \end{pmatrix}, \quad Z_i\sim\mathcal N(0,0.17^2 I_{d-2}), Y_i^{(2)} = (1-A_i)Y_i^0+A_iY_i^1,
\]
where $Z_i$ is independent of the covariates, treatment and original outcome $Y_i^{(2)}$. Since the appended coordinates $Z_i$ have the same distribution under both treatments, the treatment effect is confined to the first two coordinates.

\Cref{fig:experiment-6} shows that DR-KTE-Opt improves power over DR-KTE, while DR-KTE-Deep achieves substantially higher power through $d=8$. DR-ATE has little power beyond its null rejection rate throughout, as expected for this mean-preserving alternative. At $d=16$ and $d=32$, all four methods have little power beyond their null rejection rates. Type-I error is above the nominal level for all four methods.

\begin{figure}[!htbp]
    \centering
    \includegraphics[width=\linewidth]{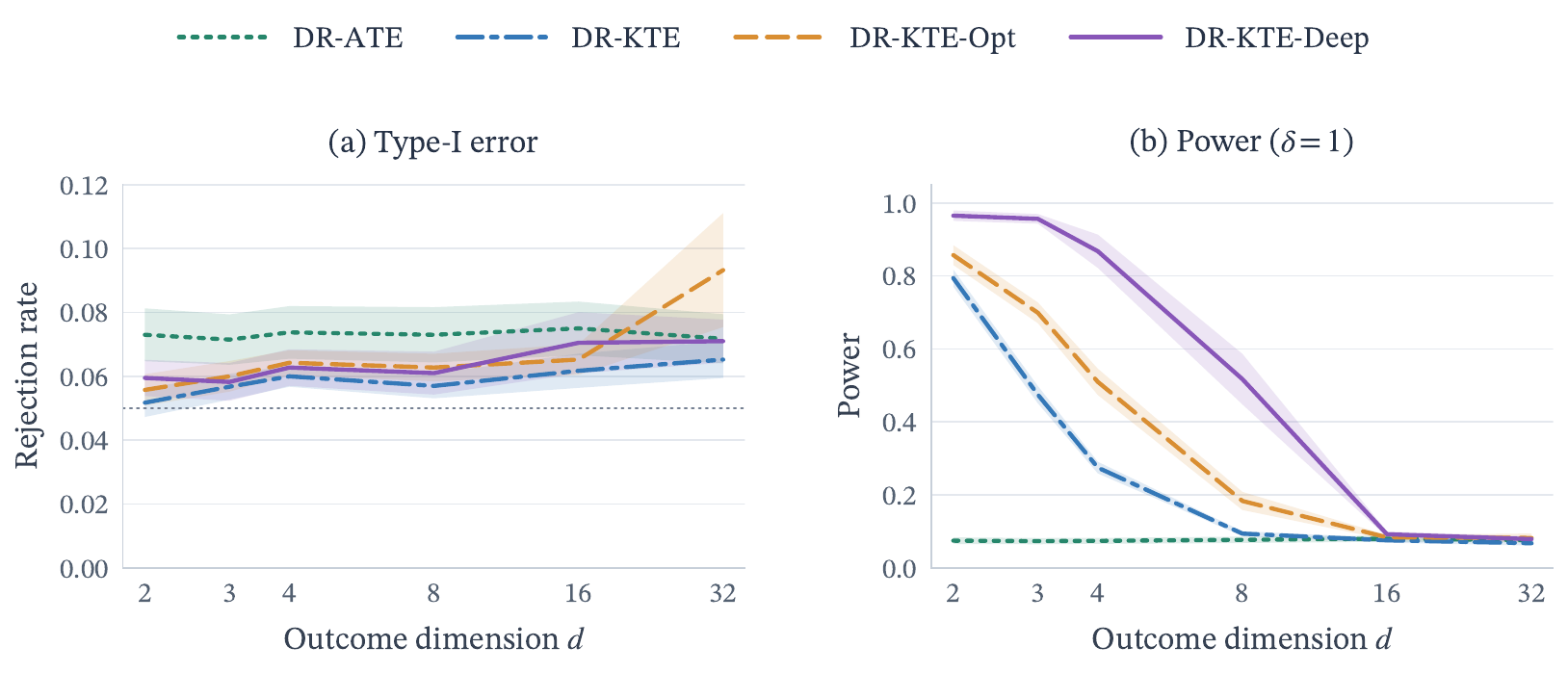}
    \caption{Causal-Blob at $N=4000$ when varying the outcome dimension, with $\delta=0$ under $H_0$ and $\delta=1.0$ under $H_1$. Shaded bands show $\pm$ one standard error across 20 independent outer repetitions. The dotted grey line marks $\alpha=0.05$.}
    \label{fig:experiment-6}
    \label{fig:linear-kernel-dimension}
\end{figure}

\FloatBarrier

\subsection{Effect of Joint Maximisation}
\label{sec:ablation-joint-maximisation}

We compare DR-KTE-Deep to a variant that jointly maximises the power learning criterion of the averaged ensemble. In DR-KTE-Deep, each kernel member is trained and selected separately. The final ensemble can therefore be composed of members selected from different training epochs. The joint maximisation variant instead selects a single ensemble checkpoint on the validation sample. Joint maximisation therefore removes the flexibility to select each member at its own best validation checkpoint.

Joint maximisation generally yields similar or lower estimated power than DR-KTE-Deep, while improving on DR-KTE (\Cref{tab:joint-maximisation}). OCT-Drusen is the exception, with higher power under joint maximisation. Type-I error remains close to the nominal level.

\begin{table}[h]
\centering
\resizebox{\linewidth}{!}{%
\begin{tabular}{lccc}
\toprule
& \multicolumn{3}{c}{Power/Type-I error ($H_1/H_0$)} \\
\cmidrule(lr){2-4}
Scenario & DR-KTE & DR-KTE-Deep & Joint maximisation \\
\midrule
Breast-Tumor & $0.749\pmerr{ 0.018}/0.053\pmerr{ 0.003}$ & $\mathbf{0.989}\pmerr{ 0.002}/0.054\pmerr{ 0.003}$ & $\mathbf{0.989}\pmerr{ 0.003}/0.053\pmerr{ 0.003}$ \\
Chest-Cardiomegaly & $0.156\pmerr{ 0.007}/0.050\pmerr{ 0.003}$ & $\mathbf{0.928}\pmerr{ 0.046}/0.051\pmerr{ 0.004}$ & $0.763\pmerr{ 0.077}/0.051\pmerr{ 0.004}$ \\
Chest-Effusion & $0.231\pmerr{ 0.009}/0.053\pmerr{ 0.004}$ & $\mathbf{0.999}\pmerr{ 0.000}/0.048\pmerr{ 0.004}$ & $0.998\pmerr{ 0.001}/0.056\pmerr{ 0.004}$ \\
Chest-Mass & $0.436\pmerr{ 0.018}/0.050\pmerr{ 0.003}$ & $\mathbf{1.000}\pmerr{ 0.000}/0.056\pmerr{ 0.004}$ & $0.999\pmerr{ 0.000}/0.050\pmerr{ 0.004}$ \\
Derma-Melanoma & $0.171\pmerr{ 0.009}/0.051\pmerr{ 0.002}$ & $\mathbf{0.898}\pmerr{ 0.040}/0.053\pmerr{ 0.004}$ & $0.795\pmerr{ 0.060}/0.046\pmerr{ 0.003}$ \\
OCT-CNV & $0.907\pmerr{ 0.014}/0.052\pmerr{ 0.003}$ & $\mathbf{1.000}\pmerr{ 0.000}/0.053\pmerr{ 0.003}$ & $\mathbf{1.000}\pmerr{ 0.000}/0.053\pmerr{ 0.003}$ \\
OCT-DME & $0.139\pmerr{ 0.008}/0.055\pmerr{ 0.004}$ & $\mathbf{1.000}\pmerr{ 0.000}/0.063\pmerr{ 0.005}$ & $\mathbf{1.000}\pmerr{ 0.000}/0.058\pmerr{ 0.004}$ \\
OCT-Drusen & $0.051\pmerr{ 0.003}/0.058\pmerr{ 0.004}$ & $0.371\pmerr{ 0.080}/0.058\pmerr{ 0.005}$ & $\mathbf{0.483}\pmerr{ 0.089}/0.053\pmerr{ 0.004}$ \\
Pneumonia & $0.857\pmerr{ 0.021}/0.055\pmerr{ 0.003}$ & $\mathbf{0.998}\pmerr{ 0.001}/0.053\pmerr{ 0.004}$ & $0.996\pmerr{ 0.001}/0.051\pmerr{ 0.003}$ \\
Retina-Severe & $0.285\pmerr{ 0.015}/0.050\pmerr{ 0.003}$ & $\mathbf{0.959}\pmerr{ 0.008}/0.052\pmerr{ 0.003}$ & $0.952\pmerr{ 0.009}/0.049\pmerr{ 0.003}$ \\
\midrule
Oncology Notes & $0.371\pmerr{ 0.015}/0.048\pmerr{ 0.003}$ & $\mathbf{0.809}\pmerr{ 0.023}/0.055\pmerr{ 0.003}$ & $0.770\pmerr{ 0.033}/0.050\pmerr{ 0.005}$ \\
Movie Reviews & $0.255\pmerr{ 0.012}/0.050\pmerr{ 0.003}$ & $\mathbf{0.449}\pmerr{ 0.026}/0.053\pmerr{ 0.004}$ & $0.423\pmerr{ 0.031}/0.054\pmerr{ 0.005}$ \\
Clinical Notes & $0.058\pmerr{ 0.004}/0.052\pmerr{ 0.003}$ & $\mathbf{0.939}\pmerr{ 0.013}/0.058\pmerr{ 0.004}$ & $0.933\pmerr{ 0.014}/0.062\pmerr{ 0.004}$ \\
\midrule
LLM Code Repairs & $0.968\pmerr{ 0.008}/0.048\pmerr{ 0.004}$ & $\mathbf{1.000}\pmerr{ 0.000}/0.051\pmerr{ 0.003}$ & $\mathbf{1.000}\pmerr{ 0.000}/0.048\pmerr{ 0.004}$ \\
Contact Networks & $0.959\pmerr{ 0.006}/0.052\pmerr{ 0.003}$ & $\mathbf{0.962}\pmerr{ 0.006}/0.051\pmerr{ 0.002}$ & $\mathbf{0.962}\pmerr{ 0.006}/0.051\pmerr{ 0.003}$ \\
\bottomrule
\end{tabular}%
}
\caption{Effect of joint maximisation at $N=1500$ and $\alpha=0.05$. Each cell reports $H_1/H_0$, with means $\pm$ one standard error across 20 outer repetitions.}
\label{tab:joint-maximisation}
\end{table}

\end{document}